\documentclass[10pt,onecolumn,a4paper]{report}

\usepackage{titlesec}
\titleformat{\chapter}[display]
{\normalfont\huge\bfseries}{\chaptertitlename\ \thechapter}{20pt}{\Huge}
\titlespacing*{\chapter}{0pt}{-15pt}{40pt}
\usepackage{pdfpages}
\usepackage{bbm}
\usepackage{etoc}
\usepackage[numbers]{natbib}
\let\cite\citep
\let\cite\citep  
\usepackage{graphicx}
\usepackage{wrapfig}
\usepackage{float}
\usepackage{url}
\usepackage{nicefrac}
\usepackage{algorithm}
\usepackage[noend]{algpseudocode}

\usepackage{amsmath,amsfonts,amssymb,mathtools}
\usepackage{dsfont}
\usepackage{booktabs}
\usepackage{multirow}
\usepackage{multicol}
\usepackage{caption}
\usepackage{subfigure}
\usepackage{wrapfig,lipsum,booktabs} 
\usepackage[export]{adjustbox}
\usepackage[update,prepend]{epstopdf}
\usepackage{acronym}
\usepackage{hyperref}
\usepackage{setspace}
\usepackage[titletoc]{appendix}

\usepackage{pdfpages}

\usepackage[]{xcolor}
\usepackage{tcolorbox}
\newtcolorbox{mybox}[2]{
    arc=0pt,
    boxrule=#2pt,
    colback=#1,
    width=15cm,
    halign=left,
    opacityframe = 0.1,
    opacityback = 0.0,
}

\hypersetup{
	bookmarksopen=true,
	bookmarksnumbered=true,
	colorlinks=true,
	citecolor=black,
	linkcolor=red,
	linktocpage=true,
	linkbordercolor={1 1 1},
	hypertexnames=false,
	plainpages=true,
	pdfsubject={},
	pdfkeywords={},
	pdftitle={Bayesian Deep Learning with Limited Data},
	pdfauthor={Idan Achituve, Ethan Fetaya and Gal Chechick}
}
\usepackage[cc]{titlepic}
\usepackage{amsthm}

\theoremstyle{plain}
\newtheorem{theorem}{Theorem}[section]

\newtheorem{lemma}[theorem]{Lemma}

\theoremstyle{definition}
\newtheorem{definition}[theorem]{Definition}

\theoremstyle{remark}

\newcommand{\T}{{\scriptscriptstyle \top}}

\acrodef{bdl}[BDL]{Bayesian Deep Learning}

\definecolor{atomictangerine}{rgb}{0.8, 0.2, 0.1}
\definecolor{turq}{rgb}{0.0, 0.5, 0.5}
\definecolor{darkturq}{rgb}{0.0, 0.4, 0.4}
\definecolor{bright}{rgb}{0.8, 0.1, 0}
\definecolor{darkgray}{gray}{0.3}
\definecolor{mahogany}{rgb}{0.6, 0.05, 0.05}
\definecolor{pink}{rgb}{1,0.05,0.6}
\definecolor{pink}{rgb}{0.761, 0.482, 0.627}
\definecolor{olive}{rgb}{0.537, 0.627, 0.318}
\definecolor{green}{rgb}{0.22, 0.463, 0.114}
\definecolor{grey}{rgb}{0.4, 0.4, 0.4}
\definecolor{blue}{rgb}{0.435, 0.659, 0.863}
\definecolor{darkred}{rgb}{0.7,0.05,0.2}
\definecolor{purp}{rgb}{0.5,0.0,0.5}

\definecolor{myblue}{rgb}{0.3,0.05,0.9}

\usepackage{amsmath,amsfonts,bm}

\newcommand{\1}{\mathbbm{1}}

\def\eqref#1{equation~\ref{#1}}

\def\1{\bm{1}}

\DeclareMathAlphabet{\mathsfit}{\encodingdefault}{\sfdefault}{m}{sl}
\SetMathAlphabet{\mathsfit}{bold}{\encodingdefault}{\sfdefault}{bx}{n}

\DeclareMathOperator{\Tr}{Tr}

\newcommand{\lt}{\ensuremath <}
\newcommand{\gt}{\ensuremath >}

\newcommand\ignore[1]{}

\begin{document}

\pagenumbering{gobble}	
\begin{center}
    ~\\[0.5cm]

    {\huge \bf Confidence Calibration of \\ Deep Learning Systems\\[2cm]} 

    {\huge Yacob (Coby) Penso}\\[2cm] 
     
    {\Large Faculty of Engineering}\\[2cm]

    {\Large Ph.D. Thesis}\\[2cm]

    {\Large Submitted to the Senate of Bar-Ilan University}\\[4cm]

    {\Large Ramat Gan, Israel \quad \quad \quad \quad \quad \quad \quad \quad \quad \quad \quad \quad May 2025}
\end{center}
\thispagestyle{empty}
\newpage

\begin{center}
    ~\\[7.0cm]
    {\large This work was carried out under the supervision of 
      \\[0.15cm] Prof. Jacob Goldberger, 
      \\The Faculty of Engineering,
      \\[0.15cm] Bar-Ilan University
    }
    \\[4.0cm]
\end{center}
\thispagestyle{empty}

\chapter*{Acknowledgment}
I would like to express my deepest gratitude to my supervisor, Prof. Jacob Goldberger, for his invaluable guidance, support, and encouragement throughout the course of my doctoral studies. His insight, patience, and high standards of research have shaped not only this thesis but also my approach to science as a whole.

I am also grateful to my collaborators and co-authors, Ethan Fetaya, Bar Mahpud, and Lior Frenkel, whose ideas, feedback, and dedication have greatly enriched this work. It has been a privilege to learn and work alongside them.

Finally, I extend my heartfelt thanks to my family and friends for their unwavering support, understanding, and encouragement during this journey. Without their love and patience, this thesis would not have been possible.

\begin{spacing}{0.5}
    \tableofcontents 
    \noindent\normalsize \textbf{Hebrew Abstract \hfill \large \color{red} \textbf{\textsl{$\aleph$}}}
\end{spacing}

\addtocontents{toc}{\protect\thispagestyle{empty}}
\pagenumbering{gobble}

\thispagestyle{empty}
\listoffigures
\thispagestyle{empty}
\listoftables
\thispagestyle{empty}
\listofalgorithms

\newpage
\thispagestyle{empty}
\noindent
\textbf{\Huge List of Abbreviations and Notations}
\newline
\newline
\newline
In this thesis $\epsilon$ used as the noise rate, except for Chapter \ref{ch:ldpcp} when $\epsilon$ is the privacy in $\epsilon$-LDP and $\beta $ used as the noise rate.
\newline

\begin{flushleft}
\begin{tabular}{ll c c }
CE & Conformal Prediction\\~\\
CE & Cross-Entropy\\~\\
CNN & Convolutional Neural Network\\~\\
DL & Deep Learning\\~\\
DNN & Deep Neural Network\\~\\
ECE & Expected Calibration Error\\~\\
adaECE & Adaptive Expected Calibration Error\\~\\
LLMs & Large-Language Models\\~\\
MCE & Maximum Calibration Error\\~\\
MSE & Mean Square Error\\~\\
NN & Neural Network\\~\\
DA & Domain Adaptation\\~\\
UDA & Unsupervised Domain Adaptation\\~\\
\end{tabular}
\end{flushleft}
\begin{flushleft}
\begin{tabular}{ll c c }
OOD & Out-of-Distribution\\~\\
UTDC & Unsupervised Target Domain Calibration \\~\\
NRCP & Noise Robust Conformal Prediction \\~\\
NACP & Noise Aware Conformal Prediction \\~\\
NTS & Noisy Temperature Scaling  \\~\\
TS & Temperature Scaling \\~\\
DP & Differential Privacy \\~\\
LDP & Local Differential Privacy \\~\\
LDP-CP & Local Differential Private Conformal Prediction \\~\\
LDP-CP-L & Local Differential Private Conformal Prediction On Labels \\~\\
LDP-CP-S & Local Differential Private Conformal Prediction On Scores \\~\\
\end{tabular}
\end{flushleft}


\setcounter{tocdepth}{2} 
\pagenumbering{gobble}

\pagenumbering{roman}
\addcontentsline{toc}{chapter}{Abstract}
\chapter*{Abstract}\vspace{-0.5in}

In high-stakes applications such as medical imaging, the reliability of a model’s confidence in its predictions is as crucial as the predictions themselves. Confidence calibration ensures that a model's predicted probabilities accurately reflect its likelihood of correctness, making it a critical component for safe and effective deployment of deep learning models in medical diagnostics. However, existing calibration techniques assume access to clean validation data, which is often unrealistic in medical imaging settings due to the prevalence of label noise and domain shifts. This thesis explores novel methods for improving confidence calibration under these challenging conditions.

First, we address confidence calibration in the presence of label noise. When calibration methods are applied to data with unreliable labels, they may yield misleading confidence estimates that undermine the trustworthiness of model predictions. We propose a calibration framework that accounts for label noise by leveraging an estimated noise model. Specifically, we demonstrate how to reconstruct noise-free confidence estimates by modeling the relationship between noisy and clean label distributions. We extend this idea to Conformal Prediction (CP), a framework that provides set-valued predictions with a guaranteed level of coverage. We introduce a noise-aware conformal prediction approach that estimates the true conformity scores despite label noise, allowing us to maintain efficient and reliable uncertainty quantification.

Next, we investigate confidence calibration in unsupervised domain adaptation (UDA), where a model trained on a labeled source domain is adapted to an unlabeled target domain. Traditional calibration methods require labeled validation data from the target domain, which is unavailable in this setting. To overcome this limitation, we develop an approach that estimates the target domain accuracy based on the model’s performance in the source domain and known domain discrepancies. This allows us to directly calibrate model confidence without access to target domain labels.

Furthermore, we extend our study to privacy-preserving settings, where individual user labels and model outputs must be protected. We propose a locally differentially private conformal prediction framework that ensures valid uncertainty quantification while maintaining rigorous privacy guarantees. Our approach balances the trade-offs between privacy, computational feasibility, and prediction reliability, making it applicable to sensitive medical data applications.

Through extensive experiments on natural and medical imaging datasets, we demonstrate that our proposed methods significantly improve calibration robustness under both label noise and domain shift conditions. We provide theoretical guarantees and empirical validations that bridge the gap between theoretical calibration guarantees and practical deployment in safety-critical environments. Our findings contribute to the development of reliable, privacy-preserving, and noise-resilient calibration frameworks, enhancing the trustworthiness of neural network predictions in real-world medical and high-stakes applications.

\pagenumbering{arabic}
\newcommand{\methoname}{FED } 



\chapter{Introduction}

In high-stakes domains like medical imaging, the accuracy of a model’s confidence in its predictions can be just as important as the predictions themselves. Confidence calibration is the process of ensuring that a model's predicted probabilities align with its actual accuracy, offering reliable confidence estimates for each prediction.
Confidence calibration is defined as the ability of a classifier network to provide an accurate probability of correctness for any of its predictions.
Neural networks have been shown to be more overconfident in their predictions than their predecessors even though their generalization accuracy is higher, partly due to the fact that they can overfit on the
negative log-likelihood loss without overfitting on the classification error \cite{Guo2017,Balaji2017,Hein2019}.
 In medical imaging applications, images for which the model makes low-confidence predictions are sent to a physician for review. Skipping the human review based on  confident but incorrect predictions can have disastrous consequences \cite{minderer2021revisiting}. 
The  gap between the model's predicted probabilities and its  accuracy is one of the  key obstacles to
the applicability  of  neural network models to  fully automatic medical diagnosis.

Calibration methods can generally be divided into two main approaches. The first focuses on calibrating the confidence in a predicted class, while the second addresses the problem by generating a prediction set - a collection of possible classes - with a specified probability that the true class is included in this set. These approaches will be referred to as \textit{Confidence Calibration} and \textit{Conformal Prediction}, respectively, throughout this thesis.

Within the framework of confidence calibration, various methods have recently been developed to address the issue of excessive overconfidence in predictions. Network calibration can either be performed alongside training (e.g., \cite{mukhoti2020calibrating, muller2019does, mixup, xu2023mismatch}) or applied as a post-hoc procedure (e.g., Platt scaling \cite{Platt1999}, isotonic regression \cite{Zadrozny2002}, and temperature scaling \cite{Guo2017}). Post-hoc methods improve calibration by applying it as a post-processing step, using hold-out validation data to create a calibration map that adjusts the model's predictions. Among these, temperature scaling stands out as a practical and widely adopted approach due to its simplicity and ease of implementation. Despite the critical role of network calibration in automating medical reporting, only a limited number of studies have specifically addressed the calibration of medical imaging systems (e.g., \cite{fernando2021dynamically, frenkel2022calibration, rousseau2021post, zhang2020layer}).

In conformal prediction, the goal is to return a (preferably small) set of potential class candidates that includes the true class with a predefined level of confidence. This approach is particularly well-suited for medical imaging, where safety is paramount and the final decision is made by a human. By reducing the number of possible diagnoses a practitioner needs to consider, conformal prediction helps streamline decision-making while maintaining a controlled risk of error. The general method of producing a prediction set without making assumptions about the data distribution (aside from i.i.d. samples) is known as Conformal Prediction (CP) \cite{angelopoulos2023conformal, vovk2005conformal}. CP guarantees that the probability of the correct class being included in the set meets or exceeds a specified confidence level, while aiming to return the smallest set possible that still maintains this guarantee. With the increasing use of neural networks in safety-critical applications like medical imaging, CP has emerged as a crucial calibration tool \cite{lu2022improving, lu2022fair, olsson2022estimating}. It is important to note that CP is a general framework rather than a single algorithm, with the most common implementations constructing the prediction set based on a conformity score. Different algorithms mainly differ in how this conformity score is defined.

Confidence calibration and conformal prediction are extensively studied in settings where clean data and labels are provided. However, our research focuses on more complex and realistic scenarios, specifically when labels are noisy and in the context of unsupervised domain adaptation, where the goal is to calibrate the target model without access to labeled data.

Deep neural networks have been highly successful in various natural image and medical image computing tasks. However, these achievements depend on having accurate annotated training data. Neural networks require massive amounts of carefully labeled data to succeed, but acquiring such data is expensive and time-consuming. Non-expert sources, like Amazon's Mechanical Turk, have been used to reduce labeling costs, but their labels can be unreliable. Experienced domain experts may also struggle with complex labeling tasks. 
Medical imaging datasets often have problems with noisy labels due to ambiguous images that can confuse clinical experts.  Physicians may disagree on the diagnosis of the same medical image, resulting in variability in the ground truth label. Furthermore, using Natural Language Processing (NLP) tools to extract labels from radiological reports can also introduce label noise \cite{irvin2019chexpert}. Therefore, addressing annotation noise is a crucial topic in medical image analysis.

 Training neural networks with noisy labels is problematic because the models can easily overfit to the corrupted labels, resulting lack of generalizability when evaluated on a separate test dataset.
  While popular regularization techniques have been used to address overfitting, they do not entirely solve the problem. Even when these techniques are applied, there is a significant gap in test accuracy between models trained on clean vs. noisy data, and the accuracy decreases with label noise.
Noisy labels are difficult to avoid, and studies indicate that Deep Neural Networks (DNNs) can memorize entire datasets. Consequently, errors in datasets may result in erroneous predictions, which can impact medical diagnoses. Therefore, effectively managing noisy labels is crucial for automated medical image classification. A review of network training methods for noisy labels can be found in \cite{9729424} and an excellent up-to-date discussion of training medical image classification networks from data with noisy labels can be found in \cite{Xue2022}. 

Numerous studies have examined the problem of training networks that are resilient to label noise, which can also disrupt the network calibration process. Our findings suggest that network calibration methods are more susceptible to label noise compared to network training. Nevertheless, we have not come across any previous research that tackles the challenge of network calibration using a validation set containing noisy labels.       

In addition to the challenges posed by noisy labels, another critical issue in real-world applications of deep learning is the performance degradation that occurs when a network trained on data from one domain is applied to data from a different domain, where the feature distribution differs - a phenomenon known as domain shift (see e.g., \cite{pmlr-v139-miller21b}). In the context of Unsupervised Domain Adaptation (UDA), the goal is to adapt a model to a target domain where labeled data is unavailable, though data from the target domain itself is accessible. Our findings indicate that existing calibration methods for unsupervised domain adaptation often fail in practice, particularly when the domain gap is large, further complicating the task of achieving reliable network calibration.

Lastly, in many critical settings, the calibration procedure is performed by a centralized component, referred to as an aggregator, which may be untrusted. In such scenarios, exposing sensitive data directly to this untrusted aggregator poses significant privacy risks. A promising approach to mitigating these risks is Local Differentially Private (LDP) Conformal Prediction, where individual data contributors apply noise to their calibration data before sharing it with the aggregator. This ensures that the aggregator can perform calibration without directly accessing private or sensitive information from individual sources. However, the introduction of noise through LDP presents new challenges in maintaining both the validity and efficiency of conformal prediction methods. We explore strategies to adapt conformal calibration techniques to function effectively under differential privacy constraints, ensuring that predictions remain reliable while preserving user privacy.

In our research, we tackle these challenges in three key areas: first, by addressing the calibration of neural networks using validation sets with inaccurate labels, focusing on both confidence calibration and conformal prediction; second, by exploring confidence calibration in systems facing unsupervised domain shift scenarios; and third, by investigating privacy-preserving calibration methods using local differentially private conformal prediction, ensuring robust calibration without compromising data privacy. Through these efforts, we aim to enhance the reliability, robustness, and privacy of neural networks in real-world applications where label noise, domain shifts, and privacy concerns are significant challenges.

The remainder of this thesis is structured as follows:

Chapter \ref{ch:bg}, Background, provides an overview of the foundational concepts and related work, including confidence calibration, conformal prediction, unsupervised target domain calibration, noisy labels, and local differential privacy.

Chapter \ref{ch:nts}, Confidence Calibration under Noisy Labels, introduces our proposed methods for improving confidence calibration when dealing with noisy labels. We explore both the challenges of confidence calibration with noisy labels and network training strategies, followed by a comprehensive set of experiments to validate our approach.

Chapter \ref{ch:nrcp_nacp}, Conformal Prediction under Noisy Labels, extends our focus to conformal prediction, presenting a robust scoring method and a threshold estimation procedure tailored to noisy label scenarios. We provide theoretical insights and experimental evaluations, including prediction size comparisons, coverage guarantees, and adaptations to more general noise models.

Chapter \ref{ch:ldpcp}, Local Differential Private Conformal Prediction, delves into privacy-preserving techniques, introducing methods to apply local differential privacy (LDP) to both labels and prediction scores. We discuss theoretical guarantees, practical considerations, and experimental results, comparing different approaches under privacy constraints.

Chapter \ref{ch:utdc}, Unsupervised Target Domain Confidence Calibration, addresses the challenge of calibrating confidence when transitioning models to a new, unlabeled target domain. We describe our proposed calibration methods, present experimental findings, and analyze the performance under unsupervised domain adaptation settings.

Chapter \ref{ch:discussion}, Discussion, summarizes our contributions, highlights the key insights and implications of our work, and outlines potential directions for future research. The chapter concludes by reaffirming the significance of our research in advancing confidence calibration and conformal prediction in challenging real-world scenarios.

As part of the thesis, the following papers have been
published in various conferences and journals:
\begin{itemize}
    \item Confidence calibration of a medical imaging classification
system that is robust to label noise - Coby Penso, Lior
Frenkel, Jacob Goldberger, IEEE Transactions on Medical
Imaging (TMI), vol. 43(6), pp. 2050-2060, 2024, \cite{penso2024confidence}.
    \item A joint training and confidence calibration procedure that is
robust to label noise - Coby Penso, Jacob Goldberger, IEEE
International Symposium on Biomedical Imaging (ISBI),
2024, \cite{10635160}.
    \item A conformal prediction score that is robust to label noise -
Coby Penso, Jacob Goldberger, MICCAI, Machine Learning for Medical Imaging Workshop, 2024, \cite{penso2024conformal}.
\item Network calibration under domain shift based on estimating
the target domain accuracy - Coby Penso, Jacob
Goldberger, ECCV, Uncertainty in Computer Vision Workshop, 2024, \cite{penso2024calibrationuda}.
\item Conformal Prediction of Classifiers with Many Classes based on Noisy Labels - Coby Penso, Jacob Goldberger, Eitan Fetaya, accepted to the Symposium on Conformal and Probabilistic
Prediction with Applications (COPA), 2025, \cite{PensoNACP2025}.
\item Privacy-Preserving Conformal Prediction Under Local Differential Privacy - Coby Penso, Bar Mahpud, Jacob Goldberger, Or Sheffet, accepted to Symposium on Conformal and Probabilistic
Prediction with Applications (COPA), 2025, \cite{PensoLDPCP2025}.

\end{itemize}

\chapter{Background}
\label{ch:bg}
\label{C:definitions}

\section{Confidence Calibration}
\label{s:cc}

In this section we review the definition of confidence calibration. Consider a network that classifies an input  image  $x$ into $k$ pre-defined categories.
 The last layer of the network architecture is comprised of  $k$ real numbers $z=(z_1,...,z_k)$ known as \emph{logits}. Each of these numbers is the score for one of the $k$ possible classes. The logits are then converted into a soft decision distribution using  a \emph{softmax} layer: $p(y=i|x) = \frac{\exp(z_i)}{\sum_j \exp(z_j)}$ where $x$ is the input image and $y$ is the image class.
Despite having the mathematical form of a distribution, the output of the softmax layer does not necessarily represent the true posterior distribution of the classes, and the network often tends to have  overconfidence in its predictions.

The predicted class is calculated from the output distribution  by $\hat{y} = \arg \max_{i} p(y=i|x)= \arg \max_i z_i $. The network \emph{confidence} for this sample is defined by $\hat{p} = p(y=\hat{y}|x) =  \max_{i} p(y=i|x)$. The network \emph{accuracy}  is defined by the probability that the most probable class
$\hat{y}$ is indeed correct. The network is said to be \emph{calibrated} if the estimated confidence coincides with the actual accuracy. 

The Expected Calibration Error (ECE) \cite{Naeini2015} stands as the conventional metric employed for quantifying the calibration of a model. It is characterized by the expected absolute disparity between the model's accuracy and its level of confidence.  In practice, we only have a validation set with a finite number of samples $(x_1,y_1),...,(x_n,y_n)$ thus an approximation is used. Denote the  predictions and confidence values of the validation set by $(\hat{y}_1,\hat{p}_1),...,( \hat{y}_n,\hat{p}_n)$.
 To compute the ECE measure we first divide the unit interval $[0,1]$ into $m$ equal size bins $b_1,...,b_m$ and let $B_i=\{t| \hat{p}_t \in b_i\}$ be the set of samples whose confidence values   belong to  bin $b_i$. The network  average accuracy at this bin is computed as:
 \begin{equation}
A_i = \frac{1}{|B_i|} \sum_{t \in B_i} \mathbbm{1}_{\{\hat{y}_t = y_t\}},
\label{aidef}
 \end{equation}
 where $\mathbbm{1}$ is the indicator function, and $y_t$ and $\hat{y}_t$ are the  correct and  predicted labels for $x_t$ respectively. 
 $A_i$ is the relative number of correct predictions of instances that  were assigned to  $B_i$ based on their confidence value. 
 The average confidence at bin $b_i$ is computed as:
 \begin{equation}
C_i = \frac{1}{|B_i|} \sum_{t \in B_i} \hat{p}_t.
\label{cidef}
\end{equation}
If the network is under-confident at  bin $b_i$ then $A_i > C_i$  and  vice-versa.
The ECE is defined as follows: \begin{equation}
\mathrm{ECE} = \sum_{i=1}^{m} \frac{|B_i|}{n} \left| A_i - C_i \right|.
\label{ECEdef}
\end{equation}
 The ECE is based on a uniform bin width. If the model is well-trained, most of the
samples should lie within the highest confidence bins.   
Hence, the low confidence bins should be almost empty and therefore have no influence on the computed value of the ECE. For this reason, we  can consider another metric, Adaptive ECE (adaECE) where bin sizes are taken into account so as to evenly distribute samples between bins  
\cite{Nguyen2015}:
\begin{equation}
\mathrm{adaECE} = \frac{1}{m} \sum_{i=1}^{m} \left| A_i - C_i \right|
\label{adaECEdef}
\end{equation}
such that each bin contains $1/m$ of the data points with similar confidence values.  
AdaECE is considered a better and more resilient method than ECE for assessing network calibration. In this study we used the adaECE for both calibration and evaluation.

Temperature Scaling (TS) is a standard  highly effective technique for calibrating  the output distribution of a classification network \cite{Guo2017}. It uses a single parameter $T > 0$ to rescale logit scores before applying the softmax function to compute the class distribution.
 Temperature scaling is expressed as follows: 
\begin{equation}
    p_{{\scriptscriptstyle T }}(y=i|x) = \frac{\exp (z_i / T)}{\sum_{j=1}^k\exp (z_j / T)}, \hspace{0.4cm}   i=1,\dots,k
    \label{classcal}
\end{equation} 
s.t. $z_1,...,z_k$ are the logit values derived from the application of the network to the input vector $x$.
 The optimal temperature $T$ for a trained model can be found by  maximizing the log-likelihood $\sum_t \log  p_{{\scriptscriptstyle T }}(y_t|x_t)$ for the held-out validation dataset. Studies show that finding the optimal $T$ by directly minimizing the ECE or adaECE measures yields better calibration results \cite{mukhoti2020calibrating}.

\section{Conformal Prediction}
\label{s:ncp}

Consider a setup involving a classification network that categorizes an input $x$ into $k$ predetermined classes.  Given a coverage level of $1-\alpha$,
we aim to identify the smallest possible prediction set (a subset of these classes)
ensuring the correct class is within the set with a probability of at least $1 -\alpha$.
A straightforward strategy to achieve this objective involves sequentially incorporating classes from the highest to the lowest probabilities until their cumulative sum 
exceeds the threshold of $1-\alpha$. Despite the network's output adopting a mathematical distribution format, it does not inherently reflect the actual class distribution. Typically,
the network will not be calibrated and it tends to be overly optimistic \cite{Guo2017}. 
Consequently, this straightforward approach doesn't assure the inclusion of the correct class with the desired probability.

The first step of the CP algorithm involves forming a conformity score $S(x,y)$ that  measures the network's uncertainty between $x$ and its true label $y$ (larger scores indicate worse agreement). The Homogeneous Prediction Sets (HPS) score \cite{vovk2005conformal} is
$S_{\scriptscriptstyle \textrm HPS}(x,y)=1-p(y|x;\theta)$, s.t. $\theta$ is the network parameter set.
The Adaptive Prediction  Score (APS) \cite{romano2020classification}  is the sum of all class probabilities that are not lower  than the probability of the true class:
\begin{equation}
 S_{\scriptscriptstyle APS}(x,y) =  \sum_{\{i|p_i \ge p_{y}\}} p_i,
   \label{aps_score}
  \end{equation}
such that $p_i= p(y=i|x;\theta)$ and $p_y$ is the probability of the  label $y$.
The RAPS score \cite{angelopoulos2020uncertainty} is a variant of APS, which  is  defined as follows:
\begin{equation}
 S_{\scriptscriptstyle RAPS}(x,y) =  \sum_{\{i|p_i \ge p_{y}\}} p_i + a \cdot \max(0,(NC - b))
      \label{raps_score}
   \end{equation}
s.t. $NC=|\{i|p_i \ge p_{y}\}|$ and  $a,b$ are parameters that need to be tuned.
RAPS is especially effective in the case of a large number of classes where it explicitly encourages small prediction sets.

We can also define a randomized version of a conformity score. For example in the case of  APS
we define: \begin{equation}
 S_{\scriptscriptstyle rand-APS}(x,y,u) =  \sum_{\{i|p_i > p_{y}\}} p_i + u \cdot p_y,\hspace{1cm} u\sim U[0,1].
  \label{radaps_score}
  \end{equation}
The random version tends to yield the required coverage more precisely and thus it produces smaller prediction sets \cite{angelopoulos2023conformal}.
The CP prediction set of a data point $x$ is defined as $C(x)=\{y| S(x,y) \le q\}$ where $q$ is a threshold that is found using a labeled validation set $(x_1,y_1),...,(x_n,y_n)$. The CP theorem states that if we set  $q$ to be the $(1\!-\!\alpha)$ quantile of the conformal scores $S(x_1,y_1),...,S(x_n,y_n)$  we can guarantee that $1\!-\!\alpha \le p( y\in C(x)) \le 1\!-\!\alpha + \frac{1}{n+1}$,
where $x$ is a test point and $y$ is its the unknown true label \cite{vovk2005conformal}. 
 In the random case there is still a coverage guarantee, which is defined by marginalizing over all test points $x$  and samplings  $u$ from the uniform distribution  \cite{romano2020classification}.
Note that the coverage guarantee is for a marginal probability over all possible test points and coverage may be worse or better for different points. It can be proved that obtaining a conditional coverage guarantee is impossible \cite{foygel2021limits}.

\section{Unsupervised Target Domain Calibration}
\label{s:uda}

When a network trained on data from one domain is applied to data from a different domain, the distribution of features often changes between domains, leading to what is known as the domain shift problem (see e.g. \cite{pmlr-v139-miller21b}). In an Unsupervised Domain Adaptation (UDA) setup, we assume the availability of data from the target domain without any annotations. Numerous UDA methods have been developed to address this issue, employing strategies such as adversarial training to align the distributions of the source and target domains \cite{ganin2016domain}, or self-training algorithms that compute pseudo labels for the target domain data \cite{zou2019confidence}.
Studies show that present-day  UDA methods are prone to learning improved accuracy at the expense of deteriorated prediction confidence \cite{wang2020}. 

This brings us to the challenge of calibrating the network’s confidence on the target domain data. In UDA, adapting a network to the target domain typically focuses on accuracy improvements, but accurately calibrating the model’s confidence is equally important. Without calibration, the model may become overconfident in its predictions despite reduced performance on the target domain.

Next, we formulate the problem of unsupervised target domain calibration. Assume a network was trained on the source domain. We are given a labeled source domain validation-set dataset, denoted as $\mathcal{S}= \{(x_s^i, y_s^i)\}_{i=1}^{n_s}$ with $n_s$ samples, and an unlabeled target domain dataset $\mathcal{T}=\{x_t^i\}_{i=1}^{n_t}$ with $n_t$ samples. Adapting the network trained on the source domain to the target domain in an unsupervised manner without access to the labels can be achieved using various methods.
 Here, our goal is to calibrate the confidence of the adapted network prediction on samples from the target domain. 

Calibrating the confidence of the adapted model on data from the target domain is challenging due to the coexistence of the domain gap and the lack of target labels. Current UDA calibration methods use the labeled validation set from the source domain to approximate the target domain statistics in certain aspects.
Some studies \cite {salvador2021improved,tomani2021post} propose to modify the calibration set to represent a generic distribution shift.
Other methods  \cite{park2020,wang2020,pampari2020}  apply Importance Weighting (IW)
 by assigning higher weights to source examples that resemble those in the target domain.
 In practice, even after the domain adaptation process, the accuracy on the source domain, where labels are available,  remains greater than the accuracy on the target domain.  Hence, the accuracy estimation when calibrating
the target domain using the source data is still too optimistic.  Calibrating neural networks is necessary because they are often overconfident in their predictions compared to their actual accuracy \cite{Guo2017,Balaji2017,Hein2019}. If the accuracy is overestimated, it conceals the overconfidence issue, leading to a suboptimal temperature scaling value in the case of temperature scaling. 
Another drawback of IW methods is that they only use the unlabeled target data to train a binary source/target classifier, but the actual calibration is done on the source domain data while the target domain data are ignored.  The network confidence, however, is independent of the true labels and can thus be directly computed on the target data.

\section{Noisy Labels}
\label{s:noisy}

In supervised classification tasks a dataset is defined as pairs of input and labels $\mathcal{D}= \{(x^i, y^i)\}_{i=1}^{n_D}$.
In the setting of label noise, we only observe the corrupted labels $\tilde{y}_i = g(y_i)$ for some corruption function $g: Y \times [0,1] \xrightarrow{} Y$.

One notable label noise function follows a uniform distribution, where with a probability of $\epsilon$,  the correct label is replaced by a randomly selected label \ref{uniformn_intro_v1}. The noise is applied to each sample independently.  This noise model is commonly referred to as uniform noise. The uniform noise may be formulated also such that with a probability of $\epsilon$,  the correct label is replaced by a randomly selected label from the remaining $(k-1)$ classes \ref{uniformn_intro_v2}. In this thesis, we adopt both definitions interchangeably depending on the context, as the transition between them is straightforward.

A more general noise model assumes that the true label is corrupted by a label noise matrix  $P$, where
$P(i,j) = p(\tilde{y} = j |y = i)$ is the probability of the true label $i$ being flipped to a corrupted label $j$. 
In the simpler uniform noise model, $P$ takes the form of:  

\begin{equation}
P(i,j) = (1-\epsilon)  \mathbbm{1}_{\{i=j\}}+ \frac{\epsilon}{k}.
\label{uniformn_intro_v1}
\end{equation}

\begin{equation}
P(i,j) = (1-\epsilon)  \mathbbm{1}_{\{i=j\}}+ \frac{\epsilon}{k} \mathbbm{1}_{\{i\ne j\}}.
\label{uniformn_intro_v2}
\end{equation}

In future chapters, we concentrate on calibrating the network confidence based on noisy labels. A preliminary step before confidence calibration is training the network using a training set with noisy labels (see Figure \ref{fig:nts_schema}).
We next provide a brief overview of current training methods that are resilient to label noise and describe the training method we used in our experiments. 

There is a plethora of recent works on learning with noisy labels, which include  estimating the noise 
matrix~\cite{goldberger2017training, dgani2018training,li2021provably,HendrycksMWG18,xia2019anchor,xia2020part,cheng2022cvpr}, reweighting examples~\cite{liu2015classification,RenZYU18,ShuXY0ZXM19,wang2021graph}, selecting confident examples~\cite{huang2019o2u,YangICML2020,chen2022anomman,Li2020aaai}, designing robust loss functions~\cite{zhang2018generalized,GhoshKS17,ChengZLGSL21,wei2021robust}, introducing regularization~\cite{zhang2018mixup,hu2020simple,chen2021noise} and generating pseudo labels~\cite{tanaka2018joint,zheng2020error,zhang2021learningwith,han2019deep,Li2021TMM}.
Zhang et al. \cite{zhang2020disentangling} addressed the problem of learning from noisy labels in the context of inconsistent annotation collected from several medical experts. They also relied on the fact that the
confusion matrix of the noisy labels can be expressed as
the matrix product between the confusion matrix of the
clean labels and the label noise.
Our focus in this context is on learning techniques that deal with noisy data by estimating the label noise matrix. These methods have been successful in producing cutting-edge results, and our calibration approach utilizes the estimated noise matrix to obtain a noise-robust calibration measure.

Noise robust training methods which estimate the noise matrix, 
are all based on the following observation.  The clean class-posterior $p(y|x)$ can be inferred by utilizing the noisy class-posterior $p(\tilde{y}|x)$ and the
class-dependent noise  matrix  $P$,  where $P_{ij} = P(\tilde{y} = j|y =i )$, as follow:  $p(y|x) = P^{-1}p(\tilde{y}|x)$. While this approach theoretically guarantees statistical consistency, it relies heavily  on the success of estimating the noise matrix.
  Several  methods have been developed to estimate the noise matrix  under the so-called anchor-point assumption. Anchor points are  instances belonging to a specific class with a probability of one  \cite{liu2015classification}. This assumption is  reasonable in certain applications but typically, we cannot assume the availability of anchor points. This has motivated the development
of noise-robust training algorithms that do not  exploit anchor points.
Several studies have implemented modifications to the classification network's architecture  to better represent the label noise matrix  in noisy datasets \cite{bekker2016training,goldberger2017training,li2021provably}. These adjustments encompass the inclusion of a noise adaptation layer on top of the softmax layer and the creation of a specialized architecture.
 The noise adaptation layer is intended to mimic the label transition behavior in learning a network.
These changes have led to enhanced generalization by altering the network output according to the estimated label transition probability. In our implementation, of training of noisy labels we follow the approach in  Li et al. \cite{li2021provably} which yields state-of-the-art results. 

Denote the network's soft-max label prediction by $p_{\theta}(y|x)$ where $\theta$ is the network's parameter set. Given  training data $x_1,...,x_n$ with corresponding noisy labels $\tilde{y}_1,...,\tilde{y}_n$, 
the standard loss function is: 
\begin{equation}
 L(\theta) = \sum_{t=1}^n \textrm{CE}( \tilde{y}_t, p_{\theta}(y_t|x_t)),
\end{equation}
such that CE is the cross-entropy loss. 
Li et al. \cite{li2021provably}
proposed  minimizing the following loss function:
\begin{equation}
 L(\theta,P) = \sum_{t=1}^n \textrm{CE}( \tilde{y}_t, P \cdot p_{\theta}(y_t|x_t)) + \lambda 
 \log \det (P)
\end{equation}
such that $\lambda>0$ is  
 a regularization coefficient that trades off
distribution fidelity with  the complexity of the matrix $P$.
$P$ is enforced to be a diagonally dominant stochastic matrix
(i.e. $P_{i,i} > P_{i,j}$ for every $i\ne j$). We first create a matrix $A$ s.t. $A_{i,i}=1$ and 
$A_{i,j} = \sigma(w_{i,j})$
for all $ i \ne j$  where $\sigma$ is
the sigmoid function and each $w_{i,j}$ 
is a real-valued variable that is  updated throughout training. 
Then we normalize each row to obtain a stochastic matrix: $P_{ij}=\frac{A_{i,j}}{\sum_l A_{i,l}}$.
Once we finish the training phase, we eliminate the noise adaptation layer defined by  matrix $P$,  because our objective is to predict the clean label. As a by-product of  the training process, we also obtain an approximation of the label noise matrix  $P$, which we can utilize in our network calibration approach. 

Current label-noise learning methods generally assume that the class distribution of the training data is balanced, i.e.,  that each class is represented by almost the same number of samples. However, data in real-world applications are often imbalanced.
In cases where the training labels are both noisy and imbalanced it is difficult to distinguish between clean and noisy samples in rare classes because the clean samples are overwhelmed by noisy labels from frequent classes. Several recent attempts have been made to find noise-robust training procedures for imbalance data \cite{lu2022label,huang2022uncertainty,jiang2022delving}.
However, all these methods are focused on extracting confident examples. We are not aware of any methods for estimating the noise matrix with unbalanced data.

\textbf{Additional noise types.} Next, we define two common general noise matrices that will be used through out this thesis.
The Neighborhood noise as:
\begin{equation}
P_{i,j} = p(\tilde{y}=j|y=i) =
    \begin{cases}
    \xi \hspace{1.7cm} \text{if } i=j\\
    1 - \xi \hspace{1.1cm} \text{if } |i-j|=1 \text{ and } i \in (1,k)\\
    (1 - \xi) / 2 \hspace{0.5cm} \text{if } |i-j|=1 \text{ and } i \not \in (1,k)\\
    0       \hspace{1.65cm}    \text{otherwise}
\end{cases}
\end{equation}
The Random noise is defined as:
first, on the diagonal, we have $\xi$. next, for each line (aka $\forall i$) the rest of the values (i.e. $k-1$ items) are sampled from a random distribution $u_i$ vector of size $k-1$ and then normalized to sum up to $1-\xi$ to keep the matrix a probability matrix.

\begin{equation}
P_{i,j} = p(\tilde{y}=j|y=i) =
    \begin{cases}
    \xi \hspace{3.9cm} \text{if } i=j\\
    (1 - \xi) \cdot \frac{u_i[j]}{\sum_{z \neq i} u_i[z]}       \hspace{1.3cm}    \text{otherwise}
\end{cases}
\end{equation}

In this study, we tackle the challenge of applying calibration on a validation set with noisy labels. Calibration methods, particularly for neural networks, are known to be highly sensitive to label noise. In the context of conformal prediction, Einbinder et al. \cite{einbinder2022conformal} proposed ignoring the label noise and directly applying the standard CP algorithm to the noisy-labeled validation set. However, this approach tends to result in overly large prediction sets. In contrast, within the field of confidence calibration, we have not encountered any prior research that specifically addresses the issue of calibrating networks using validation sets with noisy labels.



\section{Local Differential Privacy}
\label{s:ldp}

Traditional (central) differential privacy \citep{dwork2006differential} presumes a trusted curator who sees the raw data and then adds noise before publishing. In contrast, LDP \citep{duchi2013local,kairouz2016discrete, kasiviswanathan2011can} treats the aggregator as \emph{untrusted}: individual users randomize their own data \emph{locally} before sending it to the aggregator, thus ensuring strong privacy. LDP is considered a harder setting since noise insertion is done on the user side in a distributed manner, whereas in the centralized DP model the curator holds the entire data and can apply operations on the clean data.

\begin{definition}
    A discrete randomized mechanism $Q(\cdot)$ is $\varepsilon$-LDP if for any pair of input labels $y,y'\in\mathcal{Y}$ and any output $z$,
    \[
    Q(z\mid y) \;\le\; e^{\varepsilon}\,Q(z\mid y').
    \]
\end{definition}
This definition ensures that any two possible labels are (roughly) indistinguishable from the aggregator's perspective. A common mechanism is the \emph{$k$-ary randomized response} ($k$-RR) \citep{warner1965randomized,wang2017locally}. For a label $y\in\{1,\dots,k\}$, it outputs:
\[
\tilde{y} =
\begin{cases}
y, &\text{w.p } \tfrac{e^\varepsilon}{(k-1) + e^\varepsilon},\\
\text{any other label (uniformly)}, &\text{w.p } \tfrac{1}{(k-1) + e^\varepsilon}.
\end{cases}
\]

This preserves label privacy, preventing the aggregator from easily inferring the user's true label from the reported $\tilde{y}$. The parameter $\varepsilon$, known as the \emph{privacy loss}, controls the privacy-utility trade-off: lower $\varepsilon$ enforces stronger privacy guarantees but introduces more noise, potentially degrading the utility of downstream applications. Note that if $k=2$, we recover \emph{Warner's binary randomized response} (RR) \citep{warner1965randomized}, flipping the label with some probability. 

LDP has gained popularity as a strong privacy paradigm that enables data owners to randomize their data locally before sharing it with an untrusted aggregator, thus ensuring that sensitive information remains protected (e.g.\ Google's RAPPOR \cite{erlingsson2014rappor} and Apple's locally private data collection of emojis and usage patterns \cite{apple2017privacy}).

\textbf{Challenges of LDP for statistical learning and conformal prediction.} While LDP ensures strong privacy guarantees, its main challenge lies in the significant increase in variance due to local randomization. Unlike central DP, where controlled noise can be added post-aggregation, in LDP, the noise is introduced at the user level, leading to a loss of information before any statistical inference is performed. This introduces several key challenges in machine learning and uncertainty quantification:
\begin{itemize}
    \item Challenges in calibration: Many traditional statistical methods assume access to clean calibration data. However, in an LDP setting, the observed data is randomized, affecting the empirical coverage of conformal prediction intervals.
    \item Impact on distribution-free guarantees: Conformal prediction provides finite-sample marginal coverage guarantees without assumptions on the underlying data distribution. However, when predictions are made using noisy, privatized data, the standard conformal prediction framework may require adaptation to account for the added uncertainty.
\end{itemize}

\chapter{Confidence Calibration under Noisy Labels}
\label{ch:nts}

In this chapter, we explore the critical challenge of confidence calibration in neural networks, particularly under the influence of noisy labels. Confidence calibration is vital in applications such as medical imaging, where overly confident yet incorrect predictions can have serious consequences. We review existing calibration methods, including post-hoc techniques like Temperature Scaling, and discuss the unique difficulties posed by label noise in medical datasets. Finally, we introduce a novel method that leverages noisy validation data to achieve robust calibration, demonstrating its effectiveness across various medical imaging scenarios.

\section{Problem Statement}

Confidence calibration is defined as the ability of a classifier network to provide an accurate probability of correctness for any of its predictions.
Neural networks have been shown to be more overconfident in their predictions than their predecessors even though their generalization accuracy is higher, partly due to the fact that they can overfit on the
negative log-likelihood loss without overfitting on the classification error \cite{Guo2017,Balaji2017,Hein2019}.
 In medical imaging applications, images for which the model makes low-confidence predictions are sent to a physician for review. Skipping the human review based on  confident but incorrect predictions can have disastrous consequences \cite{minderer2021revisiting}. 
The  gap between the model's predicted probabilities and its  accuracy is one of the  key obstacles to
the applicability  of  neural network models to  fully automatic medical diagnosis. 

 Various confidence calibration methods have recently emerged,  aiming to address the issue of excessive overconfidence.  
Network calibration can be performed in conjunction with training (see e.g. 
\cite{mukhoti2020calibrating,
muller2019does,mixup,xu2023mismatch}).
  Post-hoc scaling approaches to calibration (e.g. Platt
scaling \cite{Platt1999}, isotonic regression \cite{Zadrozny2002}, and temperature scaling \cite{Guo2017}) are widely used. 
In order to enhance calibration, they incorporate calibration as a post-processing step, utilizing hold-out validation data to acquire a calibration map that modifies the model's predictions.
Temperature scaling, which is currently the widely accepted practical calibration method, is a straightforward approach that can be easily implemented.
Despite the importance of network calibration for automating medical reports, there are only a few studies that have addressed the problem of calibrating medical imaging systems  (see e.g.
\cite{fernando2021dynamically,frenkel2022calibration,rousseau2021post,zhang2020layer}).

Deep neural networks have been highly successful in various natural image and medical image computing tasks. 
However, these achievements depend on having accurate annotated training data. 
 Neural networks require massive amounts of carefully labeled data to succeed, but acquiring such data is expensive and time-consuming. Non-expert sources, like Amazon's Mechanical Turk, have been used to reduce labeling costs, but their labels can be unreliable. Experienced domain experts may also struggle with complex labeling tasks. 
Medical imaging datasets often have problems with noisy labels due to ambiguous images that can confuse clinical experts.  Physicians may disagree on the diagnosis of the same medical image, resulting in variability in the ground truth label. Furthermore, using Natural Language Processing (NLP) tools to extract labels from radiological reports can also introduce label noise \cite{irvin2019chexpert}. Therefore, addressing annotation noise is a crucial topic in medical image analysis.
 Training neural networks with noisy labels is problematic because the models can easily overfit to the corrupted labels, resulting  lack of generalizability when evaluated on a separate test dataset.
  While popular regularization techniques have been used to address overfitting, they do not entirely solve the problem. Even when these techniques are applied, there is a significant gap in test accuracy between models trained on clean vs. noisy data, and the accuracy decreases with label noise.
Noisy labels are difficult to avoid, and studies indicate that Deep Neural Networks (DNNs) can  memorize entire datasets. Consequently, errors in datasets may result in erroneous predictions, which can impact medical diagnoses. Therefore, effectively managing noisy labels is crucial for automated medical image classification.   A review of network training methods for noisy labels can be found in \cite{9729424} and an excellent up-to-date discussion of training medical image classification networks from data with noisy labels can be found in \cite{Xue2022}.

In the following chapter, we address the challenge of calibrating medical networks with a validation set that has inaccurate labels. Numerous studies have examined the problem of training networks that are resilient to label noise, which can also disrupt the network calibration process.   Our findings suggest that network calibration methods are more susceptible to label noise compared to network training.
Nevertheless, we have not come across any previous research that tackles the challenge of network calibration using a validation set containing noisy labels.  The findings reveal that the Temperature Scaling method \cite{Guo2017}, which is commonly used, is highly susceptible to label noise and can even result in worse calibration than the original model.  We present a simple method that uses data with noisy labels to calibrate a network by taking advantage of the fact that in calibration, we only need to estimate the average accuracy at pre-determined confidence bins, rather than determining the correctness of each label. Testing the method on various medical imaging datasets, network architectures, and noise levels, showed that the calibration results were comparable to those obtained using a noise-free validation set.

The study described in this chapter was published in \cite{penso2024confidence, 10635160}.

\section{Confidence Calibration with Noisy Labels}


Consider a multi-class classification task with $k$ classes. 
Suppose we have a validation set with labels that are potentially inaccurate.
Let $y_1,...,y_n$ be the correct labels of the validation set  and let $\tilde{y}_1,...,\tilde{y}_n$ be the  corresponding  observed corrupted labels.
We assume that the label noise follows a uniform distribution, where with a probability of $\epsilon$,  the correct label is replaced by a randomly selected label from the remaining $(k-1)$ classes. The noise is applied to each sample 
independently.  This noise model is commonly referred to as uniform noise. Our objective is to calibrate the network using these noisy labels.

\begin{figure*}[h!]
    \centering
         \includegraphics[scale=0.35]{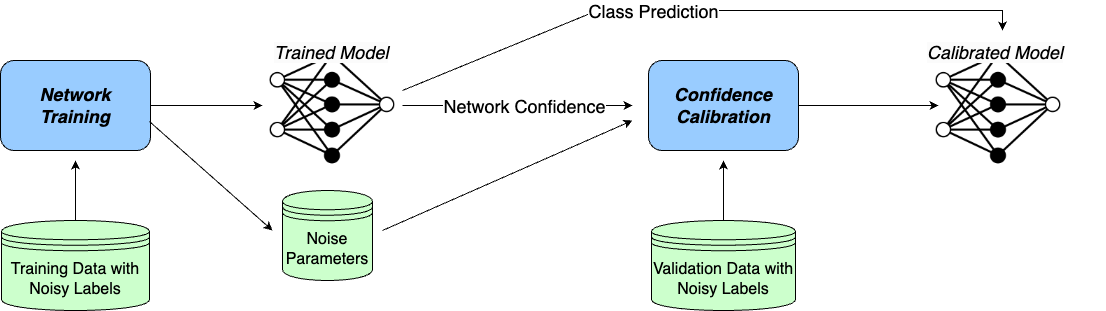}
    \caption{Schema of the proposed model that includes the full pipeline of network training and calibration based on data with noisy labels.}
    \label{fig:nts_schema}
\end{figure*}

It can be verified from the adaECE definition (\ref{adaECEdef}) that only the accuracy terms $\{A_i\}_{i=1}^{m}$ are affected by the noise, whereas the  confidence terms $\{C_i\}_{i=1}^{m}$ remain the same.
Let  $\tilde{A}_i$ be the average accuracy at bin $i$ that is computed using the corrupted labels.
Note that a prediction is considered correct if the label is not corrupted and the network's prediction matches it, or if the label is corrupted and the network's (incorrect) prediction matches the corrupted label. 
This implies that:  \begin{align}
&\tilde{A}_i  =  \frac{1}{|B_i|} \sum_{t \in B_i}
\mathbbm{1}_{\{\hat{y}_t = \tilde{y}_t\}}  
 =\frac{1}{|B_i|} \sum_{\{t \in B_i | \tilde{y}_t = y_t \}} \mathbbm{1}_{\{\hat{y}_t =  {y}_t\}}    +
\frac{1}{|B_i|} \sum_{\{t \in B_i | \tilde{y}_t \ne y_t \}} \mathbbm{1}_{\{\hat{y}_t = \tilde{y}_t\}}.
\nonumber
\label{align:tildea}
\end{align}
The law of large numbers implies that as the size of the validation set increases, the empirical average noisy accuracy at each adaECE bin becomes increasingly close to the mean noisy accuracy.
Therefore:
\begin{equation}\tilde{A}_i \approx (1- \epsilon)
{A}_i + \frac{\epsilon}{k-1}(1-A_i).
\label{lln}
\end{equation}
We can thus obtain an estimation $\hat{A}_i$ of the correct  accuracy $A_i$ from the noisy accuracy $\tilde{A}_i$ as follows:
\begin{equation}\tilde{A}_i = (1- \epsilon)
\hat{A}_i + \frac{\epsilon}{k-1}(1-\hat{A}_i).
\label{lln2}
\end{equation}
Rearranging (\ref{lln2}) we finally obtain:
\begin{equation}
\hat{A}_i = \frac{\tilde{A}_i- \epsilon/(k-1)}{(1-\epsilon) - \epsilon/(k-1)}.
\label{aiform}
\end{equation}
Substituting the estimated accuracy term, based on data with noisy labels (\ref{aiform}) into the adaECE definition (\ref{adaECEdef}),
yields the following noise-robust adaECE measure:
\begin{equation}
\mathrm{Noisy\mbox{-}adaECE}(\epsilon) = \frac{1}{m} \sum_{i=1}^{m} 
 \left| 
\frac{\tilde{A}_i- \epsilon/(k\!-\!1)}{(1\!-\!\epsilon) - \epsilon/(k\!-\!1)}
- C_i \right|.
\label{tildeadaECEdef}
\end{equation}
When training a network with noisy labels we need to make individual decisions for each sample regarding the corruption of its label. The key feature of the Noisy-adaECE metric is that here we only require an estimation of the average label noise within each bin, which is a significantly simpler task.

\begin{algorithm*}
\caption { Noisy Temperature Scaling (NTS) - Uniform noise}
\begin{algorithmic}[0]
\State   {\bf input:} A  validation set    $(x_1,\tilde{y}_1),...,(x_n,\tilde{y}_n)$ whose labels are corrupted with noise level $\epsilon$.
\vspace{0.1cm}
\State - Feed each $x_t$  into the classifier network to produce class distribution  $p_{t1},...,p_{tk}$. 
 Compute \\ \hspace{1cm}  the confidence values
 $\hat{p}_t = \max_j p_{tj}$ and the  predictions
 $\hat{y}_t = \arg \max_j p_{tj}$.
\State - Order the points based on their confidence and divide them into equal-sized sets $B_{1},...,B_{m}$.
\State - Find  $T$ that minimizes the Noisy-adaECE score: 
$$\mathrm{Noisy\mbox{-}adaECE}(T) = \frac{1}{m} \sum_{i=1}^{m} \left| \max(0,\min(1,
\frac{\tilde{A}_i- \epsilon/(k-1)}{(1-\epsilon) - \epsilon/(k-1)}))
- C_i(T) \right| $$
s.t.
$$
C_i(T) =  \frac{1}{|B_i|} \sum_{t \in B_i} 
\frac {\exp(\log (\hat{p}_t/T))} { \sum_{l=1}^k \exp( \log(p_{tl}/T))} 
\hspace{1cm}
\mbox{and} 
\hspace{1cm}
 \tilde{A}_i = 
\frac{1}{|B_i|} \sum_{t \in B_i} \mathbbm{1}_{\{\hat{y}_t = \tilde{y}_t\}}
$$
\State   {\bf output:} \hspace{1cm}
$\hat{T}= \arg \min \mathrm{Noisy\mbox{-}adaECE}(T) $
  \end{algorithmic}
  \label{alg:nts_uniform_noise_alg}
\end{algorithm*}

This Noisy-adaECE calibration measure  assumes  knowledge of the noise level $\epsilon$. If $\epsilon$  is not known, we can estimate it from the noisy data
(see e.g. \cite{patrini2017making,li2021provably,zhang2021learning}).
In the next section, we describe the noise level estimation that was used in our experiments. 
For each calibration method whose parameters can be found by minimizing the adaECE measure, we can form a noise-robust variant in which Noisy-adaECE (\ref{tildeadaECEdef}) is minimized instead of adaECE (\ref{adaECEdef}).   Examples of these calibration methods include Temperature Scaling (TS), Vector Scaling, Matrix Scaling \cite{Guo2017},  Mix-n-Match \cite{zhang2020mix}, Wight Scaling \cite{frenkel2022calibration}, and others.  We next present the noise-robust calibration  measure in the case of the TS method.   The  optimal temperature is obtained by  finding the temperature $T$ that minimizes  the Noisy-adaECE (\ref{tildeadaECEdef}) calibration measure. The proposed robust variant of TS which we dub the Noise-robust Temperature Scaling (NTS) algorithm,  is summarized in Algorithm Box \ref{alg:nts_uniform_noise_alg}.

So far we have considered the simplest uniform label noise model. 
A more general noise model assumes that the true label is corrupted by a label noise matrix  $P$, where
$P(i,j) = p(\tilde{y} = j |y = i)$ is the probability of the true label $i$ being flipped to a corrupted label $j$. 
In the simpler uniform noise model, $P$ takes the form of:  
\begin{equation}
P(i,j) = (1-\epsilon)  \mathbbm{1}_{\{i=j\}}+ \frac{\epsilon}{k-1} \mathbbm{1}_{\{i\ne j\}}.
\label{uniform}
\end{equation}
We next extend the noise-robust calibration measure Noisy-adaECE defined above   to the case of a general label noise matrix  $P$. Denote \[ M_i(r,s)= p_i( \hat{y} = r, {y}=s)=\frac{1}{|B_i|} \sum_{t \in B_i} \mathbbm{1}_{\{\hat{y}_t = r , {y}_t=s\}}.\]
$M_i$ is the classifier confusion matrix computed on the 
validation data from the $i$-th bin using clean labels. 
The adaECE accuracy term (\ref{aidef}) is thus: 
\begin{equation}
A_i = \sum_{j=1}^k M_i(j,j) = \textrm{Tr}(M_i).
\label{aitr}
\end{equation}
In a similar manner,
we define a confusion matrix based on the available noisy labels: 
\begin{equation}
\tilde{M}_i(r,s)= 
\frac{1}{|B_i|} \sum_{t \in B_i} \mathbbm{1}_{\{\hat{y}_t = r , \tilde{y}_t=s\}}.
\label{m_noise}
\end{equation}
According to our noise model,  given a sample from the validation set along with its true label,  the corresponding  noisy label and the network soft prediction are conditionally independent. This implies that:     
\begin{align}\tilde{M}_i(r,s) & = 
p_i( \hat{y} = r, \tilde{y}=s)= \sum_j  p_i( \hat{y} = r, \tilde{y}=s, {y}=j) \nonumber  \\ & 
=  \sum_j  p_i( \hat{y} = r, {y}=j )  p(  \tilde{y}=s | \hat{y}=r, {y}=j) \label{generalp}  \\ & 
=  \sum_j  p_i( \hat{y} = r, {y}=j )  p(  \tilde{y}=s | {y}=j) \label{generalp} 
= \sum_j M_i(r,j) P(j,s). \nonumber
 \end{align}
 We can write (\ref{generalp}) as a matrix multiplication:  $\tilde{M}_i = M_i P$.  This implies that
 \begin{equation}
 M_i = \tilde{M}_i P^{-1}.
\label{mip}
 \end{equation} 
 By substituting  (\ref{mip}) in  (\ref{aitr})  
  we obtain an estimation of the adaECE accuracy term of the clean data $A_i$ as a function of the confusion matrix of the noisy data $\tilde{M}_i$ and the label noise matrix  $P$:
\begin{equation} \hat{A}_i = \textrm{Tr}(M_i) = \textrm{Tr} (\tilde{M}_i P^{-1}).
\label{trai}
\end{equation}
 We note that, by applying this derivation to the case of  uniform  noise (\ref{uniform}), we obtain:
  \begin{align}
 \tilde{A}_i & = \sum_r \tilde{M}_i(r,r) = 
 \sum_r (M_i  P)(r,r) = 
 \sum_{s,r} M_i(r,s) P(s,r)  \nonumber \\ & =
\sum_s ( M_i(s,s)(1-\epsilon) + \sum_{r\ne s} M_i(r,s) (\frac{\epsilon}{k-1}))    = (1-\epsilon) A_i +  \frac{\epsilon}{k-1}(1-A_i).
 \end{align}
This coincides with the direct derivation of binwise average accuracy for the case of uniform noise. 

\begin{algorithm*}
\caption { Noisy Temperature Scaling (NTS) - General Noise Matrix}
\begin{algorithmic}[0]
\State   \textbf{ input:} A  validation set    $(x_1,\tilde{y}_1),...,(x_n,\tilde{y}_n)$ whose labels are corrupted by a noise matrix $P$.
\vspace{0.1cm}
\State - Feed each $x_t$  into the classifier network to produce class distribution  $p_{t1},...,p_{tk}$. 
 Compute \\ \hspace{1cm}  the confidence values
 $\hat{p}_t = \max_j p_{tj}$ and the  predictions
 $\hat{y}_t = \arg \max_j p_{tj}$.
\State - Order the points based on their confidence and divide them into equal-sized sets $B_{1},...,B_{m}$.
\State - Find  $T$ that minimizes the Noisy-adaECE score: 
\[\mathrm{Noisy\mbox{-}adaECE}(T) = \frac{1}{m} \sum_{i=1}^{m} \left| \max(0,\min(1,
\textrm{Tr} (\tilde{M}_i P^{-1})))
- C_i(T) \right| \]
s.t.
\[
C_i(T) =  \frac{1}{|B_i|} \sum_{t \in B_i} 
\frac {\exp(\log (\hat{p}_t/T))} { \sum_{l=1}^k \exp( \log(p_{tl}/T))} 
\hspace{1cm}
\mbox{and} 
\hspace{1cm}
\tilde{M}_i(r,s)= \frac{1}{|B_i|} \sum_{t \in B_i} \mathbbm{1}_ {\{\hat{y}_t = r , \tilde{y}_t=s\}}
\]
\State   \textbf{ output:} \hspace{1cm} 
$\hat{T}= \arg \min \mathrm{Noisy\mbox{-}adaECE}(T) $
  \end{algorithmic}
  \label{alg:nts_general_noise_alg}
\end{algorithm*}

The Noise-adaECE in the case of a general  noise matrix $P$ is defined by:
\begin{equation}
\mathrm{Noisy\mbox{-}adaECE}(P) = \frac{1}{m} \sum_{i=1}^{m} \left| 
\textrm{Tr} (\tilde{M}_i P^{-1})
- C_i \right|,
\label{madaECEdef}
\end{equation}
such that $\tilde{M}$ is the confusion matrix of the noisy validation set data (\ref{m_noise}) and we used (\ref{trai}) to estimate the binwise average clean data accuracy.  To apply TS calibration in the case of a general noise model, we need to find a temperature $T$ that minimizes the Noise-adaECE expression (\ref{madaECEdef}). 
In the case where the label noise matrix $P$  is not known, there is a plethora of methods for estimating $P$ without accessing clean labels \cite{patrini2017making,li2021provably,zhang2021learning}. 
The Noisy Temperature Scaling (NTS) for the case of a general noise matrix, is summarized in Algorithm Box \ref{alg:nts_general_noise_alg}. 
Figure \ref{fig:nts_schema} illustrates the entire pipeline composed of the noisy label training followed by the noisy label calibration process.

\label{sec:nts_network_training}
\section{Network Training with Noisy Labels}

To achieve robust confidence calibration, it is crucial to first train the network in a manner that accounts for label noise. Various approaches have been developed to mitigate the adverse effects of noisy labels, ranging from reweighting and selecting reliable samples to modifying the network architecture to model label noise explicitly. Among these, a particularly effective strategy involves estimating the label noise matrix and incorporating it into the training process (see Background Section \ref{s:noisy} for details).

Our approach follows the provably consistent method proposed by Li et al. \cite{li2021provably}, where a noise adaptation layer is introduced to learn the label transition probabilities. This method optimizes a loss function that balances prediction fidelity with a regularization term enforcing the structural properties of the estimated noise matrix. Once training is complete, the noise adaptation layer is removed, allowing the network to predict clean labels while retaining an estimated noise matrix as a by-product. This estimated noise matrix plays a key role in our confidence calibration procedure.

By leveraging noise-robust training, we ensure that the network maintains reliable predictive confidence even when exposed to datasets with significant label noise. In the following sections, we detail how this estimated noise matrix enhances the calibration process.

Figure \ref{fig:nts_schema} provides a flow diagram of our method in conjunction with model training. A noise-robust training method \cite{li2021provably}
is applied to the noisy training data, yielding a trained network and an estimation of the noise matrix. Then, given a noisy validation set, we apply our NTS method to calibrate the network confidence.

\newpage
\newpage
\section{Experiments}


We implemented the proposed NTS noisy calibration method on various medical imaging classification tasks to evaluate its  performance. We share our code for reproducibility~\footnote{
\url{{https://github.com/cobypenso/noisy_calibration}}}. The experimental setup included the following medical imaging classification datasets:
\begin{itemize}
\item \textbf{ChestX-ray14} \cite{wang2017chestx}: A    huge dataset that contains 112,120 frontal-view X-ray images of 30,805 unique patients of size $1024 \times 1024$, individually labeled with up to 14 different thoracic diseases. The original dataset is multi-label. 
The problem is treated as a multi-class task by choosing the samples containing only one annotated positive label or the "No-finding" case without any positive label. 
More than 60\% of the images belong to this class, which makes the dataset highly imbalanced.
 We used a train/validation/test split of 89,696/11,212/11,212 images.
We used a balanced variant of the dataset denoted by ChestX-ray14+bal which only included images with exactly one  annotated positive label.
\item \textbf{HAM10000} \cite{tschandl2018ham10000}: This dataset contains 10,015 dermatoscopic images of size $800 \times 600$. Cases include a representative collection of 7 diagnostic categories in the realm of pigmented lesions. We used a train/validation/test split of 8,013/1,001/1,001 images.
\item \textbf{PathMNIST} \cite{medmnistv2}: A dataset that contains 97,176 images of Colon Pathology with nine classes. The images'  size is $28 \times 28$. Here, we used a train/validation/test split of 89,996/3,590/3,590 images.
\end{itemize}

Each dataset was fine-tuned on pre-trained ResNet-18, ResNet-50 \cite{he2016cvpr}, and DenseNet-121 \cite{huang2017densely} networks. The models were taken from the PyTorch site \footnote{
\url{https://pytorch.org/vision/stable/models.html}}.
These network architectures were selected because of their widespread use in classification problems.
The last fully-connected  layer output size of each  was adjusted to fit the corresponding number of classes for each dataset. All the models were fine-tuned using the   Adam optimizer \cite {kingma2014adam}.

\begin{figure*}
    \includegraphics[width=0.95\textwidth]{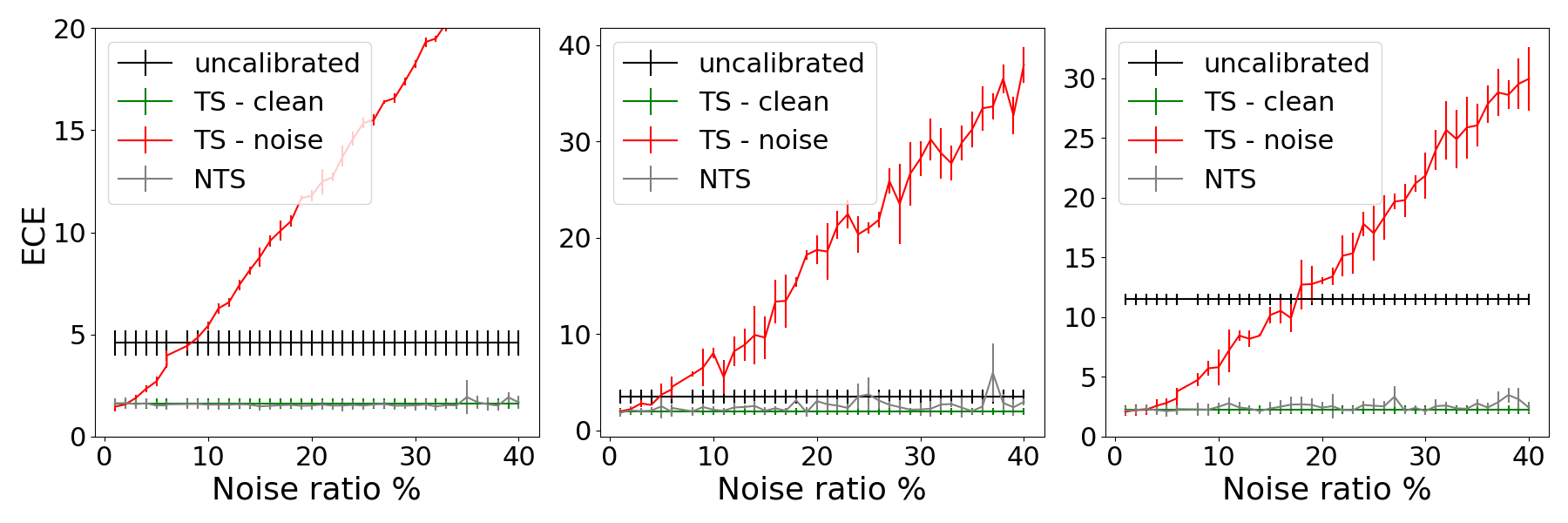}
    \vspace{0.3cm}
    \includegraphics[width=0.95\textwidth]{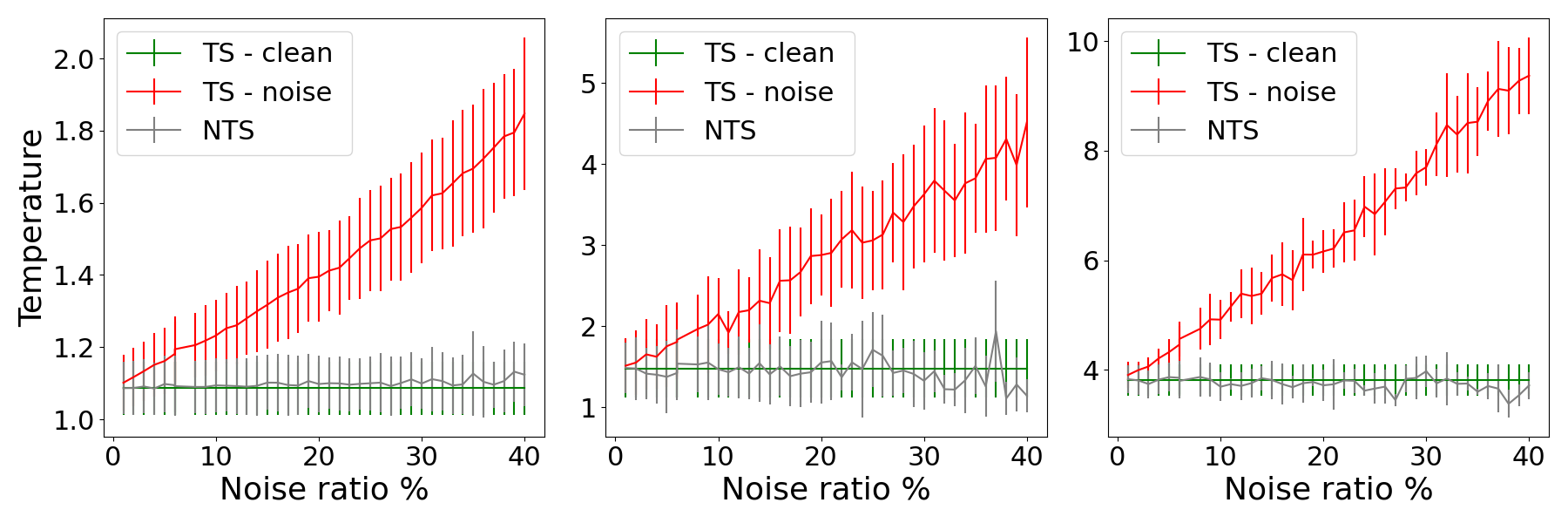}
    \center
    \hspace{0.3cm} (a) ChestX-ray14 
    \hspace{3.1cm} (b) HAM-10000 
    \hspace{3.1cm} (c) PathMNIST
    \caption{Comparative calibration results on   several datasets that were trained with ResNet-50.   The top row shows the ECE
    results on the (clean) test set. 
    The bottom row shows the optimal temperature found on the noisy validation set. }
    \label{fig:different_epsilons}
\end{figure*}

\begin{figure*}[t]
    \centering
    \includegraphics[width=1 \textwidth]{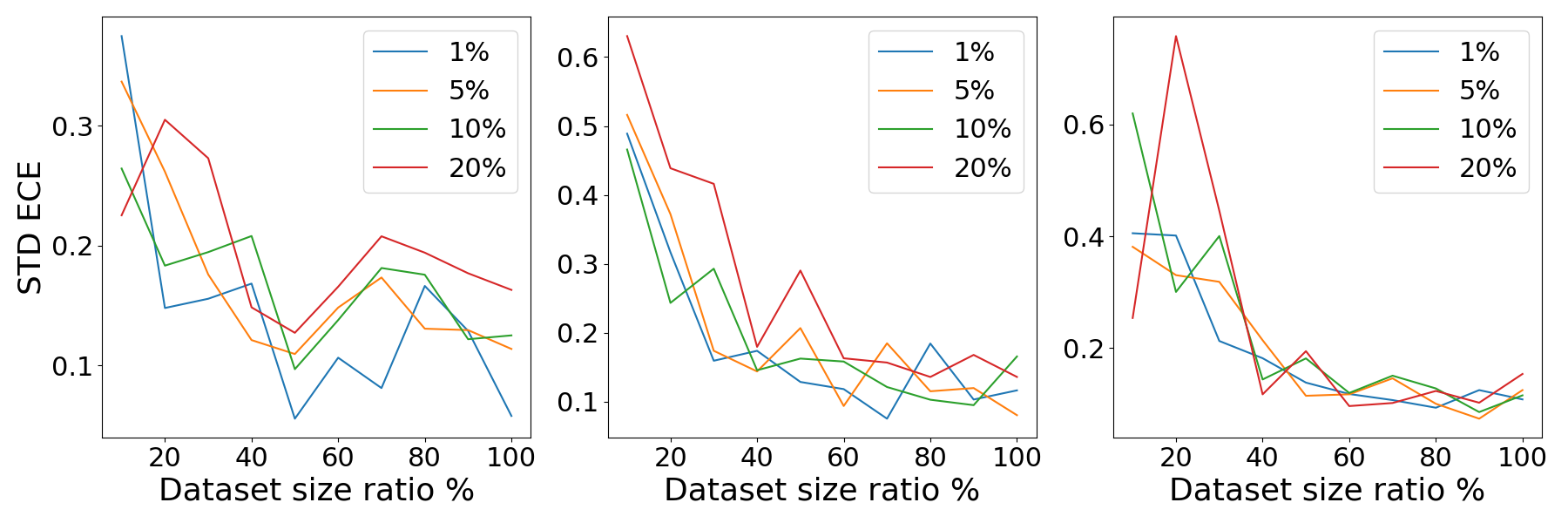} 
      \hspace{1.2cm}  (a) ResNet-18 
    \hspace{3.1cm} (b) ResNnet-50
    \hspace{3.1cm} (c) Densenet-121
    \caption{ Standard deviation of ECE scores on the test set after the NTS calibration as a function of the size of the noisy validations set.}
    \label{fig:dataset_size}
\end{figure*}

\begin{table*}
	\caption{Adaptive ECE for top-1 predictions (in \%) using 15 bins.
   on various medical imaging classification datasets and models with different calibration methods with a noise level $\epsilon=0.2$.  The lowest is highlighted in bold.}
	\label{table:adaece_tab1}
	\centering
          \resizebox{\linewidth}{!}{%
		\begin{tabular}{ll|c|cc|cccc}
			\toprule 
			{\small Dataset}& {\small Architecture}& { Acc (\%)} & {\small Uncalibrated} & {TS$_{clean}$}& VS$_{noisy}$ & MS$_{noisy}$ & 
			\hspace{1mm}{TS$_{noisy}$}\hspace{3mm}& \hspace{3mm}{NTS} \hspace{3mm} \\
 \midrule
                \multirow{3}{*}{ChestX-ray14} & ResNet-18 &  65.80 $\pm$ 0.01 &$1.92 \pm 0.43$ &$1.52 \pm 0.19$& $12.25\pm 0.17$ & $12.27\pm0.17$ & 11.90 $\pm$ 0.46 & \textbf{1.45 $\pm$ 0.15}\\
                & ResNet-50 &  $65.84 \pm 0.02$  &  $4.59 \pm 0.62$ & $1.63 \pm 0.25$  & $12.26\pm0.21$ & $12.22\pm0.18$ & $11.81 \pm 0.28$& \textbf{1.52 $\pm$ 0.19}\\
                & Densenet-121 &  $65.67 \pm 0.01$  &  $3.61 \pm 0.75$ &  $1.49 \pm 0.10$ & $12.27\pm 0.18$ & $12.24\pm0.16$ & 12.15 $\pm$ 0.49& \textbf{1.50 $\pm$ 0.06} \\  
			\midrule
                \multirow{3}{*}
                {HAM-10000} & ResNet-18 &   $88.16 \pm 0.17$ &  $4.38 \pm 1.52$ &   $3.12 \pm 0.74$ & $19.75\pm2.40$ & $17.17\pm0.91$ & $14.22 \pm 5.88$& \textbf{3.67 $\pm$ 0.59}\\
                & ResNet-50 &   $91.06 \pm 0.01$ &$3.51 \pm 0.69$&  $1.98 \pm 0.34$ & $19.95\pm0.73$ & $18.03\pm0.54$ & 18.75 $\pm$ 1.49 &\textbf{3.04 $\pm$ 1.08}\\
                & Densenet-121 &  $89.43 \pm 0.18$  &  $3.75 \pm 1.53$ & $2.67 \pm 0.11$& $18.85\pm1.23$ & $17.56\pm1.09$ & $17.50 \pm 2.04$& \textbf{2.74 $\pm$ 0.68}  \\ 
                \midrule \multirow{3}{*}
                {PathMNIST} & ResNet-18 &  $84.68 \pm 0.05$  &  13.14 $\pm$ 0.51&  $2.19 \pm 0.20$ & $23.19\pm0.57$ & $22.99\pm0.05$ &  11.94 $\pm$ 1.30 & \textbf{2.37 $\pm$ 0.52}\\
                & ResNet-50 &   $85.97 \pm 0.04$ &  11.50 $\pm$ 0.44& $2.23 \pm 0.35$  & $21.47\pm3.71$ & $27.35\pm3.96$ & $13.00 \pm 0.32$ & \textbf{2.42 $\pm$ 0.48} \\
                & Densenet-121 & $86.54 \pm 0.05$   &  $11.12 \pm 0.43$ & $1.83 \pm 0.49$  &  $25.47\pm0.71$ &$ 25.58\pm0.42$ & 12.31 $\pm$ 1.00 & \textbf{2.55 $\pm$ 0.35} \\
                			\bottomrule
		\end{tabular}%
          }
	\end{table*}

\begin{table*}
	\caption{ECE for top-1 predictions (in \%) using 15 bins.
   on various medical imaging classification datasets and models with different calibration methods with a noise level $\epsilon=0.2$.  The lowest is highlighted in bold.}
	\label{table:ece_tab1}
	\centering
          \resizebox{\linewidth}{!}{%
		\begin{tabular}{ll|cc|cccc}
			\toprule 
			{\small Dataset}& {\small Architecture}&  {\small Uncalibrated} & {TS$_{clean}$}& VS$_{noisy}$ & MS$_{noisy}$ & 
			\hspace{1mm}{TS$_{noisy}$}\hspace{9mm}& \hspace{9mm}{NTS} \hspace{9mm} \\
      \midrule
              
                        \multirow{3}{*}{ChestX-ray14} & ResNet-18 & $1.93 \pm 0.49$ &$1.24 \pm 0.14$&  $12.25\pm0.17$ & $12.27\pm0.18$ & 11.92 $\pm$ 0.46 & \textbf{1.21 $\pm$ 0.12}\\
                & ResNet-50 &  $4.51 \pm 0.30$ & $1.35 \pm 0.32$  & $12.24\pm0.22$ &$ 12.21\pm0.20$ & $11.84 \pm 0.28$& \textbf{1.26 $\pm$ 0.29}\\
                & Densenet-121 &  $3.71 \pm 0.80$ &  $1.22 \pm 0.19$ &  $12.25\pm0.18$ &$ 12.24\pm0.17$ & 12.10 $\pm$ 0.49& \textbf{1.26 $\pm$ 0.06} \\  
			\midrule
                \multirow{3}{*}{HAM-10000} & ResNet-18 &  $2.72 \pm 1.41$ &   $0.79 \pm 0.35$ &  $18.84\pm1.61$ &$ 14.87\pm 4.01$ & $12.53 \pm 6.51$& \textbf{2.52 $\pm$ 0.69}\\
                & ResNet-50 &  $1.69 \pm 0.61$&  $0.61 \pm 0.12$ &  $19.00\pm0.14$ &$ 17.50\pm0.88$ & 16.50 $\pm$ 2.41 &\textbf{1.41 $\pm$ 1.10}\\
                & Densenet-121 &  $3.04 \pm 1.67$ & $1.23 \pm 0.34$&  $18.39\pm1.74$ &$ 17.13\pm1.24$ & $16.23 \pm 2.13$& \textbf{1.72 $\pm$ 0.74}\\ \midrule          
                \multirow{3}{*}{PathMNIST} & ResNet-18 & 13.12 $\pm$ 0.52&  $1.89\pm 0.11$ &   $23.22\pm0.59$ &$ 23.00\pm0.06$ & 11.91 $\pm$ 1.35 & \textbf{2.09 $\pm$ 0.46}\\
                & ResNet-50 &   11.50 $\pm$ 0.52& $2.30 \pm 0.34$  &  $21.31\pm3.93$ &$ 27.14\pm 4.75$ & $12.90 \pm 0.11$ & \textbf{2.52 $\pm$ 0.49} \\
                & Densenet-121 & $11.02 \pm 0.40$ & $1.82 \pm 0.54$  &  $25.29\pm0.67$ &$ 25.58\pm 0.44$ & 12.02 $\pm$ 0.79& \textbf{2.58 $\pm$ 0.46} \\
               			\bottomrule
		\end{tabular}%
          }
	\end{table*}

\begin{table*}[!t]
	\caption{Adaptive ECE for top-1 predictions (in \%) using 15 bins (with the lowest in bold) on various medical imaging classification datasets and models with different calibration methods with varying noise levels. Model training and noise matrix  estimation used \cite{li2021provably}.}
	\label{table:adaece_tab2}
	\centering
          \resizebox{\linewidth}{!}{%
		\begin{tabular}{lc|c|cc|c|cc}
			\toprule 
			{\small Dataset}& {\small Noise level (\%)}& { Acc (\%)} & {\small Uncalibrated} & {TS$_{clean}$}& NTS($P$) 
			  & {TS$_{noisy}$} &  NTS($\hat{P}$) \\ \midrule
\multirow{4}{*}{ChestX-ray14-bal} & 
        $0$      & $44.37 \pm 0.12$ &  $26.21 \pm 0.56$ & $2.88 \pm 0.48$ && & \\
       &  $5$    & $44.10 \pm 0.08$ &  $25.52 \pm 0.79$ & $3.18 \pm 0.44$ &  {3.20 $\pm$ 0.46} & $5.91 \pm 0.46$ & \textbf{5.27 $\pm$ 1.41} \\
      &   $10$   & $43.76 \pm 0.31$ & $23.85 \pm 1.02$ & $3.41 \pm 0.30$ & {3.40 $\pm$ 0.54}  & $7.35 \pm 1.18$ & \textbf{6.08 $\pm$ 1.52} \\
         & $20$  & $43.02 \pm 0.28$ & $21.34 \pm 0.33$ & $3.30 \pm 0.36$ & {3.33 $\pm$ 0.44}  & $10.48 \pm 0.63$ & \textbf{6.43 $\pm$ 0.96}\\
\midrule   
    \multirow{4}{*}{ChestX-ray14} & 0 & $65.13 \pm 0.09$   &  $21.50 \pm 0.18$&  $5.70 \pm 0.63$ & & & \\
                & 5 &  $64.95 \pm 0.05$  & $18.71 \pm 0.20$ & $4.46 \pm 0.67$& {4.29 $\pm$ 0.44} & \textbf{5.86 $\pm$ 0.49} &  $12.8 \pm 1.34$ \\
                & 10 &  $64.76 \pm 0.36$  & $19.04 \pm 2.85$ &$4.34 \pm 0.55$ & {4.42 $\pm$ 0.33} & \textbf{6.81 $\pm$ 0.64} &  $9.65 \pm 0.78$ \\
                & 20 & $64.71 \pm 0.09$   & $18.22 \pm 3.26$ & $4.64 \pm 0.54$  & {4.64 $\pm$ 0.54} & 12.02 $\pm$ 1.74 & \textbf{7.28 $\pm$ 1.75}\\
			\midrule 
                \multirow{4}{*}{HAM-10000} & 0 & $90.94 \pm 0.78$ & $3.95 \pm 0.67$ &  $2.48 \pm 1.02$ & & & \\ 
                & 5 &  $90.64 \pm 0.83$ &  $3.47 \pm 1.16$& $3.19 \pm 1.03$& {2.98 $\pm$ 0.49} & $4.34 \pm 0.57$&  \textbf{ 3.41 $\pm$ 1.21}\\
                & 10 & $90.16 \pm 0.81$ & $2.95 \pm 0.32$ & $2.70 \pm 0.39$&  {2.71 $\pm$ 0.29} &$8.08 \pm 1.12$ & \textbf{  2.76 $\pm$ 0.31 }\\
                & 20 & $89.79 \pm 1.11$ &$4.78 \pm 0.67$ & $3.56 \pm 0.72$ &  {4.44 $\pm$ 0.97} & $19.12 \pm 1.28$ & \textbf{  6.52 $\pm$ 2.35 }\\
                \midrule 
                \multirow{4}{*}{PathMNIST} & 0 &  $85.33 \pm 0.99$  &  $4.52 \pm 0.75$&  $1.39 \pm 0.61$ &&& \\ 
                & 5 & $85.30 \pm 0.67$ & $3.71 \pm 0.49$ &$1.73 \pm 0.28$ & {1.65 $\pm$ 0.26} &$4.08 \pm 0.37$& \textbf{ 1.86 $\pm$ 0.37 }\\
                & 10 &  $85.26 \pm 0.55$ & $3.73 \pm 0.71$ & $1.67 \pm 0.62$& {1.56 $\pm$ 0.59} &$7.09 \pm 0.61$  & \textbf{  2.44 $\pm$ 0.92} \\
                & 20 & $84.46 \pm 0.57$ & $3.36 \pm 0.63$ & $1.73 \pm 0.16$  &  {2.33 $\pm$ 0.82} & $12.30 \pm 1.34$ & \textbf{ 2.63 $\pm$ 0.19}\\
              			\bottomrule
		\end{tabular}%
          }
	\end{table*}

\begin{table*}[!t]
	\caption{ECE for top-1 predictions (in \%) using 15 bins (with the lowest in bold) on various medical imaging classification datasets and models with different calibration methods with varying noise levels. Model training and noise matrix  estimation used \cite{li2021provably}.}
	\label{table:ece_tab2}
	\centering
          \resizebox{\linewidth}{!}{%
		\begin{tabular}{lc|cc|c|cc}
			\toprule 
			{\small Dataset}& {\small Noise level (\%)} & {\small Uncalibrated} & {TS$_{clean}$}& NTS($P$)  &
			TS$_{noisy}$  &  NTS($\hat{P})$ \\
\midrule
          \multirow{4}{*}{ChestX-ray14-bal} & 0  &  $26.21 \pm 0.56$&   $2.63 \pm 0.32$& && \\ 
                & 5 &  $25.53 \pm 0.80$ & $3.03 \pm 0.37$&  $3.08 \pm 0.38$& $4.89 \pm 0.82$&  \textbf{3.83 $\pm$ 0.36}\\
                & 10 &  $23.87 \pm 1.01$ & $3.00 \pm 0.61$&  $2.92 \pm 0.46$&  $7.49 \pm 0.73$&   \textbf{4.13 $\pm$ 0.39}\\
                & 20 &  $21.37 \pm 0.35$& $3.33 \pm 0.37$ &   $3.26 \pm 0.28$&  $9.78 \pm 0.98$&  \textbf{4.94 $\pm$ 0.19}\\
   \midrule
          \multirow{4}{*}{ChestX-ray14} & 0  &  $21.05 \pm 0.55$&  $5.86 \pm 0.82$ & 
          & & \\
                & 5 &   $18.73 \pm 0.18$ & $4.43 \pm 0.71$&  {4.25 $\pm$ 0.49} &  \textbf{5.74 $\pm$ 0.68}& $12.9 \pm 1.44$ \\
                & 10 &   $19.11 \pm 2.85$ &$4.21 \pm 0.75$ & {4.30 $\pm$ 0.62} &  \textbf{7.06 $\pm$ 0.22}  & $10.0 \pm 1.27$ \\
                & 20 &  $16.24 \pm 2.02$ & $4.95 \pm 1.01$  & {4.98 $\pm$ 1.08} & $11.6 \pm 0.93$ &   \textbf{5.62 $\pm$ 0.79}\\
			\midrule 
                \multirow{4}{*}{HAM-10000} & 0 & $1.83 \pm 0.27$ &  $1.19 \pm 0.59$ 
                & & \\ 
                & 5 &  $1.54 \pm 0.38$ & $1.31 \pm 0.38$&  {2.03 $\pm$ 0.94} &  $4.10 \pm 0.87$&\textbf{ 1.31 $\pm$ 0.33} \\
                & 10 &  $1.53 \pm 0.13$ & $1.41 \pm 0.43$&  {2.35 $\pm$ 0.47}  & $6.61 \pm 1.56$ & \textbf{ 1.40 $\pm$ 0.68} \\
                & 20 & $3.49 \pm 0.80$ & $0.94 \pm 0.30$ &   {3.30 $\pm$ 1.68} & $12.3 \pm 8.82$ & \textbf{ 5.37 $\pm$ 2.02} \\
                \midrule 
                \multirow{4}{*}{PathMNIST} & 0 &    $4.48 \pm 0.76$&  $1.58 \pm 0.04$ & & & \\ 
                & 5 &  $4.00 \pm 0.01$ &$1.40 \pm 0.00$ & {1.65 $\pm$ 0.01} & $3.91 \pm 0.00$& \textbf{1.34 $\pm$ 0.01}\\
                & 10 &   $3.82 \pm 0.45$ & $1.73 \pm 0.37$& {1.65 $\pm$ 0.17} & $6.41 \pm 0.80$  & \textbf{1.68 $\pm$ 0.39} \\
                & 20 &  $3.42 \pm 0.44$ & $1.97 \pm 0.51$  & {2.42 $\pm$ 0.41} &  $12.5 \pm 1.36$ & \textbf{2.57 $\pm$ 0.29}\\
             			\bottomrule
		\end{tabular}%
          }
	\end{table*}

\textbf{The compared calibration methods}. (1)  TS$_{noisy}$ - A standard TS that was applied to the validation set with noisy labels.
(2) VS$_{noisy}$ and MS$_{noisy}$ -  Vector Scaling (VS) and  Matrix Scaling (MS) calibration methods \cite{Guo2017}  were applied to the validation set with noisy labels.
(3) The proposed Noisy Temperature Scaling (NTS).
 (4) TS$_{clean}$ - Temperature Scaling that was applied to  the clean validation set.
The main comparison was between the  TS$_{noisy}$ and NTS that were applied to the noisy  validation dataset  with the same label corruption.  TS$_{clean}$ is served as an upper bound on the  performance of NTS. 
We are not aware of any other noise robust calibration strategies. 
  For each method, we report the adaECE and ECE scores (computed using 15 bins) on the test set. Although adaECE was used as the objective function in our algorithm and is a more robust calibration measure, ECE is still a standard measure to evaluate  calibration results, so we also used it to compare our calibration results to previous studies.

\textbf{Results of calibration  with uniform noise.}
We first applied the calibration methods to a validation set corrupted by noise with  noise level $\epsilon=0.2$.  
Table \ref{table:adaece_tab1} reports the calibration results using adaECE. We report the mean and standard deviation over 3 different trained models and noise labels sampling. The results indicate that when the validation set contains noise, the standard TS$_{noisy}$ method fails to properly calibrate the network and may make the calibration worse. The same behavior was observed for VS$_{noisy}$ and NTS$_{noisy}$.
By contrast,  the proposed NTS method  produces calibration results that are similar to those obtained by the TS$_{clean}$ method that has access to the clean data. we tested the assumption that (in the case of noise level $\epsilon=0.2$) NTS has no effect compared to TS$_{noisy}$ and its  $p$-value was 0.005. 
Note that since  NTS works well, the temperatures computed by NTS and TS$_{clean}$ are similar so that  the performance of NTS can  even be slightly better than TS$_{clean}$ on the test set.

In this experiment we assume that the noise level is known and the network was trained using clean data. Below we show the results of end-to-end experiments where the noise matrix is estimated during the training procedure.

We next  show that our noise-tolerant NTS method works well across various levels of noise.
Figure \ref{fig:different_epsilons} displays the adaECE calibration results for various classification tasks 
with noise levels $\epsilon$ that ranged from 0\% to 40\%.
 While higher noise levels negatively impacted the calibration of TS$_{noisy}$, NTS demonstrated resilience even in the presence of high levels of noise. 
 Figure \ref{fig:different_epsilons} also illustrates the optimal temperature obtained for each experiment. It shows that the optimal temperature found by TS$_{noisy}$ increased in a linear manner with the noise level. The presence of noisy labels led to an underestimation of accuracy in each adaECE bin, causing the TS$_{noisy}$ algorithm to incorrectly conclude that the network was over-confident, thus resulting in an aggressive calibration using a high temperature.
Our proposed NTS algorithm avoids this misinterpretation, by computing an accurate  estimation of the network accuracy.
It is worth mentioning that typical neural network training techniques can handle a limited amount of incorrectly labeled data in the training set, typically less than 10\% (see e.g. \cite{Xue2022}). However, TS$_{noisy}$ does not perform well under these conditions.

Our method is justified by the law of large numbers (see (\ref{lln})). Therefore, we expected a correlation between the size of the validation set used for calibration and the stability of the NTS method when dealing with different noise samples. 
To verify this, we selected ChestX-ray14 as a dataset from our experiments and generated various validation set sizes (using the sklearn split function). For each sample size, and noise level $\epsilon$, we created multiple  noisy versions of the validation set. For each noisy version we applied the NTS algorithm and computed the adaECE score on the test set. Finally, we computed the standard deviation (STD) of all the adaECE scores.
The results, as shown in Figure \ref{fig:dataset_size}, indicate that as the validation set size increased, the standard deviation of the NTS algorithm decreased, thus resulting in more stable model accuracy estimates for each bin.

{\bf End-to-End training and calibration with Noisy labels.}
In the following experiment, we combined our noise-robust calibration method with the  method for network training using  training data with noisy labels described in the previous section \cite{li2021provably}. Thus, we simulated a real-world scenario in which the labels of both the training set and the validation set had the  same noise level. Making it even more realistic, we assumed here that the label noise  matrix $P$ is unknown and was estimated during the training step
\cite{li2021provably}. 
We trained  models for the three medical-imaging  datasets (ChestX-ray14, HAM-10000, PathMNIST), using training data with  label noise level $\epsilon \in \{0\%, 5\%, 10\%, 20\%\}$. In addition to training the network, we also estimated the noise matrix  $P$, denoted as $\hat{P}$.  We report the results of two variants of our calibration method NTS;  namely,    NTS($P$) where $P$ is known and NTS($\hat{P}$) where $P$ is estimated during the training phase from the noisy training set.
Tables \ref{table:adaece_tab2} and \ref{table:ece_tab2} report the calibration results using adaECE and ECE respectively.
The results show that the noisy labels corrupted  the calibration of the network, whereas our method, even in the case where $P$ was estimated achieved a much better calibration result. We can also observe that the network training procedure is much more resilient to label noise compared to calibration. A noise level of 20\% only results in a slight degradation of the network's accuracy. However, the calibration process does not provide any assistance and instead significantly reduces the network's calibration.
The ChestX-ray14 task contains a dominant "no-finding"
class in addition to the 14 pathology-related classes and is heavily unbalanced.
In that case, when the true noise matrix $P$
was used we obtained good calibration results but when the estimated noise matrix was used, the calibration procedure failed. The reason for that is that in  ChestX-ray14, the classes are unbalanced, and in that case there is currently no reliable way to estimate the noise matrix. When The No-finding class is removed (dataset 
ChestX-ray1-bal) the dataset is relatively balanced and the training algorithm manages to estimate the noise matrix.
Finally, results satisfy the statistical significance requirement by a large margin, with an average p-value across the 4 different datasets and $\epsilon=0.2$ of $P-value=0.00508$

{\bf Calibration results for a general noise  matrix.}
So far, we conducted experiments with  uniform noise.  Next, we evaluate our NTS method on  a validation set that was corrupted by a general label noise matrix $P$, which allows all types of noise to affect  the labels.
We tested four cases that were different in terms of their label noise matrix $P$: (1) Symmetric noise where labels are  converted to any other label with equal probability, (2) Neighbor noise where labels can only be flipped to adjacent labels, (3) Decreasing noise where labels are flipped to adjacent labels with decreasing probability and finally (4) Random noise - a general noise matrix whose off-diagonal probabilities were randomly selected. 
The same noise model was applied to the training and validation sets. Figure \ref{fig:general_noise} shows the noise transition matrices and the corresponding adaECE calibration results on the test set for the baseline calibration method  TS$_{noisy}$  and NTS.  
  We report the results of two variants of our calibration method NTS;  namely,    NTS($P$) where the noise matrix $P$ is known and NTS($\hat{P}$) where $P$ is estimated during the training phase from the noisy training set. We also report the results of TS$_{clean}$ where calibration was applied to the clean validation set.
In all cases, regardless of the noise matrix, NTS achieved calibration results that were on par with the results obtained by applying  TS$_{clean}$ on the clean validation set and much better than the results obtained by TS$_{noisy}$.

\begin{figure*}
    \centering
         \includegraphics[, trim= 10 0 720 0,clip,scale=0.16]{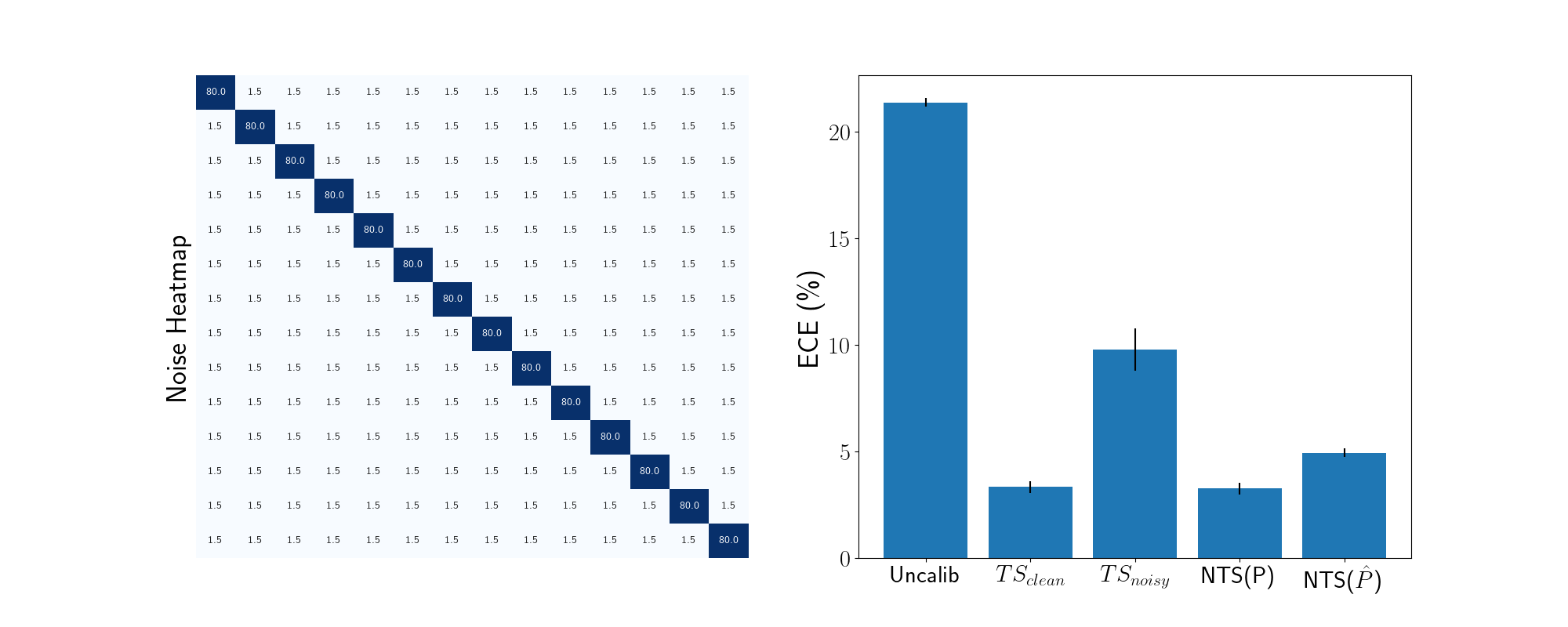}
     \includegraphics[, trim=  730 0 100 0,clip,scale=0.16]{images/TMI/images/general_noise/cxr14_only_patology_14_uniform_noise_img_with_std.png}     
     \includegraphics[, trim= 10 0 720 0,clip,scale=0.16]{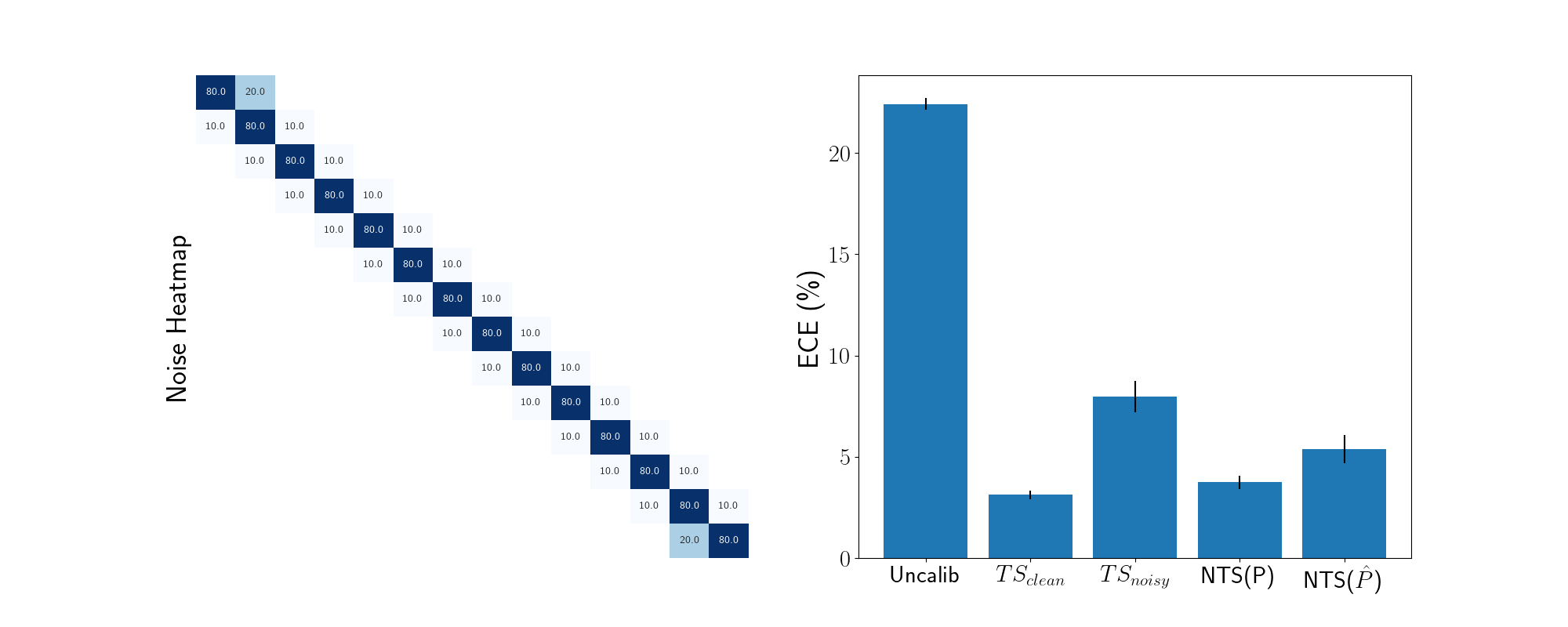}
     \includegraphics[, trim=  730 0 100 0,clip,scale=0.16]{images/TMI/images/general_noise/cxr14_only_patology_14_neighbors_noise_img_with_std.png} \hspace{0.1cm}\\
    (1) symmetric noise \hspace{6cm} (2)  neighbor  noise
       \\
     \centering
       \includegraphics[, trim= 10 0 720 0,clip,scale=0.16]{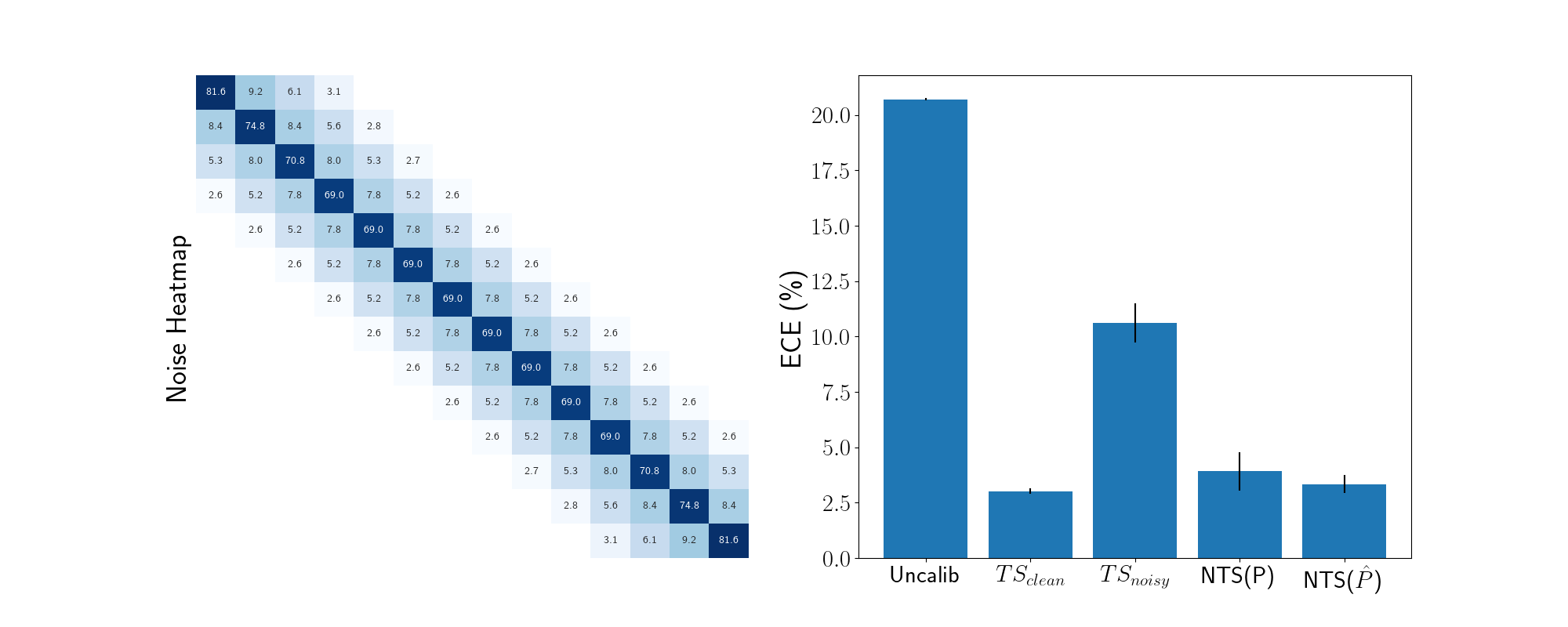}
     \includegraphics[, trim=  730 0 100 0,clip,scale=0.16]{images/TMI/images/general_noise/cxr14_only_patology_14_decremental_noise_img_with_std.png}  
    \includegraphics[, trim= 10 0 720 0,clip,scale=0.16]{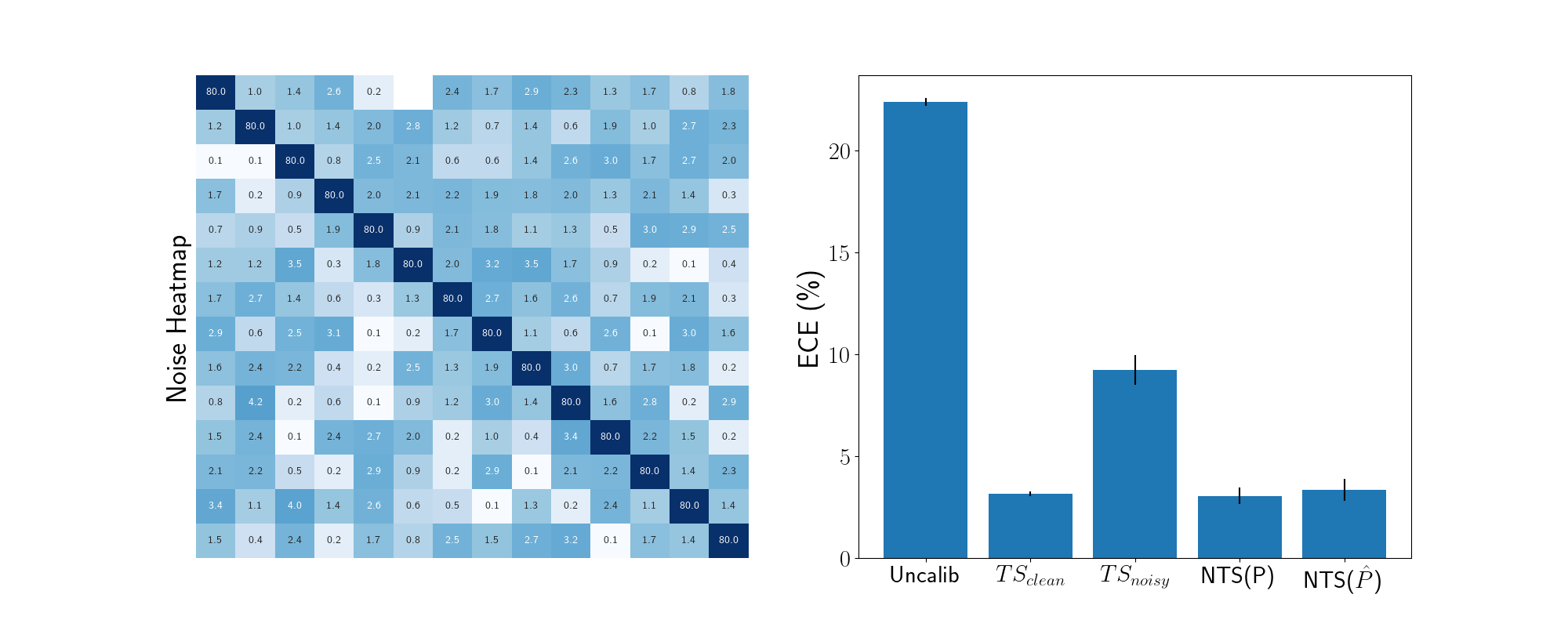}
     \includegraphics[, trim=  730 0 100 0,clip,scale=0.16]{images/TMI/images/general_noise/cxr14_only_patology_14_random_noise_img_with_std.png} \hspace{0.1cm} 
    \\
    (3) decreasing noise \hspace{6cm} (4)  random  noise
    \caption{Calibration results for several noise transition matrices on ChestX-ray14-bal and ResNet-18. In each example we show the noise transition matrix   (Left) and the adaECE measure on the test set  for the compared calibration methods (right).}
    \label{fig:general_noise}
\end{figure*}

\begin{figure*}
    \centering
         \includegraphics[scale=0.35]{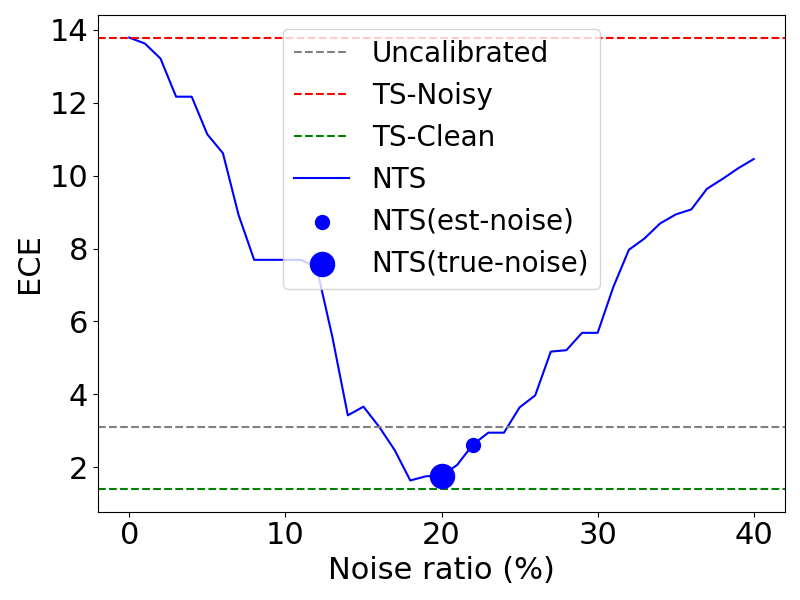} \hspace{2mm}
     \includegraphics[scale=0.35]{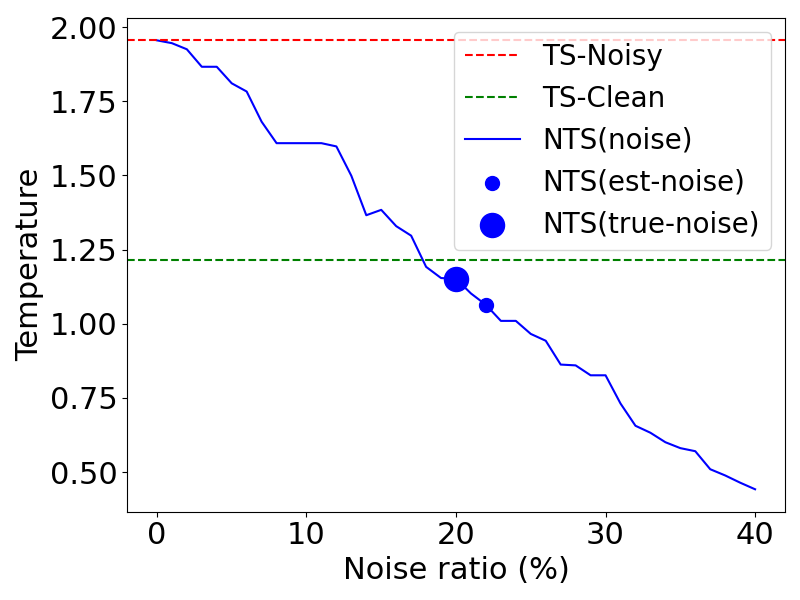}     
        \caption{ 
        Calibration performance, measured by adaECE (left),  and the corresponding temperature (right) as a function of the noise ratio used by the NTS algorithm.
                 }
    \label{fig:noise_sweep}
\end{figure*}

{\bf Sensitivity to the noise estimation accuracy.}
Our NTS method is based on estimating the noise level, which is carried out as part of the network training.  Crucially, the NTS method must not be highly sensitive to variations in the noise estimation, i.e., there is a range of values around the correct noise level in which NTS performs well. In the following experiment, we corrupted the validation set using noise level true-$\epsilon=0.2$.  Figure \ref{fig:noise_sweep} shows the adaECE and the temperature obtained by the  NTS as a function of the actual $\epsilon$ used in the calibration process. We denote this method NTS$(\epsilon)$. 
We also report the results of NTS(est-$\epsilon$) and NTS(true-$\epsilon$), which are variants of  NTS using the estimated noise matrix and the exact noise matrix, respectively. The experiments were run on the test set of PathMNIST with the ResNet18 architecture.
The results show that  there is indeed a large range of $\epsilon$ values in which NTS performs much better than TS$_{noisy}$.

\chapter{Conformal Prediction under Noisy Labels}
\label{ch:nrcp_nacp}

In this chapter, we address the application of Conformal Prediction (CP) methods in classification networks under the challenge of noisy labels. Conformal Prediction is a powerful tool for safety-critical applications, such as medical imaging, where maintaining a predefined level of certainty in model predictions is essential. The ability to generate a set of possible class candidates, ensuring the true class is included with high confidence, provides valuable support for clinical practitioners by narrowing down the possible diagnoses while controlling the risk of mistakes. We introduce novel methods that enhance the robustness of CP under label noise, demonstrating their effectiveness on medical and natural image datasets.

\section{Problem Statement}

In machine learning for safety-critical applications, the model must only make predictions it is confident about. One way to achieve this is by returning a (hopefully small) set of possible class candidates that contain the true class with a predefined level of certainty. This is a natural approach for medical imaging, where safety is of the utmost importance and a human makes the final decision. This allows us to aid the practitioner, by reducing the number of possible diagnoses he needs to consider, with a controlled chance of mistake.
The general approach to return a prediction set without any assumptions on the data distribution (besides i.i.d. samples) is called Conformal Prediction (CP) \citep{angelopoulos2023conformal,vovk2005conformal}. It creates a prediction set with the guarantee that the probability of the correct class being within this set meets or exceeds a specified confidence threshold. The goal is to return the smallest set possible while maintaining the confidence level guarantees. Recently, with the growing use of neural network systems in safety-critical applications such as medical imaging, CP has become an important calibration tool \citep{lu2022improving,lu2022fair,olsson2022estimating}. We note that CP is a general framework rather than a specific algorithm. The most common approach builds the prediction set using a conformity score, and different algorithms mostly vary in terms of how the conformity score is defined.

When dealing with conformal predictions, a critical challenge arises in applications such as medical imaging due to label noise. In these domains, datasets frequently contain noisy labels stemming from ambiguous data that can confuse even clinical experts. Furthermore, physicians may disagree on the diagnosis for the same medical image, leading to inconsistencies in the ground truth labeling. Noisy labels also occur when applying differential privacy techniques to overcome privacy issues \citep{ghazi2021deep}. While significant efforts have been devoted to the problem of noise-robust network training \citep{9729424,Xue2022}, the challenge of calibrating the models has only recently begun to receive attention.

In this chapter, we tackle the challenge of applying CP to classification networks using a validation set with noisy labels. \citet{einbinder2022conformal} suggested ignoring label noise and simply applying the standard CP algorithm on the noisy labeled validation set. This strategy results in large prediction sets especially when there are many classes. The most related studies to ours are \citep{sesia2023adaptive, Clarkson2024} which  present a noisy CP algorithm using conservative coverage guarantee bounds which can result in large prediction sets, failing in classification tasks with many classes.  In the following sections, we present two novel approaches and algorithms for CP on noisy data. The first approach propose a new score that is robust to label noise. The second approach propose to estimate the threshold in the presence of noisy labels, similar to \citet{sesia2023adaptive} and \cite{Clarkson2024}, but yields an effective coverage guarantee even in tasks with a large number of classes in the uniform noise setting. We applied the algorithms to several standard medical and scenery imaging classification datasets and show that the latter method outperformed previous methods by a significant margin and achieved results comparable to those obtained by using a clean validation set. The greatest value and novelty of our approach lies in tasks with many classes, such as CIFAR-100, TinyImageNet, and ImageNet, where all other methods fail.

The study described in this chapter was published in \cite{Penso_2024} and \cite{PensoNACP2025}.

\section{Score That Is Robust to Label Noise}
\label{subsec:nrcp}

Let $(x_1,\tilde{y}_1),...,(x_n,\tilde{y}_n)$ be a noisy validation set where the labels were corrupted by uniform noise with a noise level $\epsilon$.
 We aim to find a noise-robust conformal score that can be applied to the noisy labeled data. 
 Since $y_t$ is not observed, we cannot directly compute the score $S(x_t,y_t)$. Instead, we can estimate it using its noisy version $\tilde{y}_t$:
\begin{equation}
\mathbf{E} ( S(x_t,y_t)| \tilde{y}_t) = \sum_{i=1}^k p(y_t=i|\tilde{y}_t) S(x_t,i).
\label{est}
\end{equation}
Assuming a non-informative uniform  prior on the correct label $y_t$, i.e., $p(y_t=i)=1/k$), we obtain:
\begin{equation}
  p(y_t=i|\tilde{y}_t)  = (1\!-\!\epsilon) 1_{\{\tilde{y}_t=i\}}+ \frac{\epsilon}{k}.
  \label{ytnoise}
  \end{equation}
Substituting  (\ref{ytnoise}) in  (\ref{est}), yields an estimate 
$ \hat{S}(x_t,\tilde{y}_t,\epsilon)$ of the noise-free score:
\begin{equation}
\hat{S}(x_t,\tilde{y}_t,\epsilon) = \mathbf{E} ( S(x_t,y_t)| \tilde{y}_t) = (1\!-\!\epsilon) S ( x_t,\tilde{y}_t)  + \epsilon S(x_t)
\label{avg_score}
\end{equation}
s.t. $S(x_t) =   \frac{1}{k} \sum_{i=1}^k S(x_t,i)$.
Note that to obtain the score estimation $\hat{S}(x_t,\tilde{y}_t,\epsilon)$,  we need to either know the noise level $\epsilon$ or estimate it from the noisy-label data. We elaborate further on this issue in the next section.

\begin{table}[H]
\caption{Conformal Prediction methods for validation sets with noisy labels. Given an image $x$, $y$ is the true label, $\hat{y}$ is its noisy version and $\epsilon$ is the noise level.  $S$ is a conformal score and $\hat{S}$ is its noise robust variant.}
\centering
    \resizebox{\linewidth}{!}{%
        \begin{tabular}{l|c|c|c|c}
        \hline
        Stage & CP (Oracle) & Noisy-CP \cite{einbinder2022conformal} & NRESCP & NRSCP  
        \\  \hline
        Learning phase& $S(x,y) \xrightarrow{} q$ & $S(x,\tilde{y}) \xrightarrow{} q_{noise}$ & $\hat{S}(x,\tilde{y},\epsilon) \xrightarrow{} q_{\epsilon}$ & $\hat{S}(x,\tilde{y},\epsilon) \xrightarrow{} q_{\epsilon}$  \\  \vspace{1mm} \hspace{-2mm}
        Inference phase & $\{y|S(x,y) \le q\} $ & $\{y|S(x,y)\! \le \!q_{noise}\} $ & $\{y|\hat{S}(x,y,\epsilon) \le q_{\epsilon}\}$ & $\{y|S(x,y) \le q_{\epsilon}\} $ \\  \hline
        \end{tabular} %
    }
\label{noisy_methods}
\end{table}

We next apply the CP algorithm on the estimated conformal scores and set $q_{\epsilon}$ to be the $(1\!-\!\alpha)$ quantile of $\hat{S}(x_1,\tilde{y}_1,\epsilon),...,\hat{S}(x_n,\tilde{y}_n,\epsilon)$.
According to the general CP theory, the prediction set of a given test sample $x$ is:
  \begin{equation}
  \hat{C}_{\epsilon}(x)=\{ y \,|\, \hat{S}(x,{y},\epsilon) \le q_{\epsilon}\}  =
 \{ y \, | \, S (x, y)  \le \frac{q_{\epsilon}-\epsilon S(x)}{1\!-\!\epsilon}\}.
 \label{nrescp}
\end{equation}
Let $x$ be a test point and $y$ and $\tilde{y}$ be its true label and noisy label respectively.
  The general CP theory guarantees that $1\!-\!\alpha \le p( \tilde{y}\in \hat{C}_{\epsilon}(x))$. This guarantee, however, is for the noisy labeled data. Assume that for every $x$ the order of the network class predictions  $p(y=i|x;\theta)$ coincides with the order of the true probabilities $p(y=i|x)$. In that case, the same argument that appears in  \cite{einbinder2022conformal} for Noisy-CP, implies a coverage guarantee in the noise-free case.
However, the prediction set obtained by the estimated score  (\ref{nrescp}), is usually still too large.

In the case of noisy labels, during the CP learning phase, we need to  estimate the score of the correct class. However,  to form the prediction set at test time we need to compute the scores of all the possible classes. Hence, 
in a way similar to the Noisy-CP algorithm,
\begin{algorithm}[t]
\caption{Noise-Robust Score Conformal Prediction (NRSCP)}
       \begin{algorithmic}[1]
      \State Input: A conformal score $S(x,y)$, a coverage level $1\!-\!\alpha$ and a validation set $(x_1,y_1),...,(x_n,y_n)$,
      s.t. the labels are corrupted by uniform noise with parameter~$\epsilon$.  
    \State Compute the estimated scores:
    \[  s_t = \hat{S}(x_t,y_t,\epsilon) = (1-\epsilon) S ( x_t,y_t)  + \frac{\epsilon}{k} \sum_{i=1}^k {S}(x_t,i), \hspace{0.5cm} t=1,...,n\]  
    \State Set $q$ to be the 
    $\lceil(n+1)(1\!-\!\alpha)/n\rceil$ quantile of $s_1,...,s_n$.
        \vspace{1mm}
    \State The prediction set of a test sample $x$
         is $C(x)=\{y\,|\,S(x,y)<q\}$. 
   \end{algorithmic}
   \label{nrcp_alg1}
       \end{algorithm}
we can thus construct the prediction set of a test sample $x$ using the exact score $S(x,y)$  instead of the estimated score $\hat{S}(x,\tilde{y},\epsilon)$:
\begin{equation}
C_{\epsilon}(x)=\{ y \,| \,S(x,y) \le q_{\epsilon}\}.
\label{nrcp}
\end{equation}
 We denote the algorithm variants based on Eqs. (\ref{nrescp}) and (\ref{nrcp}) as the  Noise Robust Estimated Score CP  (NRESCP) and the Noise Robust Score  CP (NRSCP) respectively. 
 The only difference between them lies in how the test-time prediction set is formed.
 The various noisy label CP methods discussed above are summarized in Table~\ref{noisy_methods}.
We empirically show below that NRSCP satisfies the coverage requirement and yields an average size that is much smaller than the one obtained by NRESCP. The NRSCP is summarized in Algorithm Box \ref{nrcp_alg1}.

We next analyze the proposed noise robust conformal score.
It is easy to verify that the prediction set obtained by NRSCP is smaller than the one obtained by NRESCP, i.e.,  $C_{\epsilon}(x)  \subset \hat{C}_{\epsilon}(x)$ if and only if  $S(x) \le q_{\epsilon}$.
In the case of HPS, $S(x)=\frac{1}{k}\sum_i (1-p(y=i|x))=(k-1)/k$ and therefore $\hat{S}(x,\tilde{y},\epsilon) = (1-\epsilon)  S(x,\tilde{y})+ \epsilon\frac{k-1}{k}$.
This implies that  $q_{\epsilon}  = (1-\epsilon)  q_{noise} + \epsilon\frac{k-1}{k}$ 
such that $q_{noise}$ and $q_{\epsilon}$ are the thresholds computed by Noisy-CP \cite{einbinder2022conformal} and NRESCP (\ref{nrescp}) respectively. It is easy to verify that Noisy-CP  and NRESCP yield the same prediction set, i.e. 
$\{y|S(x,y) \le q_{noise}\}   = \{y|\hat{S}(x,y,\epsilon) \le q_{\epsilon}\}$.  
  In the case of the APS conformal score, 
  it is easy to see that:  $
  (\hat{p}+1)/2 \le  S(x) \le (\hat{p} + k-1)/k$,
 s.t.  $ \hat{p}=  \max_i p(y=i|x;\theta)$ is the network confidence on its single-class prediction.
      Hence,  when the prediction sets obtained by NRSCP are smaller than those obtained by NRESCP, the prediction confidence satisfies $(\hat{p}+1)/2 \le S(x) \le  q_{\epsilon}$. 

\section{Procedure of Threshold Estimation That Is Robust to Label Noise}
\label{subsec:nacp}
In the previous section we presented a new conformal score that takes into account the noise rate making it robust to the label noise. Next, we tackle the problem of conformal prediction with noisy labels from a different angle by defining a conformal prediction procedure that is robust to noisy labels along with theoretic and coverage guarantees. Here we show how, given a simple noise model and a known noise level, we can get the correct CP threshold based on noisy data. We will generalize this beyond the simple noise model in the following section. Consider a network that classifies an input $x$ into $k$ pre-defined classes.
Given a conformity score $S(x,y)$ and a specified coverage $1-\alpha$, the goal of the conformal prediction algorithm is to find a minimal $q$ such that $p(y\in C_q(x))\ge 1-\alpha$.
Let $(x_1,\tilde{y}_1),...,(x_n,\tilde{y}_n)$ be a validation set with noisy labels and
let $y_i$ be the unknown correct label of $x_i$. Let $s_i=S(x_i,\tilde{y}_i)$ be the conformity score of $(x_i,\tilde{y}_i)$.
We assume that the label noise follows a uniform distribution, where with a probability of $\epsilon$,  the correct label is replaced by a label that is randomly sampled from  the $k$ classes:
 \begin{equation}
 p(\tilde{y}=j| y=i) = \mathds{1}_{\{i=j\}}(1-\epsilon) + \frac{\epsilon}{k}.
 \label{nnoise}
 \end{equation}
Uniform noise is relevant, for example, when applying differential privacy techniques to overcome privacy issues \citep{ghazi2021deep}.
In that setup the noise level $\epsilon$ is usually known. 
In other applications such as medical imaging, where the noise parameter $\epsilon$ is not given, it can be estimated with sufficient accuracy from the noisy-label data during training \citep{zhang2021learning, li2021provably, Lin2023Holistic}.
 We can write $\tilde{y}$ as $\tilde{y}=(1-z)\cdot y+z \cdot u$,
s.t. $u$ is a random label uniformly sampled from $\{1,...,k\}$ and
  $z$ is a binary random variable $(p(z=1)=\epsilon)$ indicating whether the label of the sample $(x,y)$ was replaced by a random label or not.
For each candidate threshold, $q$ denote:
\begin{equation}
F^c(q) = p(y\in C_q(x)),  \hspace{1cm}   F^n(q) = p(\tilde{y}\in C_q(x)), \nonumber
\end{equation}
$$
F^r(q) = p(u\in C_q(x)), 
$$
where $F^c$, $F^n$, and $F^r$ represent the clean, noisy and random labels. Note as well that each one is the CDF of the appropriate conformal score function, e.g., $F^c(q) = p(y\in C_q(x))=p(S(x,y)\leq q)$.

It is easily verified that
  \begin{equation}
  F^n(q) =  p(z=0)F^c(q) + p(z=1) F^r(q)
    \label{eqAB}
\end{equation}
  $$= (1 - \epsilon )F^c(q) + \epsilon F^r(q).$$

For each value $q$, we can estimate $F^n(q)$ from the noisy validation set:
 \begin{equation}
\hat{F}^n(q) = \frac{1}{n} \sum_i \mathds{1}_{\{\tilde{y}_i\in C_q(x_i)\}}
          = \frac{1}{n}\sum_i \mathds{1}_{\{s_i \le q \}}.
\label{hatA}
   \end{equation}
   Note that $q$ is the $\hat{F}^n(q)$-quantile of $s_1,...,s_n$.
Similarly we can also estimate  $F^r(q)$:
 \begin{equation}
\hat{F}^r(q) = \frac{1}{n} \sum_i p ( u_i \in C_q(x_i) ) =
\frac{1}{n} \sum_i  \frac {|C_q(x_i)|}{k},
\label{hatD}
 \end{equation}
s.t. $u_i$ is  uniformly sampled from $\{1,...,k\}$.

By substituting  (\ref{hatA}) and (\ref{hatD}) in  (\ref{eqAB}) we obtain an estimation of $F^c(q) = p ( y\in C_q(x))$ based on the noisy validation set and the noise level $\epsilon$:
 \begin{equation}
  \hat{F}^c(q) = \frac{\hat{F}^n(q) - \epsilon  \hat{F}^r(q)}{1-\epsilon}.
  \label{hatB}
\end{equation}

 For each candidate  $q$ we first compute $\hat{F}^n(q)$ and $\hat{F}^r(q)$ and
 then by using (\ref{hatB})  obtain the coverage estimation $\hat{F}^c(q)$. Given a coverage requirement $(1-\alpha)$,  we can thus use the noisy validation set to find a threshold $q$ such that $\hat{F}^c(q)=1-\alpha$. Note that since $F^c(q)$ is monotonous, it seems reasonable to search for the threshold $q$ using the bisection method. However, as $\hat{F}^c(q)$ is an approximation based on the \emph{difference} between two monotonic functions, it might not be exactly monotonous. We therefore find the threshold $q$ using an exhaustive grid search.
If there are several solutions we select the largest value.
(In practice selecting one of the solutions has almost no effect on the results.)  We note that even with an exhaustive search the entire runtime is negligible compared to the training time.

We can narrow the threshold search domain  as follows:
 \begin{lemma}
For every threshold $q$ we have: $\hat{F}^n_q/k \le \hat{F}^r(q)$.
\end{lemma}
\begin{proof}
Denote $A=\{i|\hat{y}_i\in C_q(x_i)\}$ and $B=\{i| 1\le |C_q(x_i)|\}$. Note that  $\hat{F}^n(q) = |A|/n$.
$$ |B| =\sum_{i \in B} 1  \le \sum_{i\in B} |C_q(x_i)| \le \sum_{i=1}^n |C_q(x_i)|
 =nk \hat{F}^r(q).$$
Finally  $A \subset B$ implies that:
$ \hat{F}^n(q) = |A|/n \le |B|/n \le k \hat{F}^r(q). $
\end{proof}
 \begin{theorem}
Let $q_1$ be the $(1\!-\!\alpha)(1-\epsilon)/(1-\frac{\epsilon}{k})$  quantile of  $s_1,...,s_n$ and let $q_2$ be the
$(1-\alpha)+\alpha\epsilon$ quantile.
If  $q$ satisfies  $\hat{F}^c(q)=1-\alpha$ then
 $q_1\le q\le q_2$.
\end{theorem}
\begin{proof} 
 Assume $q$  satisfies
$\hat{F}^c({q})=1-\alpha$. Eq. (\ref{hatB}) implies that
  \begin{equation}
  1-\alpha =  \hat{F}^c({q})=\frac{\hat{F}^n({q}) - \epsilon  \hat{F}^r({q})}{1-\epsilon } 
    \label{hatcx}
\end{equation}
  $$  \Rightarrow \,\,
     \hat{F}^n({q})  = (1-\alpha) (1-\epsilon) + \epsilon\hat{F}^r({q}).
$$
Since $0 \le \hat{F}^r(q) \le 1$ we get that:
  \begin{equation}
  (1-\alpha)(1-\epsilon) \le \hat{F}^n({q}) \le  (1-\alpha)+\alpha\epsilon =\hat{F}^n(q_2).
  \label{aq}
    \end{equation}
For every $q$ we have  $ \hat{F}^n(q)/k \le \hat{F}^r(q)$ (Lemma 3.1).
Hence, $ (1-\alpha)(1-\epsilon) \le \hat{F}^n(q)$  (\ref{aq})    implies that $ (1-\alpha)(1-\epsilon)/k \le   \hat{F}^r(q) $.
Combining this inequality with Eq. (\ref{hatcx}) yields a better lower bound:
$(1-\alpha) (1-\epsilon)(1+\epsilon/k) \le \hat{F}^n(q)   $.
Iterating  this process yields:
$$
   (1-\alpha) (1-\epsilon)\left(1+\frac{\epsilon}{k} + \left(\frac{\epsilon}{k}\right)^2 + \dots   \right)
$$ $$ = (1-\alpha) \frac{1-\epsilon}{1-\frac{\epsilon}{k}} = \hat{F}^n(q_1) \le \hat{F}^n(q).$$
Finaly, $\hat{F}^n(q)$ is a monotonically increasing function of $q$  which implies that $q_1 \le q\le q_2$.
\end{proof}

As an alternative to the grid search we can sort the noisy conformity scores $s_i=S(x_i,\tilde{y}_i)$ and look for the minimal $i$ such that $\hat{F}^c(s_i)\ge 1-\alpha$. In the 
noise-free case  $\hat{F}^c$ is piece-wise constant, with jumps determined exactly by the order statistics $s_i$, namely,  $\hat{F}^c(s_i)=i/n$ and thus this algorithm coincides with the standard CP algorithm. In the noisy case   $\hat{F}^c(q)$ depends on the conformity scores of all the $k$ classes and thus its structure is more complicated.
We dub our algorithm Noise-Aware Conformal Prediction (NACP), and summarize it in Algorithm Box \ref{nacp_alg1}.
Note that in the noise-free case ($\epsilon=0$) the NACP algorithm coincides with the standard CP algorithm and selects $q$ that satisfies $\hat{F}^c(q)=\hat{F}^n(q)=1-\alpha$, i.e., $q$ is the $1-\alpha$ quantile of the validation set conformity scores.

\subsection{Prediction Size Comparison}

We next compare our NACP approach analytically to Noisy-CP \citep{einbinder2022conformal} in terms of the average size of the prediction set.

\begin{theorem}
Let $q$ and $\tilde{q}$ be the thresholds computed by the NACP and the Noisy-CP algorithms respectively.
Then $q \le \tilde{q}$  if and only if $\hat{F}^r(\tilde{q})  \le (1-\alpha)$.
\label{theorem3-2}
\end{theorem}

\begin{proof}
The threshold $\tilde{q}$ computed by the Noisy-CP
algorithm (by applying standard CP on the noisy validations set) satisfies $\hat{F}^n(\tilde{q})  = (1-\alpha)$.
The true threshold $q$ satisfies
$\hat{F}^n({q})  = (1-\alpha) (1-\epsilon) + \epsilon\hat{F}^r({q}) $
(\ref{hatcx}).  Looking at the difference
\begin{equation}\label{eq:diff}
   \hat{F}^n(\tilde{q})-\hat{F}^n({q})=1-\alpha-(1-\alpha)(1-\epsilon)-\epsilon \hat{F}^r({q})
   \end{equation}
   $$ =\epsilon(1-\alpha-\hat{F}^r({q})).$$
Hence from the monotonicity of $\hat{F}^n(q)$ we have  $q \le \tilde{q}$ iff $ \hat{F}^n(q) \le \hat{F}^n(\tilde{q})$ iff $\hat{F}^r({q})\le 1-\alpha$.
 \end{proof}
 
The theorem above states that if the size of the prediction set obtained by NACP is less than $k(1-\alpha)$,  NACP is more effective than Noisy-CP.
For example, assume $k=100$ and $1-\alpha=0.9$. In this case, if the average size of the NACP prediction set is less than 90, NACP is more effective than Noisy-CP. We also see from eq. (\ref{eq:diff}) that the smaller $\hat{F}^r$ is the larger the gap between the two methods. Since  $\hat{F}^r$ is inversely proportional to the number of classes, we expect the difference to be substantial when there is a large number of classes to consider, which is exactly where CPs' ability to reliably exclude possible classes is very useful. In our experiments, we indeed found a considerable gap between the two methods when we experimented on classification tasks with a large number of classes.

\begin{algorithm}[t]
\caption{Noise-Aware Conformal Prediction (NACP) for uniform noise}
       \begin{algorithmic}[1]
      \State Input: A conformity score $S(x,y)$, a coverage level $1\!-\!\alpha$ and a validation set $(x_1,\tilde{y}_1),...,(x_n,\tilde{y}_n)$,
      s.t. the labels are corrupted by a uniform noise with parameter~$\epsilon$.
\State Set $q_1$ to be the $(1\!-\!\alpha)(1-\epsilon)/(1-\frac{\epsilon}{k})$ quantile of  $S(x_1,\tilde{y}_1),...,S(x_n,\tilde{y}_n)$ and set $q_2$ to be
$((1-\alpha)+\alpha\epsilon)$ quantile.
    \State For each candidate threshold $q$ compute:
\begin{align}
\hat{F}^n(q)  &= \frac{1}{n} \sum_i \mathds{1}_{\{\tilde{y}_i\in C_q(x_i)\}}, \nonumber \\ 
\hat{F}^r(q)  &= \frac{1}{n} \sum_i  \frac {|C_q(x_i)|}{k}, \nonumber \\ 
 \hat{F}^c(q) &= \frac{\hat{F}^n(q) - \epsilon \hat{F}^r(q)}{1-\epsilon } \nonumber
\end{align}
       \State Apply a grid search to find $q\in[q_1,q_2]$ that satisfies $\hat{F}^c(q)=1\!-\!\alpha$.
    \State The prediction set of a test sample $x$
         is: $$C_q(x)=\{y\,|\,S(x,y)<q\}.$$
         \State Coverage guarantee:  $p ( y \in C_q(x) ) \ge 1-\alpha - \Delta(n,\epsilon,\delta)$  with probability $(1-\delta)$ over the noisy validation set sampling  (see Theorem \ref{theorem3}).
   \end{algorithmic}
   \label{nacp_alg1}
       \end{algorithm}

\subsection{Coverage Guarantees}

We next provide a coverage guarantee for NACP. We show that if we apply the NACP to find a threshold $q$ for $1-\alpha+\Delta$, then  $P(y\in C_q(x))\ge 1-\alpha$ were $\Delta$ depends on the validation set size. $\Delta$ is a finite-sample term that is needed to approximate the CDF to set the threshold instead of simply picking a predefined quantile.  Because $\Delta$ can be computed, one can adjust the $\alpha$ used in the NACP algorithm to get the desired coverage guarantee.  However, we note that we empirically found this bound to be over-conservative, and that the un-adjusted method does reach the desired coverage.

\begin{lemma}
    Given $\delta>0$, define $\Delta=\sqrt{\frac{\log(4/\delta)}{2nh^2}}$ such that $h=\frac{1-\epsilon}{1+\epsilon}$
    and $n$ is the size of the noisy validation set. Then 
    \begin{equation}
        p( \sup_q |F^c(q)-\hat{F}^c(q)| >  \Delta) \le \delta,
    \end{equation}
    such that the probability is over the validation set.
\label{nacp_lemma_delta}
\end{lemma}

\begin{proof}
    The Dvoretzky–Kiefer–Wolfowitz (DKW) inequality \citep{massart1990tight} states  that if we estimate a CDF $F$ from $n$ samples using the empirical CDF $F_n$ then $p(\sup_x|F_n(x)-F(x)|>\Delta)\leq 2\exp(-2n\Delta^2$). Eq. (\ref{hatB}) defines $\hat{F}^c(q)$ using $\hat{F}^n(q)$ and $\hat{F}^r(q)$. Both are empirical CDF, so from the DKW theorem and the union bound we get that:
    \begin{align}
    &p(\sup_q|F^r(q)-\hat{F}^r(q)|>h\Delta \,\, \mbox{or} \\ &  \,\, \sup_q|F^n(q)-\hat{F}^n(q)|>h\Delta)
    \leq 4\exp(-2nh^2\Delta^2)= \delta. \nonumber
    \end{align}
Using eq. (\ref{hatB}) we get that with probability at least $1-\delta$ for every $q$:
    \begin{align}
        &\hat{F}^c(q)= \frac{\hat{F}^n(q) - \epsilon  \hat{F}^r(q)}{1-\epsilon }\leq  \\ & \frac{(F^n(q)+h\Delta) - \epsilon  (F^r(q)-h\Delta)}{1-\epsilon } \nonumber 
         F^c(q)+\frac{h\Delta + \epsilon h\Delta}{1-\epsilon } \\ & = F^c(q)+h\Delta\frac{1+\epsilon}{1-\epsilon}= F^c(q)+\Delta. \nonumber
    \end{align}
    Similarly, we can show that $\hat{F}^c(q)\geq F^c(q)-\Delta$ which completes the proof.
\end{proof}

The proof of the main theorem now follows the standard CP proof, taking the inaccuracy in estimating $F^c(q)$ into account.

\begin{theorem}
\label{theorem3}
    Assume you have a noisy validation set of size $n$ with noise level $\epsilon$ and set $\Delta(n,\epsilon,\delta)=\sqrt{\frac{\log(4/\delta)}{2nh^2}}$ s.t. $h=\frac{1-\epsilon}{1+\epsilon}$ and that you pick $q$ such that $\hat{F}^c(q)= 1-\alpha+\Delta$. Then with probability at least $1-\delta$ (over the validation set), we have that if $(x,y)$ are sampled from the clear label distribution we get:
    $$ 1-\alpha \le p(y\in C_q(x))\leq 1-\alpha + 2\Delta.$$
\end{theorem}

\begin{proof}
    Given a clean test pair $(x,y)$, with probability $\delta$ over the validation set,  we have: $$p(y\in C_q(x))=p(S(x,y)<q) $$ $$ =F^c(q)\geq \hat{F}^c(q)-\Delta =  1-\alpha.$$
    In a similar way: $p(y\in C_q(x))=F^c(q)\leq \hat{F}^c(q)+\Delta = 1-\alpha + 2\Delta$.
\end{proof}
As the size of the noisy validation set, $n$, tends to infinity, $\Delta$ converges to  zero and thus the noisy threshold converges to the noise-free
threshold.  

\begin{algorithm}[t!]
\caption{Noise-Aware Conformal Prediction (NACP) for a  noise matrix model}
       \begin{algorithmic}[1]
      \State Input: A conformity score $S(x,y)$, a coverage level $1\!-\!\alpha$ and a validation set $(x_1,\tilde{y}_1),...,(x_n,\tilde{y}_n)$,
      s.t. the labels are corrupted by a noise matrix $P$.
    \State For each candidate threshold $q$ compute:
$$
\hat{M}_{q}(\ell,i)  = \frac{1}{n}\sum_j \mathds{1}_{\{ \tilde{y}_j = i, \,\, \ell \in C_q(x_j) \}},\hspace{0.5cm} i,\ell=1,..,k.
$$

$$
\hat{F}^c(q) = \textrm{Tr}(\hat{M}_{q} P^{-1}).
$$
       \State Apply a grid search to find $q$ that satisfies $\hat{F}^c(q)=1\!-\!\alpha$.
    \State The prediction set of a test sample $x$
         is: $$C_q(x)=\{y\,|\,S(x,y)<q\}.$$
   \end{algorithmic}
   \label{alg:nacp_alg2}
       \end{algorithm}


\subsection{A More General Noise Model}
\label{sec:general_noise_model}
Next, we will extend our approach to a more general noise model. We will assume that the noisy label $\tilde{y}$ is independent of $x$ given $y$. We also assume that the noise matrix  $P(i,j)=p(\tilde{y}=j|y=i)$ is known and that the matrix P is invertible. For each $q$ define the following matrices for the clear and the noisy data:
 ${M}_q^c(\ell,i)=p( \ell\in C_q(x),y=i)$ and
 ${M}_q(\ell,i)=p( \ell\in C_q(x),\tilde{y}=i)$. Assuming that, given the
 true label $y$,  the r.v. $x$ and $\tilde{y}$ are independent, we obtain:
    \begin{align}
        {M}_{q}(\ell,i)&=
        p(\ell\in C_q(x),\tilde{y}=i)\nonumber  \\& =
        \sum_j p(\ell\in C_q(x),\tilde{y}=i,y=j) \label{generalP}\\&=\sum_j p(\ell\in C_q(x),y=j)p(\tilde{y}=i|y=j) \nonumber
        \\ &= \sum_j {M}^c_{q}(\ell,j)P(j,i). \nonumber
    \end{align}
We can write (\ref{generalP}) in matrix notation: $ {M}_q=M^c_q P$. Eq. (\ref{generalP}) implies that:
\begin{equation}
F^c(q)=p(y\in C_q(x)) = p ( y\in C_q(x)) 
\label{estM1}
\end{equation}
$$= \sum_i p(i\in C_q(x) , y=i)=
\sum_i M^c_q(i,i) = \textrm{Tr} ( M_qP^{-1}).
$$

We can estimate matrix  $M_q$ from the noisy samples:
\begin{equation}
    \hat{M}_{q}(\ell,i)  = \frac{1}{n}\sum_j \mathds{1}_{\{ \tilde{y}_j = i, \,\, \ell \in C_q(x_j) \}},\hspace{1cm} i,\ell=1,..,k.
    \label{estM2}
\end{equation}
Substituting (\ref{estM2}) in (\ref{estM1}) yields  an estimation of the  probability $F^c(q)=p(y\in C_q(x))$:
\begin{equation}
\hat{F}^c(q) = \textrm{Tr}(\hat{M}_{q} P^{-1}).
\label{estM3}
\end{equation}
The final step is applying a grid search to find a threshold $q$ such that $\hat{F}^c(q)=1-\alpha$.

In the case that $P$ is a uniform noise matrix (\ref{nnoise}),
the Sherman-Morison formula implies that $P^{-1}=(\frac{1}{1-\epsilon}I - \frac{\epsilon}{(1-\epsilon)k} \1^{\T})$. Therefore,
$$
\hat{F}^c(q) = \textrm{Tr}(\hat{M}_{q} P^{-1})=
\frac{1}{1-\epsilon}\sum_i \hat{M}_{q}(i,i)
$$ $$ 
- \frac{\epsilon}{(1-\epsilon)k}  \sum_{\ell,i} \hat{M}_{q}(\ell,i) = \frac{\hat{F}^n(q) - \epsilon \hat{F}^r(q)}{1-\epsilon }.
$$
Thus in the case of a uniform noise the coverage estimation (\ref{estM3}) coincides with  (\ref{hatB}).
If the noise matrix is unknown, it can be estimated from the noisy-label data during training \citep{zhang2021learning, li2021provably,Lin2023Holistic}.
The NACP method for a noise matrix model is summarized in Algorithm box \ref{alg:nacp_alg2}.

We can extend the finite sample term $\Delta$ that was developed for a uniform noise to obtain a theoretical coverage guarantee for a noise matrix model (\ref{general_noise_coverage_guarentee_theorem}).

\begin{theorem}
\label{general_noise_coverage_guarentee_theorem}
Let $P$ be a general noise matrix. 
Given  $\delta>0$, define $\Delta = \|P^{-1}\|_{\infty} k\sqrt{\frac{\log(2k^2/\delta)}{2n}}$ where
 , $k$ is the number of classes and $n$ is the size of the noisy validation set.  Then  $$p ( \sup\limits_{q}|\hat{F}^c(q)-{F}^c(q)| > \Delta ) <\delta.$$
\end{theorem}

\begin{proof}
    From Eq. (\ref{estM3}) we have $|\hat{F}^c(q)-F^c(q)| =|\Tr(P^{-1}\hat{M}_q)-\Tr(P^{-1}{M}_q)|=|\Tr(P^{-1}\Delta\hat{M}_q)|$ where $\Delta{M}_q =\hat{M}_q -M_q $. 
        We first note that ${M}_q[i,j] = {p}(j \in C_q(x), \tilde{y} = i)$ is not a CDF but we can define one that agrees with it for $q\in(-\infty,C)$ which is the range of interest 
         where $C$ is a constant that bound the score function $S(x,y)$ from above. We define $\tilde{S}_{ij}(x,\tilde{y}) = 
\begin{cases} 
    S(x,j),  & \text{if } \tilde{y}=i \\
    C, & \text{if } \tilde{y}\neq i
\end{cases}$ so ${M}_q[i,j]=p(\tilde{S}_{ij}(x,\tilde{y})\leq q)$ for $q\in(-\infty,C)$. Now from the DKW theorem, we know that if we estimate a CDF using $n$ samples then with probability at least $1-\delta$ we get a uniform bound on the error of size $\sqrt{\frac{\log(2/\delta)}{2n}}$. As we are estimating $k^2$ matrix elements we can use the union bound to get that with probability $1-\delta$ the $\forall i,j,q\in(-\infty,C):|\Delta{M}_q|\leq\sqrt{\frac{\log(2k^2/\delta)}{2n}}$.
    Now if we look at the infinity norm of $P^{-1}$, then $|(P^{-1}\Delta{M}_q)_{i,j}| \leq ||P^{-1}||_\infty\sqrt{\frac{\log(2k^2/\delta)}{2n}}$. As the trace is the sum of $k$ such matrix entries, the total bound is $\kappa_\infty k\sqrt{\frac{\log(2k^2/\delta)}{2n}}$ for $q\in(-\infty,C)$. Since we know ${F}^c(q)=1$ for $q\geq C$, we can set $\hat{F}^c(q)=1$ for $q\geq C$ and get a bound for all $q\in\mathbb{R}$.
\end{proof}

However,  this approach yields large prediction sets especially in tasks with many classes and thus is ineffective.  In the experiment section we show that in practice, even without adding finite sample terms, we obtain the required coverage probability. 

\section{Related Work}
\label{sec:nacp_related_work}
In this section, we review two closely related works that address the same problem of calibration with noisy labels \citep{sesia2023adaptive, Clarkson2024}. 
The derivation of the noisy conformal threshold in these two works is similar to ours. 
These two methods compute the same threshold $q$ that satisfies $\hat{f}^c(q)=1-\alpha$ (\ref{estM3}).  The only minor difference is that in these two studies they use the distribution of correct labels given the noisy labels, while we use the more natural distribution of the noisy labels given the correct label. As a result, they need to know the marginal class frequencies for both the clean and noisy labels, whereas we do not. 
Each one of the two methods provides a different finite coverage guarantee 
in the form of: $$ p(y \in C_q(x)) \ge 1-\alpha-\Delta$$ where $\Delta$ depends on the validation set size $n$, the number of classes $k$, and the noise model, but it doesn't depend on the validation dataset itself. 

We first review the bound $\Delta$ derived in \citep{sesia2023adaptive}.
 Let $\rho_i = p(y=i) $ and $\tilde{\rho}_i=p(\tilde{y}=i)$ be the marginal true and noisy label distributions.   Let $M(y|\tilde{y})$ be the noise  conditional distribution and let $V=M^{-1}$.
 Let $c(n) = \mathbb{E} \left[ \max_{i \in [n]} \left( \frac{i}{n} - u_{(i)} \right) \right]$, such that $\{u_{(i)}\}_{i=1}^{n}$ order statistics of $\{u_{i}\}_{i=1}^{n}$ i.i.d. uniform random variables on $[0, 1]$. The size of the least common class is $n_*=\min_{i\in[k]} n_i$ s.t. $n_i$ is the number of samples of noisy label $i$. Finally, the finite sample correction is:
 \begin{align}
     \Delta&=c(n) + \frac{2 \max_{i \in [k]} \sum_{l \neq i} |V_{il}| + 
 \sum_{i=1}^{k} |\rho_i - \tilde{\rho}_i|}{\sqrt{n_*}} \\ &  \cdot \min \left(k^2 \sqrt{\frac{\pi}{2}}, 
 \frac{1}{\sqrt{n^*}} + \sqrt{\frac{\log(2k^2) + \log(n^*)}{2}} \right). \nonumber
 \end{align}
It can be easily verified that $\Delta=O(\log{k})$ and therefore the bound becomes less effective for large values of $k$.

\citet{sesia2023adaptive} also suggested a boosted version under additional assumptions. This version takes a hybrid approach that adaptively chooses between its algorithm and Noisy-CP depending on which approach leads to a lower (less conservative) calibrated threshold.

The finite sample term $\Delta$ derived in \citep{Clarkson2024} is: 
\begin{equation}
\Delta =  \sum_{i=1}^k  ( |w_i^{(1)}|b(n, i) + \sum_{i \neq j} |w_{ij}^{(2)}| b(n, j)  ) 
\label{clarkson}
\end{equation}
s.t. 
$k$ is the number of classes, $w^{(1)}_i = P_{i,i}^{-1}\rho_i - \tilde{\rho}_i$, $w^{(2)}_{ij} = \rho_i P^{-1}_{ji}$, and 
$ b(n, j) = (1 - \tilde{\rho}_j)^n + \sqrt{\frac{
\pi }{n \tilde{\rho}_j}} $.  $\rho_i = p(y=i)$ and $ \tilde{\rho}_i = p(\tilde{y}=i)$  are the marginal clean and contaminated label probabilities and $P_{ji} =p(y = j | \tilde{y} = i) $ is the conditional label noise distribution.
It can be easily verified that $\Delta=O(\sqrt{k})$ and therefore the bound becomes less effective for large values of $k$.

We note that our finite sample term $\Delta$ (see Lemma \ref{nacp_lemma_delta}) does not depend on the number of classes $k$. Therefore, unlike the algorithms of  \citet{sesia2023adaptive} and \citet{Clarkson2024}, it remains effective even in tasks with many classes.
A further distinction between us and the  previous works  is that their finite sample coverage guarantee is established for the average of all the noisy validation sets. In contrast, our approach provides an individual coverage guarantee for nearly all ($1-\delta$) of the sampled noisy validation sets. In Section \ref{sec:nrcp_nacp_experiments}  we show that the average coverage guarantee 
obtained by \citet{sesia2023adaptive} and \citet{Clarkson2024} implies that in tasks with a large number of classes, the prediction set should include all the classes and therefore it is useless. In contrast, our individual finite set coverage guarantee, on to $(1-\delta)$ portion of the noisy validation sets, remains effective for tasks with many classes.

We note that our observation that the finite sample correction  term  does not depend on the number of classes, applies to the case of a uniform label noise. In the case of a general noise matrix, all finite sample correction terms are not effective.   

\section{Experiments}
\label{sec:nrcp_nacp_experiments}

In this section, we evaluate the capabilities of our NRSCP and NACP algorithms on various medical and scenery imaging datasets.

 \textbf{Compared methods.} Our method takes an existing conformity score $S$ and computes a threshold $q$ that takes into account the label noise level.  We  implemented  three popular conformal prediction scores, namely APS
\citep{romano2020classification}, RAPS
\citep{angelopoulos2020uncertainty}
 and HPS \citep{vovk2005conformal}. For each score, we implemented the following CP  methods: (1) CP (oracle) -  using a validation set with clean labels,  (2) Noisy-CP -  applying a standard CP on noisy labels without any modifications \citep{einbinder2022conformal},
 (3)  NR-CP (w/o $\Delta$) - Noise-Robust CP approach without the finite sample coverage guarantee $\Delta$,  
see Eq. (\ref{estM3}) and \citep{sesia2023adaptive,Clarkson2024}. 
  We also implemented three methods that add finite sample coverage guarantee terms
  to the NR-CP method. 
 (4)   Adaptive Conformal Classification with Noisy labels (ACNL) \citep{sesia2023adaptive}, (5)  Boosted Adaptive Conformal Classification with Noisy labels (ACNL$^+$) \citep{sesia2023adaptive} (6) Contamination Robust Conformal Prediction (CRCP) \citep{Clarkson2024}, (7) NRSCP - our first approach (\ref{subsec:nrcp}) and (8) NACP - our second approach (\ref{subsec:nacp}).
 For methods (4), (5), and (6), we used their official codes 
\footnote{ \url{https://github.com/msesia/conformal-label-noise}} \footnote{ \url{https://github.com/jase-clarkson/cp_under_data_contamination}} 
 and we share our code for reproducibility\footnote{\url{https://github.com/cobypenso/Noise-Aware-Conformal-Prediction}}.
\begin{table*}[t]
\centering
\caption{ CP calibration results  for $1\!-\!\alpha$ = 0.9
and noise level $\epsilon=0.2$.  We report the mean and the std over 1000 different splits. We show the  best result with theoretical guarantees in bold.}
\label{raps2}
   \scalebox{.7}{
\begin{tabular}{ll|rr|rr|rr}
& & \multicolumn{2}{c|}{APS}  & \multicolumn{2}{c|}{RAPS}  & \multicolumn{2}{c}{HPS}\\
\hline
Dataset      &  CP Method & size $\downarrow$ &  coverage(\%)   &  size $\downarrow$ &  coverage(\%) & size $\downarrow$ &  coverage(\%)   \\ \hline

                      &  CP (oracle) &  1.1 $\pm$ 0.01 & 90.0 $\pm$ 0.62 & 1.1 $\pm$ 0.01 & 90.0 $\pm$ 0.61  &  0.9 $\pm$ 0.01 & 90.0 $\pm$ 0.59  \\
                      &  Noisy-CP &  5.1 $\pm$ 0.18 & 99.9 $\pm$ 0.04 & 5.1 $\pm$ 0.18 & 99.9 $\pm$ 0.04  &  5.1 $\pm$ 0.18 & 99.8 $\pm$ 0.04  \\
                        &  NR-CP (w/o $\Delta$) &  \  {1.1 $\pm$ 0.02} & \  {90.1 $\pm$ 0.70} & \  {1.1 $\pm$ 0.02} & \  {90.1 $\pm$ 0.69}  &  \  {0.9 $\pm$ 0.02} & \  {90.0 $\pm$ 0.75}  \\
 CIFAR-10  &  ACNL &  1.5 $\pm$ 0.06 & 96.0 $\pm$ 0.61 & 1.3 $\pm$ 0.03 & 94.6 $\pm$ 0.65  &  1.1 $\pm$ 0.03 & 96.0 $\pm$ 0.59  \\
 (10 classes)  &  ACNL$^+$ &   1.4 $\pm$ 0.06 & 95.8 $\pm$ 0.59  & 1.3 $\pm$ 0.02 & 94.8 $\pm$ 0.61   &   1.1 $\pm$ 0.03 & 96.1 $\pm$ 0.57   \\
                      &  CRCP &  \textbf{1.2 $\pm$ 0.03} & \textbf{93.7 $\pm$ 0.62} & \textbf{1.2 $\pm$ 0.03} & \textbf{93.7 $\pm$ 0.62}  &  \textbf{1.1 $\pm$ 0.01} & \textbf{95.7 $\pm$ 0.18}  \\
                      & NRSCP &   2.2 $\pm$ 0.07 &  98.9 $\pm$ 0.12  &  3.3 $\pm$ 0.20 &  99.6 $\pm$ 0.08   &  1.2 $\pm$ 0.01 &  97.8 $\pm$ 0.12   \\
                      &  NACP &  1.3 $\pm$ 0.04 & 94.4 $\pm$ 0.62 & 1.3 $\pm$ 0.04 & 94.5 $\pm$ 0.62  &  1.1 $\pm$ 0.01 & 95.9 $\pm$ 0.18  \\
                     \hline
                  
                              &  CP (oracle) &  6.5 $\pm$ 0.20 & 90.0 $\pm$ 0.43 & 4.0 $\pm$ 0.08 & 90.0 $\pm$ 0.43  &  2.0 $\pm$ 0.03 & 90.0 $\pm$ 0.43  \\
                              &  Noisy-CP &  50.5 $\pm$ 1.29 & 99.8 $\pm$ 0.04 & 50.5 $\pm$ 1.33 & 99.8 $\pm$ 0.03  &  50.1 $\pm$ 1.34 & 99.9 $\pm$ 0.02  \\
                                &  NR-CP (w/o $\Delta$) &  {6.4 $\pm$ 0.28} & \  {89.9 $\pm$ 0.54} & {4.0 $\pm$ 0.11} & \  {89.9 $\pm$ 0.55}  &  \  {2.0 $\pm$ 0.06} & \  {89.9 $\pm$ 0.56}  \\ 
    CIFAR-100        &  ACNL &  100.0 $\pm$ 0.00 & 100.0 $\pm$ 0.00 & 100.0 $\pm$ 0.00 & 100.0 $\pm$ 0.00   &  100.0 $\pm$ 0.00 & 100.0 $\pm$ 0.00  \\
 (100 classes)  &  ACNL$^+$ &   50.4 $\pm$ 1.13 & 99.8 $\pm$ 0.03  & 50.4 $\pm$ 1.01 & 99.9 $\pm$ 0.03   &   50.1 $\pm$ 1.23 & 99.9 $\pm$ 0.02   \\
                              &  CRCP &  25.7 $\pm$ 3.71 & 98.7 $\pm$ 0.39 & 8.5 $\pm$ 0.41 & 98.3 $\pm$ 0.16  &  11.1 $\pm$ 3.46 & 98.7 $\pm$ 0.42  \\
                              & NRSCP &   37.6 $\pm$ 0.86 &  99.6 $\pm$ 0.06  &  50.1 $\pm$ 1.08 &  99.8 $\pm$ 0.03   &  11.7 $\pm$ 0.37 &  98.9 $\pm$ 0.08   \\
                              &  NACP &  \textbf{9.0 $\pm$ 0.46} & \textbf{93.0 $\pm$ 0.49} & \textbf{4.8 $\pm$ 0.13} & \textbf{93.0 $\pm$ 0.48}  &  \textbf{2.5 $\pm$ 0.09} & \textbf{93.0 $\pm$ 0.52}  \\
                            \hline
                              &  CP (oracle) &  14.9 $\pm$ 0.60 & 90.0 $\pm$ 0.61 & 6.9 $\pm$ 0.19 & 90.0 $\pm$ 0.62  &  3.8 $\pm$ 0.13 & 90.02 $\pm$ 0.58  \\
                              &  Noisy-CP &  99.7 $\pm$ 3.67 & 99.7 $\pm$ 0.08 & 101.4 $\pm$ 3.58 & 99.5 $\pm$ 0.09  &  98.3 $\pm$ 3.80 & 99.8 $\pm$ 0.05  \\
                                &  NR-CP (w/o $\Delta$) &  \  {14.0 $\pm$ 0.91} & \  {89.4 $\pm$ 0.81} & {6.7 $\pm$ 0.27} & \  {89.3 $\pm$ 0.80}  &  \  {3.5 $\pm$ 0.24} & \  {89.3 $\pm$ 0.80}  \\ 
                 TinyImagenet        &  ACNL & 200.0 $\pm$ 0.00 & 100.0 $\pm$ 0.00 & 200.0 $\pm$ 0.00 & 100.0 $\pm$ 0.00  &  200.0 $\pm$ 0.00 & 100.0 $\pm$ 0.00  \\
 (200 classes)  &  ACNL$^+$ & 99.7 $\pm$ 3.67 & 99.7 $\pm$ 0.08 & 101.4 $\pm$ 3.58 & 99.5 $\pm$ 0.09  &  98.3 $\pm$ 3.80 & 99.8 $\pm$ 0.05 \\
                        &  CRCP &  200.0 $\pm$ 0.00 & 100.0 $\pm$ 0.00 & 200.0 $\pm$ 0.00 & 100.0 $\pm$ 0.00 &  200.0 $\pm$ 0.00 & 100.0 $\pm$ 0.00  \\
                        & NRSCP &   79.6 $\pm$ 2.82 &  99.3 $\pm$ 0.11  &  100.6 $\pm$ 2.88 &  99.5 $\pm$ 0.09   &  28.1 $\pm$ 1.20 &  98.1 $\pm$ 0.15   \\
                              &  NACP &  \textbf{22.6 $\pm$ 1.87} & \textbf{93.7 $\pm$ 0.71} & \textbf{9.0 $\pm$ 0.50} & \textbf{93.6 $\pm$ 0.70} &  \textbf{7.0 $\pm$ 0.87} & \textbf{93.6 $\pm$ 0.72}  \\
                            \hline
                              &  CP (oracle) &  16.6 $\pm$ 0.33 & 90.0 $\pm$ 0.26 & 6.3 $\pm$ 0.06 & 90.0 $\pm$ 0.27  &  3.6 $\pm$ 0.07 & 90.0 $\pm$ 0.28  \\
                              &  Noisy-CP &  502.6 $\pm$ 8.56 & 99.9 $\pm$ 0.01 & 501.6 $\pm$ 8.51 & 99.9 $\pm$ 0.01  &  501.3 $\pm$ 10.2 & 100.0 $\pm$ 0.01  \\
                                                  &  NR-CP (w/o $\Delta$) &  {16.7 $\pm$ 0.51} & \  {90.0 $\pm$ 0.34} & {6.3 $\pm$ 0.10} & \  {90.0 $\pm$ 0.36}  &  {3.6 $\pm$ 0.14} & \  {90.0 $\pm$ 0.38}  \\
                  ImageNet        &  ACNL &  1000.0 $\pm$ 0.00 & 100.0 $\pm$ 0.00 & 1000.0 $\pm$ 0.00 & 100.0 $\pm$ 0.00  &  1000.0 $\pm$ 0.00 & 100.0 $\pm$ 0.00  \\
 (1000 classes)  &  ACNL$^+$ & 502.6 $\pm$ 8.56 & 99.9 $\pm$ 0.01 & 501.6 $\pm$ 8.51 & 99.9 $\pm$ 0.01  &  501.3 $\pm$ 10.2 & 100.0 $\pm$ 0.01  \\
               &  CRCP &  1000.0 $\pm$ 0.00 & 100.0 $\pm$ 0.00 & 1000.0 $\pm$ 0.00 & 100.0 $\pm$ 0.00  &  1000.0 $\pm$ 0.00 & 100.0 $\pm$ 0.00  \\
               & NRSCP &   275.6 $\pm$ 27.1 &  99.7 $\pm$ 0.06  &  455.2 $\pm$ 20.7 &  99.9 $\pm$ 0.02   &  55.9 $\pm$ 0.87 &  99.1 $\pm$ 0.02   \\
                              &  NACP &  \textbf{20.9 $\pm$ 0.72} & \textbf{91.9 $\pm$ 0.32} & \textbf{7.1 $\pm$ 0.13} & \textbf{91.9 $\pm$ 0.34}  &  \textbf{4.8 $\pm$ 0.23} & \textbf{91.9 $\pm$ 0.36}  \\
           \hline
\end{tabular}
    }
\label{many_classes_datasets}
\end{table*}

{\bf Evaluation Measures}.  We evaluated each CP method   based on the average size of the prediction sets (where a small value means high efficiency) and the fraction of test samples for which the prediction sets contained the ground-truth labels. The two evaluation metrics are formally defined as:
$$ \textrm{size} = \frac{1}{n} \sum_i | C(x_i) |,  \hspace{0.3cm}
 \textrm{coverage} = \frac{1}{n} \sum_i
{\bf 1}(y_i \in C(x_i))$$
such that  $n$ is the size of the test set.  We report the mean and standard deviation over 1000 random splits.

 \textbf{Datasets.}
We show results on four standard scenery image datasets {CIFAR-10},
 {CIFAR-100}
 \citep{Krizhevsky2009}, {Tiny-ImageNet}, and ImageNet \citep{Deng2009}.

\textbf{Implementation details.} Each task was trained by fine-tuning on a pre-trained ResNet-18 \citep{he2016cvpr} network. The models were taken from the PyTorch site\footnote{
\url{https://pytorch.org/vision/stable/models.html}}. We selected this network architecture because of its widespread use in classification problems.
The last fully connected layer output size was modified to fit the corresponding number of classes for each dataset. 
For the standard dataset evaluated in Table \ref{many_classes_datasets} we used publicly available checkpoints. For each dataset, we combined the validation and test sets and then constructed 1000 different splits where 50\% was used for the calibration phase and 50\% was used for testing. In all our experiments we used $\delta=0.001$. In other words, the computed
coverage guarantee is applied to the sampled noisy validation set with probability 0.999

\begin{table}
    \caption{Finite sample correction terms $\Delta$ of NACP, ACNL \citep{sesia2023adaptive} and CRCP \citep{Clarkson2024},  for several datasets and two noise levels, $n$ is the size of the validation set.}
    \centering
     \scalebox{.90}{
    \begin{tabular}{lrr|cc|cc|cc}
      Dataset &  $n $&\#classes &  \multicolumn{2}{c}{NACP  }  & \multicolumn{2}{c}{ACNL}  &
        \multicolumn{2}{c}{CRCP} \\
                 & & & $\epsilon=0.1$ & $ \epsilon=0.2$ &
            $\epsilon=0.1$ &
            $\epsilon=0.2$ &
            $\epsilon=0.1$ &
            $\epsilon=0.2$ \\
            \hline
            CIFAR-10 & 5000 &10& 0.035 & 0.043 & 0.031 & 0.059 & \textbf{0.016} & \textbf{0.036}\\
            CIFAR-100 & 10000 &100& \textbf{0.025} & \textbf{0.030} & 0.077 & 0.163 & 0.039 & 0.088 \\
            TinyImagenet & 5000 & 200 & \textbf{0.035} & \textbf{0.043} & 0.175 & 0.382 & 0.078 & 0.176 \\
            ImageNet & 25000 &1000  & \textbf{0.016} & \textbf{0.019} & 0.194 & 0.466 & 0.079 & 0.177\\ \hline
    \end{tabular}
           }
    \label{tab:deltas_only_cv}
\end{table}

 \textbf{Conformal prediction results.}
 Table \ref{many_classes_datasets} reports the noisy label calibration results across 3 different conformal prediction scores, HPS, APS, and RAPS for four standard publicly available datasets,  CIFAR-10, CIFAR-100, Tiny-ImageNet, and ImageNet. In all cases, we used $1\!-\!\alpha=0.9$ and a noise level of $\epsilon=0.2$. The results indicate that in the case of a validation set with noisy labels, the Noisy-CP threshold became larger to facilitate the uncertainty induced by the noisy labels. This yielded larger prediction sets and the coverage was higher than the target coverage which was set to $90\%$. We can see that  NACP outperformed the ACNL, and CRCP methods for all datasets except for CIFAR-10 with fewer classes. Following Theorem \ref{theorem3-2}, we expect the gain in performance when using NACP versus Noisy-CP to increase with the number of classes, indeed validated empirically in Table \ref{many_classes_datasets}.  Here for CIFAR-100, Tiny-ImageNet, and ImageNet the ACNL  and CRCP  methods failed due to the large number of classes and the relatively small number of samples per class. For TinyImagenet and Imagenet, ACNL$^+$ falls back to Noisy-CP. A direct comparison of the finite sample correction terms $\Delta$ obtained by NACP, ACNL and CRCP is shown in Table  \ref{tab:deltas_only_cv}. Note that if $1-\alpha+\Delta >1$, the prediction set includes all the classes and thus it becomes useless. We can see in Table \ref{tab:deltas_only_cv} that this is the case for ACNL and CRCP in datasets with a large number of classes.

\textbf{Correction term analysis.} Following the theoretical and empirical results, the effectiveness of our method and baselines can be fully explained by the correction terms $\Delta$ each method guarantees as practitioners require coverage guarantee and therefore will use $1-\alpha+\Delta$.  Note that, as explained in Section \ref{sec:nacp_related_work},
our finite sample coverage guarantee is different from the one provided by the baseline method. Figure \ref{fig:corr_terms_analysis} presents the correction term as a function of calibration set size and the number of classes. 
Note that the NACP curve remains exactly the same across the 3 plots. 
Our main  contribution is grounded in the fact that NACP is not dependent on the number of classes $k$, clearly shown in plots as the number of classes grows.

\begin{figure}
  \begin{center}
    \includegraphics[width=0.3\textwidth]{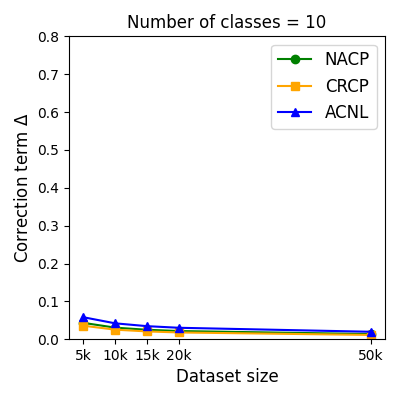}
    \includegraphics[width=0.3\textwidth]{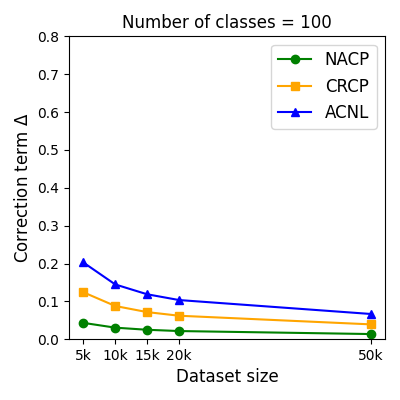}
    \includegraphics[width=0.3\textwidth]{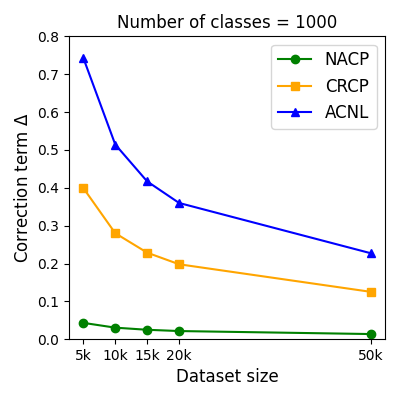}
  \end{center}
  \caption{Correction terms $\Delta$ of NACP, ACNL and CRCP as a function of the validation set size $n$ given $\epsilon=0.2$.
  We show results for 3 numbers of classes, 10, 100 and 1000.}
  \label{fig:corr_terms_analysis}
\end{figure}


\textbf{General noise transition matrix.}
Finally, we evaluate NRSCP and NACP on two common general noise matrices: Neighborhood noise and Random noise (see details in Section \ref{s:noisy}). While existing final sample terms bounds are not effective, in practice NACP  (without a finite sample correction) achieves the required coverage guarantee 
and the average prediction size is similar to the one obtained by the noise-free CP. We observe the same pattern when using uniform noise. This indicates that the current coverage guarantee bounds are too conservative. 
Table \ref{generalNoiseResults} shows the results on the CIFAR-100 dataset and the rand-APS technique when using NACP without finite sample correction term $\Delta$. Results show a clear dominance of NACP over Noisy-CP and NRCP on the two different noise models, presenting the robustness of NACP across various noise models. ACNL  (without the finite sample term) achieves here similar results. 

\begin{table*}
\caption{Rand-APS calibration results for $1\!-\!\alpha$ = 0.9  on CIFAR-100 dataset and two noise models. We report the mean and the std over 1000 different splits.}
\centering
\scalebox{.90}{
    \begin{tabular}{l|rr|rr}
    &  \multicolumn{2}{c|}{Neighborhood noise}  & \multicolumn{2}{c}{Random noise }  \\
    \hline
    CP Method & size $\downarrow$ &  coverage (\%)   &  size $\downarrow$ &  coverage (\%)     \\ \hline
    CP (oracle) &         6.48 $\pm$ 0.19 & 90.01 $\pm$ 0.41 &  6.48 $\pm$ 0.19 & 90.01 $\pm$ 0.41      \\
    Noisy-CP &  48.89 $\pm$ 1.13 & 99.80 $\pm$ 0.04   &  50.25 $\pm$ 1.37 & 99.82 $\pm$ 0.04 \\
    NRSCP & 12.82 $\pm$ 0.36 &  95.62 $\pm$ 0.21 & 37.01 $\pm$ 0.88 & 99.53 $\pm$ 0.06\\
    NR-CP (w/o $\Delta$)  &  \textbf{6.52 $\pm$ 0.22}& 90.03 $\pm$ 0.47
    & \textbf{6.45 $\pm$ 0.30} & 89.97 $\pm$ 0.57 \\  \hline
    \label{generalNoiseResults}
    \end{tabular}
    }
\end{table*}

Next, we show the results of the following experiments. In \emph{Calibration set size} we test the performance of various conformal prediction methods under noisy labels as a function of the calibration set size.  In \emph{NACP Agnostic to Different model architectures} we show that NACP is agnostic to different classification network architectures.
Finally, in \emph{Experiments on real noisy datasets} we report experiments on real noisy datasets where the noise is due to manual annotation error. We show that in this case,  by imposing a uniform noise model, we get better results than the one obtained by ignoring the noise and applying CP directly on the noisy validation set.   

\textbf{Calibration set size.}  In the following experiment, we test the performance of various conformal prediction methods under noisy labels as a function of the calibration set size on the ImageNet dataset. Figure \ref{fig:size_sweep} shows the mean size and coverage as a function of the calibration set size. In addition, the correction term $\Delta$ is depicted for ImageNet for each calibration set size. Results show that even with as little as 2500 images that correspond to 2.5 images per class the calibration results are almost on par with the oracle calibration given clean labels.

\begin{figure}
    \centering
    \includegraphics[width=0.32\linewidth]{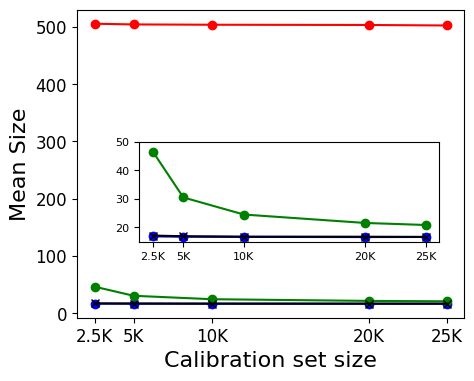}
    \includegraphics[width=0.32\linewidth]{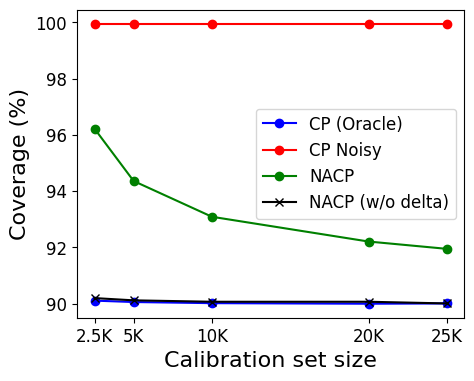} 
    \includegraphics[width=0.32\linewidth]{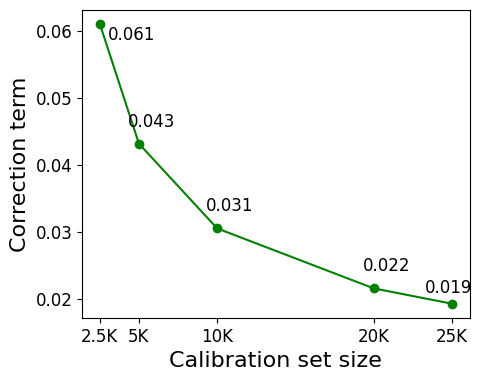} \\
    \hspace{1cm} (a) \hspace{4cm} (b) \hspace{4cm} (c)
    \caption{Noisy labels conformal prediction on ImageNet with different calibration set sizes. (a) Mean size (b) Coverage (\%), and (c) Correction terms $\Delta$ as a function of calibration set size.}
    \label{fig:size_sweep}
\end{figure}

\textbf{NACP Agnostic to Different model architectures.} Conformal prediction in general and our method NACP specifically has no assumption and is agnostic to the underlying model architecture.
In the following section, we verify that by experimenting with ImageNet across different model architectures. Table \ref{many_archs} presents the results of applying conformal prediction with and without noisy labels on ResNet18, ResNet50, DenseNet121, ViT-B16 (Vision transformer).

\begin{table}
\centering
\caption{ CP calibration results on ImageNet and various model architectures for $1\!-\!\alpha$ = 0.9
and $\epsilon=0.2$.  We report the mean and the std over 1000 different splits. Bold for best result with theoretical guarantees.}
\label{raps2appn}
   \scalebox{.65}{
\begin{tabular}{ll|rr|rr|rr|rr}
 & & \multicolumn{2}{c|}{ResNet-18} & \multicolumn{2}{c|}{ResNet-50}  & \multicolumn{2}{c|}{DenseNet121}  & \multicolumn{2}{c}{ViT-B16} \\
\hline
Dataset      &  CP Method & size $\downarrow$ &  coverage(\%)   &  size $\downarrow$ &  coverage(\%) & size $\downarrow$ &  coverage(\%)    & size $\downarrow$ &  coverage(\%)     \\ \hline

APS    &  CP (oracle) &  16.6 $\pm$ 0.33 & 90.0 $\pm$ 0.26  &  13.9 $\pm$ 0.34 & 90.0 $\pm$ 0.28   &   12.0 $\pm$ 0.28 & 90.0 $\pm$ 0.27  & 10.7 $\pm$ 0.38 & 90.0 $\pm$ 0.25   \\
            &  Noisy-CP &  502.6 $\pm$ 8.56 & 99.9 $\pm$ 0.01  &  505.5 $\pm$ 8.11 & 99.9 $\pm$ 0.01   &   502.8 $\pm$ 8.46 & 99.9 $\pm$ 0.01  & 506.8 $\pm$ 8.14 & 99.8 $\pm$ 0.02   \\
            &  NR-CP (w/o $\Delta$)     &  \  {16.7 $\pm$ 0.51} & \  {90.0 $\pm$ 0.34}  &  13.9 $\pm$ 0.47 & 90.0 $\pm$ 0.37   &   12.0 $\pm$ 0.38 & 90.0 $\pm$ 0.34  & 10.7 $\pm$ 0.55 & 90.0 $\pm$ 0.35   \\ 
            & ACNL & 1000.0 $ \pm 0.00$  & 100.0  $\pm$ 0.00   & 1000.0 $ \pm 0.00$  & 100.0  $\pm$ 0.00  & 1000.0 $ \pm 0.00$  & 100.0  $\pm$ 0.00  & 1000.0 $ \pm 0.00$  & 100.0  $\pm$ 0.00   \\
            &  ACNL$^+$ &  502.6 $\pm$ 8.56 & 99.9 $\pm$ 0.01  &  505.5 $\pm$ 8.11 & 99.9 $\pm$ 0.01   &   502.8 $\pm$ 8.46 & 99.9 $\pm$ 0.01  & 506.8 $\pm$ 8.14 & 99.8 $\pm$ 0.02   \\
            & CRCP & 1000.0 $ \pm 0.00$  & 100.0  $\pm$ 0.00  & 1000.0 $ \pm 0.00$  & 100.0  $\pm$ 0.00  & 1000.0 $ \pm 0.00$  & 100.0  $\pm$ 0.00  & 1000.0 $ \pm 0.00$  & 100.0  $\pm$ 0.00                   \\
            &  NACP   &  \textbf{20.9 $\pm$ 0.72} & \textbf{91.9 $\pm$ 0.32}  &  \textbf{17.4 $\pm$ 0.62} & \textbf{91.9 $\pm$ 0.36}   &   \textbf{15.1 $\pm$ 0.55} & \textbf{91.9 $\pm$ 0.34}  & \textbf{15.5 $\pm$ 0.81} & \textbf{91.9 $\pm$ 0.31} \\
            \hline               
RAPS   &  CP (oracle)  &  6.3 $\pm$ 0.06 & 90.0 $\pm$ 0.27  &  4.5 $\pm$ 0.05 & 89.9 $\pm$ 0.29   &   4.7 $\pm$ 0.06 & 90.0 $\pm$ 0.26  &  2.6 $\pm$ 0.04 & 90.0 $\pm$ 0.25   \\
            &  Noisy-CP  &  501.6 $\pm$ 8.51 & 99.9 $\pm$ 0.01  &  501.1 $\pm$ 8.85 & 99.9 $\pm$ 0.01   &   501.9 $\pm$ 8.80 & 99.9 $\pm$ 0.01  & 505.8 $\pm$ 7.90 & 99.9 $\pm$ 0.01   \\
            &  NR-CP (w/o $\Delta$)    &  \  {6.3 $\pm$ 0.10} & \  {90.0 $\pm$ 0.36}  &  4.5 $\pm$ 0.06 & 90.0 $\pm$ 0.35  & 4.7 $\pm$ 0.08 & 90.0 $\pm$ 0.36  &   2.6 $\pm$ 0.05 & 90.0 $\pm$ 0.36    \\
            & ACNL & 1000.0 $ \pm 0.00$  & 100.0  $\pm$ 0.00  & 1000.0 $ \pm 0.00$  & 100.0  $\pm$ 0.00  & 1000.0 $ \pm 0.00$  & 100.0  $\pm$ 0.00  & 1000.0 $ \pm 0.00$  & 100.0  $\pm$ 0.00            \\
            &  ACNL$^+$  &  501.6 $\pm$ 8.51 & 99.9 $\pm$ 0.01  &  501.1 $\pm$ 8.85 & 99.9 $\pm$ 0.01   &   501.9 $\pm$ 8.80 & 99.9 $\pm$ 0.01  & 505.8 $\pm$ 7.90 & 99.9 $\pm$ 0.01   \\
            & CRCP & 1000.0 $ \pm 0.00$  & 100.0  $\pm$ 0.00  & 1000.0 $ \pm 0.00$  & 100.0  $\pm$ 0.00  & 1000.0 $ \pm 0.00$  & 100.0  $\pm$ 0.00  & 1000.0 $ \pm 0.00$  & 100.0  $\pm$ 0.00                   \\
            &  NACP    &  \textbf{7.1 $\pm$ 0.13} & \textbf{91.9 $\pm$ 0.34}  &  \textbf{5.0 $\pm$ 0.08} & \textbf{91.9 $\pm$ 0.34}   &   \textbf{5.3 $\pm$ 0.10} & \textbf{92.0 $\pm$ 0.35}  &  \textbf{2.9 $\pm$ 0.07} & \textbf{92.0 $\pm$ 0.30}    \\
            \hline                
HPS         &  CP (oracle) &  3.6 $\pm$ 0.07 & 90.0 $\pm$ 0.28  &  2.0 $\pm$ 0.03 & 90.0 $\pm$ 0.28   &   2.4 $\pm$ 0.03 & 90.0 $\pm$ 0.25  &   1.5 $\pm$ 0.02 & 90.0 $\pm$ 0.26   \\
            &  Noisy-CP  &  501.3 $\pm$ 10.2 & 100.0 $\pm$ 0.01   &  502.4 $\pm$ 9.50 & 99.9 $\pm$ 0.01   &   502.3 $\pm$ 10.3 & 99.9 $\pm$ 0.20  &  504.3 $\pm$ 8.19 & 99.9 $\pm$ 0.01   \\
             &  NR-CP (w/o $\Delta$)    &  \  {3.6 $\pm$ 0.14} & \  {90.0 $\pm$ 0.38} & 2.1 $\pm$ 0.06 & 90.0 $\pm$ 0.38  &  2.4 $\pm$ 0.07 & 90.0 $\pm$ 0.34   &   1.5 $\pm$ 0.03 & 90.0 $\pm$ 0.35   \\
            & ACNL & 1000.0 $ \pm 0.00$  & 100.0  $\pm$ 0.00  & 1000.0 $ \pm 0.00$  & 100.0  $\pm$ 0.00  & 1000.0 $ \pm 0.00$  & 100.0  $\pm$ 0.00  & 1000.0 $ \pm 0.00$  & 100.0  $\pm$ 0.00                   \\
            &  ACNL$^+$  &  501.3 $\pm$ 10.2 & 100.0 $\pm$ 0.01   &  502.4 $\pm$ 9.50 & 99.9 $\pm$ 0.01   &   502.3 $\pm$ 10.3 & 99.9 $\pm$ 0.20  &  504.3 $\pm$ 8.19 & 99.9 $\pm$ 0.01   \\
            & CRCP & 1000.0 $ \pm 0.00$  & 100.0  $\pm$ 0.00  & 1000.0 $ \pm 0.00$  & 100.0  $\pm$ 0.00  & 1000.0 $ \pm 0.00$  & 100.0  $\pm$ 0.00  & 1000.0 $ \pm 0.00$  & 100.0  $\pm$ 0.00                   \\
            &  NACP   &  \textbf{4.8 $\pm$ 0.23} & \textbf{91.9 $\pm$ 0.36}  &  \textbf{2.6 $\pm$ 0.10} & \textbf{91.9 $\pm$ 0.37}      & \textbf{3.1 $\pm$ 0.12} & \textbf{91.9 $\pm$ 0.34} & \textbf{1.7 $\pm$ 0.04} & \textbf{91.9 $\pm$ 0.33}  \\
           \hline
\end{tabular}
    }
\label{many_archs}
\end{table}

\textbf{Experiments on real noisy datasets.} We evaluate our methods on real-world noisy datasets, focusing on the CIFAR-10N dataset, which contains human annotation errors and was introduced in \cite{wei2021learning}. Specifically, we analyze four variations of CIFAR-10N: CIFAR-10-aggregate and CIFAR-10-random-{1,2,3}. CIFAR-10-aggregate combines three noisy labels using majority voting. If the three submitted labels differ, the aggregated label is randomly selected from the three options. CIFAR-10-random-i ($i \in \{1, 2, 3\}$) refers to the i-th submitted label for each image. Importantly, the data collection process ensures that no image is labeled multiple times by the same annotator.

While this noise model realistically reflects human annotation behavior, it does not adhere to a strict uniform noise distribution. For our experiments with NACP and baseline methods, we adopt the noise ratio reported in \cite{wei2021learning} to compute $\epsilon$ for the noise-aware conformal prediction algorithm. Notably, \cite{wei2021learning} defines the noise rate with $\epsilon$ such that $T_{i,i} = 1-\epsilon$. In contrast, our notation uses $\epsilon$ with the formulation $T_{i,i} = 1-\epsilon + \frac{\epsilon}{k}$. Consequently, we adjust $\epsilon$ values from the original paper to align with our approach.

The purpose of this experiment is to demonstrate that even when the true noise—arising from annotators—is not exactly uniform, approximating it as such can still yield effective performance in real-world datasets and scenarios. 

Table \ref{real_cifar10n} summarizes the results for the four CIFAR-10N variations. In this real-world scenario, clean labels (CP Oracle) are unavailable. Instead, Noisy-CP results reflect calibration using annotator-provided labels as-is.
Our noise-aware approach demonstrates a consistent improvement over this baseline.

\begin{table}[H]
\centering
\caption{ CP calibration results on CIFAR-10N for $1\!-\!\alpha$ = 0.9
and $\epsilon=0.2$.  We report the mean and the std over 1000 different splits. Bold for best result with theoretical guarantees.}
\label{raps21}
   \scalebox{.73}{
\begin{tabular}{ll|cc|cc|cc|cc}
 & & \multicolumn{2}{c|}{CIFAR-10N-aggregate}  & \multicolumn{2}{c|}{CIFAR-10N-random-1}  & \multicolumn{2}{c}{CIFAR-10N-random-2} & \multicolumn{2}{c}{CIFAR-10N-random-3} \\
 & & \multicolumn{2}{c|}{(10.0\%)}  & \multicolumn{2}{c|}{(19.1\%)}  & \multicolumn{2}{c}{ (20.1\%)} & \multicolumn{2}{c}{(19.6\% )} \\
\hline
Dataset      &  CP Method & size $\downarrow$ &  coverage(\%)   &  size $\downarrow$ &  coverage(\%) & size $\downarrow$ &  coverage(\%)    & size $\downarrow$ &  coverage(\%)     \\ \hline

rand-APS          &  Noisy-CP &  \textbf{1.76} & \textbf{91.73} & 2.00 & 93.18 & 2.13 & 93.39 & 2.44 & 96.45  \\
                  &  NACP    &  1.58 & 88.49 & \textbf{1.82} & \textbf{91.38} & \textbf{1.93} & \textbf{91.67} & \textbf{2.13} & \textbf{94.50}  \\ \hline
                  
HPS               &  Noisy-CP &  \textbf{1.32} & \textbf{91.19} & 1.67 & 92.97 & 1.79 & 93.15 & 2.16 & 96.31  \\
                  &  NACP   &  1.38 & 91.85 & \textbf{1.45} & \textbf{91.00} & \textbf{1.47} & \textbf{90.41} & \textbf{1.53} & \textbf{92.18}  \\ \hline

\end{tabular}
    }
\label{real_cifar10n}
\end{table}
\chapter{Local Differential Private Conformal Prediction}
\label{ch:ldpcp}

In this chapter, we explore the intersection of conformal prediction (CP) and privacy-preserving techniques, introducing a Local Differentially Private Conformal Prediction (LDP-CP) framework. While traditional CP methods rely on cleanly labeled calibration sets, data privacy concerns often prevent access to true labels, particularly in sensitive domains like medical or personal data. Our proposed framework addresses the challenge of maintaining valid prediction set coverage while protecting user data through Local Differential Privacy (LDP). We present two complementary approaches—LDP-CP-L and LDP-CP-S—that cater to varying privacy goals, computational resources, and data scenarios, offering a robust solution for secure, decentralized data handling.

\section{Problem Statement}

Conformal prediction typically relies on having access to a cleanly labeled calibration set to set the CP threshold. However, in many real-world scenarios, such clean labels may be unavailable due to data privacy concerns.  For instance, in scenarios involving medical or personal data, or if the data aggregator (e.g.\ a cloud-based ML service) is considered \emph{untrusted}, privacy constraints may prohibit access to true labels and data, requiring a privacy-preserving mechanism that provides \emph{Local Differential Privacy} (LDP). In such situations, the aggregator might only be allowed to view \emph{noisy versions} of the labels.

Next we tackle the challenge of applying conformal prediction when validation-set labels (or conformity scores) must be protected through privacy-preserving mechanisms. Specifically, we introduce a Local Differentially Private Conformal Prediction (LDP-CP) framework that balances privacy with real-world considerations such as user-side computational capacity, aggregator trustworthiness, and intellectual property.

We present two complementary LDP-CP solutions. LDP-CP-L locally perturbs labels using a randomized response, that shifts all score-related computations to the aggregator. This design suits cases where users have minimal computational resources, or the model’s internal structure is not disclosed to them, while not sharing their sensitive labels. However,  it only achieves label-DP \cite{beimel2013private, ghazi2021deep}. In contrast, LDP-CP-S allows users to generate and locally randomize their own conformity scores, which is ideal for scenarios where both feature and label privacy are paramount and the user can handle additional computational tasks. We also offer guidelines on choosing which method aligns better with specific privacy goals, resource constraints, and per-dataset properties such as sample size and the number of classes.

By considering both score-based and label-based perturbations, we provide a flexible framework that adapts to various privacy budgets and computational setups. This framework guarantees valid coverage for the true labels asymptotically and in finite samples. As a result, LDP-CP aligns well with modern demands for secure, decentralized data handling.

\paragraph{Relationship to Past Work.} Conformal prediction in the presence of label noise has received growing attention \citep{einbinder2022conformal,sesia2023adaptive, Clarkson2024} and Chapter \ref{ch:nrcp_nacp},
 with methods that adjust the calibration threshold when a known noise matrix corrupts labels. Simultaneously, LDP has become a leading approach for protecting sensitive data at the user end \citep{warner1965randomized,kairouz2016discrete,wang2017locally}, allowing users to randomize their own inputs (e.g., labels) before sharing with an untrusted aggregator. We bridge these lines of work by recognizing that local DP can be \emph{seen as a known noisy channel} on the labels or conformal scores, we can plug that channel into a ``noise-aware" conformal procedure. The end result is a conformal predictor whose coverage remains valid, while also protecting each user's label or score via $\varepsilon$-LDP. To the best of our knowledge, there is no prior work on LDP conformal prediction. Thus, we are the first to research local differential private conformal prediction. The closest work to ours is that of \citet{angelopoulos2022private}, which suggests a centrally differentially private conformal prediction procedure where a trusted aggregator has access to raw data. In our setting, the aggregator never observes true labels directly, hence, closing a key gap in privacy-preserving conformal prediction research. 
 
This dissertation makes three key contributions:
\begin{itemize}
    \item We introduce two complementary LDP-CP methods, LDP-CP-S and LDP-CP-L, that accommodate different privacy constraints and computational setups. (Figures~\ref{fig:main_pipeline_l} and ~\ref{fig:main_pipeline_s}).
    \item We prove coverage guarantees for both methods under LDP constraints. (Theorems ~\ref{theorem:main_ldpcpl} and~\ref{theorem:main_ldpcps}).
    \item We demonstrate the feasibility and effectiveness of our approaches in privacy-sensitive applications, such as medical data analysis and untrusted cloud-based ML services (see Section~\ref{sec:exps}).
\end{itemize}
This work bridges the gap between privacy-preserving mechanisms and conformal prediction, by providing a foundation for robust and private uncertainty quantification in real-world scenarios. Unlike other methods that rely on trusted aggregators, our approaches ensure privacy directly at the user level, which aligns with modern demands for decentralized privacy protection.


\section{Local-DP Conformal Prediction on Labels}
\label{subsec:ldp-cp-labels}

In many real-world contexts, users lack the ability—or permission—to compute model-based scores on their own devices. This may be due to limited computational resources, a restricted API that only provides predictions, or a proprietary model architecture. To address these cases, we propose an LDP mechanism that randomizes each user’s label and then applies a noise-aware conformal calibration at the aggregator. This design protects sensitive labels while remaining model-agnostic, since the aggregator performs all the scoring steps. Figure \ref{fig:main_pipeline_l} illustrates the overall pipeline.

\begin{figure*}[t!]
    \centering
    \includegraphics[width=\linewidth]{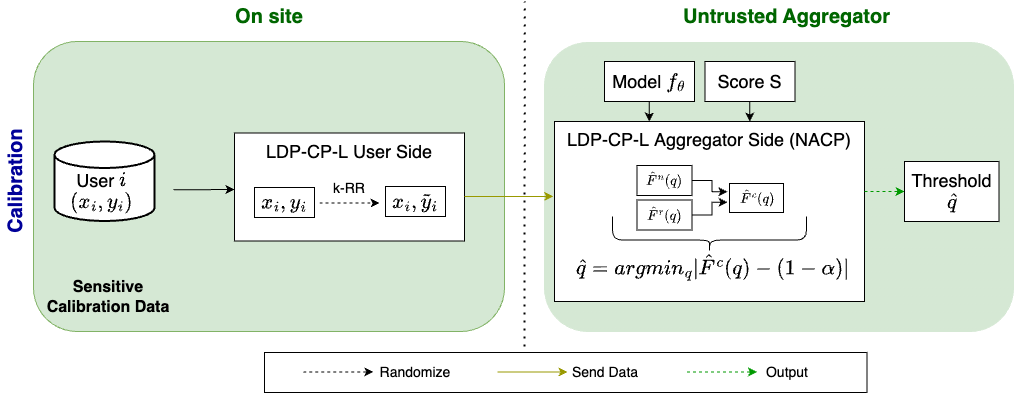}
    \caption{Local Differential Private Conformal Prediction (LDP-CP-L) Pipeline (Best viewed in color).}
    \label{fig:main_pipeline_l}
\end{figure*}

\begin{algorithm}
\caption{LDP-CP-L - User-side Procedure}
\label{alg:ldp_cp_labels_user}
\begin{algorithmic}[1]
\State \textbf{Input:} Calibration set $\{(x_i, y_i)\}_{i=1}^{n}$, privacy parameter $\epsilon$, number of labels $k$
\State \textbf{Output:} Noisy calibration set $\{(x_i, \tilde{y}_i)\}_{i=1}^{n}$
\State Set $\beta = \frac{k}{k-1+e^{\epsilon}}$
\For {each user $i \in \{1, \dots, n\}$}
    \State Apply $k$-ary Randomized Response (RR) to label $y_i$ to obtain $\tilde{y}_i$: \\
    \[
    p(\tilde{y} = t' \mid y=t)
    =\,
    {1}_{\{t=t'\}}(1-\beta) + \frac{\beta}{k}
    \]
    \State Send $(x_i, \tilde{y}_i)$ to the aggregator
\EndFor
\end{algorithmic}
\label{alg:ldpcpl-user}
\end{algorithm}

We next unpack LDP-CP-L by first describing the user's side procedure that applies to its data point, and then the procedure done by the untrusted aggregator to compute the conformal prediction threshold.

\textbf{User's side Procedure.} For each user \(i\) with pair $(x_i, y_i)$:
\begin{enumerate}
\item User applies the  LDP mechanism - $k$-ary RR to their label. This yields $\tilde{y}_i$.
\item Report $(x_i, \tilde{y}_i)$ to the aggregator.
\end{enumerate}

\textbf{Aggregator's Procedure.} Once the aggregator collects $\{x_i, \tilde{y}_i\}_{i=1}^{n}$ (it does not see $y_i$), it runs a ``noise-aware'' conformal method on the pairs $(x_i,\tilde{y}_i)$. Since the aggregator knows the noise model, it can estimate the coverage for the \emph{true} label.

A key advantage of this label-perturbation strategy is that the aggregator never observes raw labels, thus meeting label-LDP guarantees. Meanwhile, by modeling a  $k$-ary Randomized Response as a known noise channel, the aggregator recovers the necessary calibration adjustments to preserve near-correct coverage on the true labels. In what follows, we outline the procedure in more detail, leading to our first main theorem and contribution culminating in Theorem~\ref{theorem:main_ldpcpl}.

For $k$-RR, each label is replaced by a random label with probability $\beta$, known \emph{channel} from $y$ to $\tilde{y}$. Then,
\[
p(\tilde{y} = j \mid y=i)
=\,
{1}_{\{i=j\}}(1-\beta) + \frac{\beta}{k}; \quad \beta = \frac{k}{k-1+e^{\epsilon}}
\]

For each threshold $q$, let $F^c(q) = p\bigl(S(x,y)\le q\bigr)$, $F^n(q) = p\bigl(S(x,\tilde{y})\le q\bigr)$,  $F^r(q) = p\bigl(S(x,u)\le q\bigr), \text{ $u \sim Unif(1,\dots,k)$}$ coverage on true labels, noisy labels, and uniform respectively. 
    One can see that 
    \begin{equation}
    F^n(q) \;=\; (1-\beta)\,F^c(q)\;+\;\beta\,F^r(q).
    \label{fnq}
    \end{equation}
Given a calibration set $\{(x_i,\tilde{y}_i)\}$, we can \emph{estimate} $F^n(q)$ and $F^r(q)$ by
    \[
    \hat{F}^n(q)
    = \frac{1}{n}\sum_{i=1}^n {1}\bigl(S(x_i,\tilde{y}_i)\le q\bigr); \,\,\,\hat{F}^r(q)
    = \frac{1}{n}\sum_{i=1}^n \frac{|C_q(x_i)|}{k},
    \]

Thus, rearranging Eq. (\ref{fnq}) we obtain:
\[
\hat{F}^c(q)
= \frac{\hat{F}^n(q) - \beta\hat{F}^r(q)}{\,1-\beta}.
\]

\begin{algorithm}
\caption{LDP-CP-L - Aggregator-side Procedure}
\label{alg:ldp_cp_labels_agg}
\begin{algorithmic}[1]
\State \textbf{Input:} Noisy calibration set $\{(x_i, \tilde{y}_i)\}_{i=1}^{n}$, privacy parameter $\epsilon$, number of labels $k$, target coverage $1-\alpha$
\State \textbf{Output:} Threshold $q$ ensuring private coverage $1-\alpha-\Delta$ on true labels

\State Collect the noisy calibration set $\{(x_i, \tilde{y}_i)\}_{i=1}^{n}$
\State Compute $\beta = \frac{k}{k-1+e^{\epsilon}}$
\State Estimate $\hat{F}^n(q)$ and $\hat{F}^r(q)$ for candidate thresholds $q$:
\[ \hat{F}^n(q) = \frac{1}{n} \sum_{i=1}^{n} \mathbf{1}(S(x_i, \tilde{y}_i) \leq q); \quad
\hat{F}^r(q) = \frac{1}{n} \sum_{i=1}^{n} \frac{|C_q(x_i)|}{k}
\]

\State Compute $\hat{F}^c(q)$:
\[ \hat{F}^c(q) = \frac{\hat{F}^n(q) - \beta \cdot \hat{F}^r(q)}{1 - \beta} \]
\State Initialize $s_{low}=0, s_{high}=1$
\For {$j=1,...,T$}
    \State Set $q^{(j)} = \frac{s_{low}+s_{high}}{2}$
    \State Obtain $Z^{(j)} = \hat{F}^c(q^{(j)})$
    \State \textbf{if} $Z^{(j)} \gt (1-\alpha) + \frac{\Delta}{2}$ \textbf{then} $s_{high}=q^{(j)}$
    \State \textbf{else if} $Z^{(j)} \lt (1-\alpha) - \frac{\Delta}{2}$ \textbf{then} {$s_{low}=q^{(j)}$}
    \State \textbf{else} break
\EndFor
\State Return the threshold $q^{(j)}$
\end{algorithmic}
\label{alg:ldpcpl-agg}
\end{algorithm}

Hence, to find a threshold $q$ that yields coverage $1-\alpha$ on the \emph{true} labels in a private manner, we solve $\hat{F}^c(q)=1-\alpha$. In practice, we do a binary search over candidate thresholds which continues until either the estimate $\hat{Z}^{(j)}$ satisfies $|\hat{Z}^{(j)} - (1-\alpha)| \leq \Delta$ or the interval length $s_{high} - s_{low}$ becomes smaller than a predefined threshold $\tau$. This is exactly the ``noise-aware'' threshold that corrects for the $k$-RR noise mechanism. User's side and aggregator's side procedures depicted in Algorithm boxes \ref{alg:ldpcpl-user} and \ref{alg:ldpcpl-agg} respectively.

In Chapter \ref{ch:nrcp_nacp} we considered this exact scenario and referred to this procedure as \textbf{NACP} (Noise-Aware Conformal Prediction). However, while we were motivated by problems related to a general noisy channel (e.g. experts' mistakes/disagreements as to the true label), here we use the noisy channel as a privacy protection for the labels in the calibration set. As it turns out, using $k$-RR falls neatly into their paradigm and in turn yields Theorem~\ref{theorem:main_ldpcpl}.
Note that the $k$-RR mechanism is implemented here, instead of more advanced LDP methods, such as RAPPOR, since it aligns well with the NACP framework.

\section{Local-DP Conformal Prediction on Scores}
\label{subsec:ldp-cp-scores}

In some scenarios, users may be able to compute the \emph{full conformity score} locally. Concretely, the aggregator (or model provider) sends the neural network’s parameters or logits to each user, who then computes the conformity score \(S(x_i, y_i)\) for the ground-truth label \(y_i\) on their own. 
This approach leverages the ability to estimate quantiles of a distribution using noisy scores that are locally privatized by the users.
We employ the LDP-binary search algorithm described in  \cite{gaboardi2019locallyprivatemeanestimation} to estimate the $(1-\alpha)$-quantile of the scores.
This method is incorporated into the conformal prediction framework by estimating the $(1-\alpha)$-quantile of the conformity scores derived locally by the users. Once the quantile is estimated, it is used as the threshold for constructing prediction sets. The use of LDP ensures that the process is privacy-preserving, whereas the quantile estimation guarantees accurate calibration of prediction intervals.
By combining the strengths of binary search and randomized response, this method offers a robust approach to privacy-preserving conformal prediction that is both practical and theoretically sound.
The settings are depicted in Figure \ref{fig:main_pipeline_s}.

\textbf{User's side Procedure.} For each user \(i\) with pair \((x_i, y_i)\):
\begin{enumerate}
    \item Compute the conformity score \( S_i = S(x_i, y_i) \) locally.
    \item Compare \( S_i \) with a threshold \( q(j) \) provided by the aggregator. (Note that each user has a single interaction with the aggregator -- see details in Aggregator's procedure).
    \item Return a binary response using randomized response (RR), ensuring \( \varepsilon \)-LDP.
\end{enumerate}

\begin{figure*}[t!]
    \centering
    \includegraphics[trim=0cm 0cm 0cm 0cm, width=0.99\linewidth]{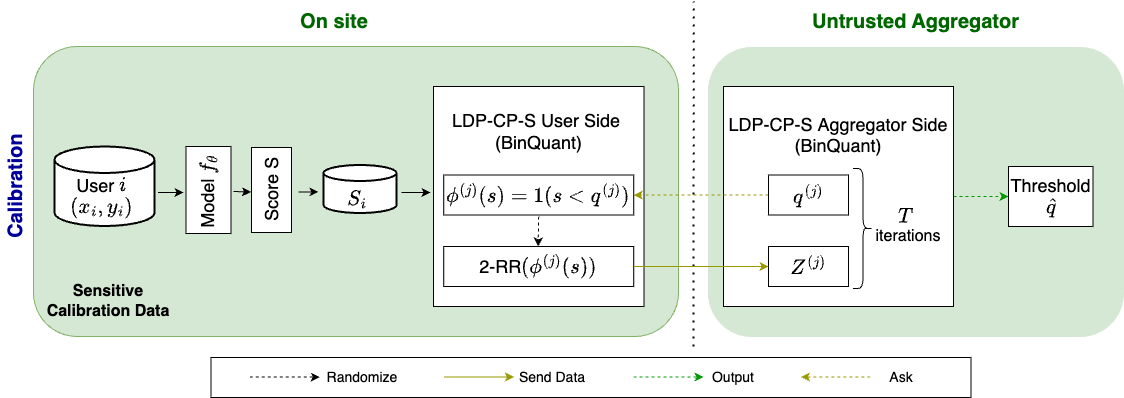}
    \caption{Local Differential Private Conformal Prediction (LDP-CP-S) Pipeline (Best viewed in color).}
    \label{fig:main_pipeline_s}
\end{figure*}

\textbf{Aggregator Procedure.} The aggregator collects the binary responses from users over a binary search procedure. The process starts by defining an initial range that is guaranteed to contain the desired $(1 - \alpha)$-quantile. The range is repeatedly divided in half through a binary search procedure. At each step $j$, a midpoint $q(j)$ is calculated, and a subset of the users is used to privately estimate how many data points fall below this midpoint. This estimation is done using randomized response, which ensures that the algorithm complies with privacy guarantees. Depending on the results of this estimation, the algorithm updates the range: if too many data points are estimated to fall below the midpoint, the upper boundary is adjusted; if too few, the lower boundary is adjusted. Note that the subsets of users used at each step are disjoint, which assures that each user has at most one interaction with the aggregator. This process continues until either the estimate $\hat{Z}^{(j)}$ satisfies $|\hat{Z}^{(j)} - (1-\alpha)| \leq \Delta$ or the interval length $s_{high} - s_{low}$ becomes smaller than a predefined threshold $\tau$. 
In Theorem~\ref{theorem:main_ldpcps} we give the concrete sample complexity, under which we can estimate $\hat{Z}^{(j)}$ both privately and accurately in all iterations of the binary search.
Once the aggregator finds the desired estimation, denoted \( \hat{q} \), the aggregator can now construct the prediction set as:
\(
   C_{\hat{q}}(x) \;=\; \{\, y \mid S(x,y) \le \hat{q}\}.
\)
User's side and aggregator's side procedures depicted in Algorithm box \ref{alg:ldpcps}.



\begin{algorithm}
\caption{LDP-CP-S - Both sides}
\label{alg:ldp_cp_s}
\begin{algorithmic}[1]
\State \textbf{Input:} Calibration set $\{(x_i, y_i)\}_{i=1}^{n}$, privacy parameter $\epsilon$, number of labels $k$, desired coverage $1-\alpha$, $\Delta$, T
\State \textbf{Output:} Threshold $q$ ensuring private coverage $1-\alpha-\Delta$ on true labels
\State Initialize $j=0, n'=\frac{n}{T}, s_{high}=Q_{max}, s_{low}=Q_{min}$
\For {$j=1,...,T$}
    \State Select users $\mathcal{U}^{(j)}=\{j \cdot n' + 1, j \cdot n' + 2,..., (j + 1) \cdot n'\}$
    \State Set $q^{(j)} = \frac{s_{low}+s_{high}}{2}$
    \State Individual users $i\in \mathcal{U}^{(j)}$: (1) compute their score $s=s(x_i, y_i)$, (2) sets $b_i = 1(s \lt q^{(j)})$, (3) sends the aggregator $RR(b_i)$
    \State Obtain $Z^{(j)} = \frac{e^\epsilon + 1}{e^\epsilon - 1}\cdot \frac{1}{n}\sum_{i\in \mathcal{U}^{(j)}} RR(b_i) - \frac{1}{e^\epsilon - 1}$
    \State \textbf{if} $Z^{(j)} \gt (1-\alpha) + \frac{\Delta}{2}$ \textbf{then} $s_{high}=q^{(j)}$
    \State \textbf{else if} $Z^{(j)} \lt (1-\alpha) - \frac{\Delta}{2}$ \textbf{then} {$s_{low}=q^{(j)}$}
    \State \textbf{else} break
\EndFor
\State Return $q^{(j)}$
\end{algorithmic}
\label{alg:ldpcps}
\end{algorithm}

\section{Theoretical Guarantees}
We now present our main theoretical results.
\begin{theorem}[\textbf{LDP-CP-L}]
\label{theorem:main_ldpcpl}
    Fix $\alpha, \delta, \Delta>0$. There exists an $\epsilon$-local differentially private algorithm that draws $n = O\left(\frac{\log(\nicefrac{1}{\delta})}{\Delta^2h^2}\right)$ exchangeable samples from any admissible distribution $\mathcal{D}$, where $h = \frac{1-\beta}{1+\beta}$, and $\beta = \frac{k}{k-1+e^{\epsilon}}$, and, within at most $T = \lceil\log(\nicefrac{1}{\tau})\rceil$ iterations, produces an estimate $\hat{q}$ that satisfies $\Pr\bigl(y\in C_{\hat{q}}(x)\bigr) \;\ge\; 1-\alpha - \Delta$ with probability at least $1 - \delta$, where $1-\alpha$ is the desired coverage and $\tau$ is an a-priori bound on the length of an interval that can hold $\Delta$-probability mass.
\end{theorem}

\begin{proof}
The LDP-CP mechanism applies k-RR to the input data, ensuring $\epsilon$-local differential privacy \citep{kairouz2016discrete}.
Additionally, using the post-processing property of differential privacy \cite{dwork2006differential}, it is safe to perform arbitrary computations on the output of a differentially private mechanism - which maintains the privacy guarantees of the mechanism. Therefore, since $\text{k-RR}(\{x_i,y_i\}_{i=1}^{n_{cal}})$ satisfies $\epsilon$-local-differential privacy, and because NACP is a deterministic or randomized post-processing function, it follows that NACP(k-RR$(\{x_i,y_i\}_{i=1}^{n_{cal}})$) satisfies $\epsilon$-local-differential privacy. 

The remainder of the proof focuses on the conformal prediction coverage guarantee bound. Given our k-RR($\epsilon$) $\epsilon$-LDP mechanism, we derive $\beta = \frac{k}{k-1+e^{\epsilon}}$. Substituting $\beta$, and $n_{cal}$ into \ref{theorem3} we obtain $\Delta(n, \beta, \delta)$ such that $\Pr\bigl(y\in C_{\tilde q}(x)\bigr) \;\ge\; 1-\alpha - \Delta.$
\end{proof}

\begin{theorem}[\textbf{LDP-CP-S}]
Fix $\alpha, \delta, \Delta>0$. There exists an $\epsilon$-local differentially private algorithm that draws $n = O\left(\frac{T}{\Delta^2}(\frac{e^\epsilon + 1}{e^\epsilon - 1})^2\log(\nicefrac{T}{\delta}))\right)$ exchangeable samples from any admissible distribution $\mathcal{D}$ and, within at most $T = \lceil\log(\nicefrac{1}{\tau})\rceil$ iterations, produces an estimate $\hat{q}$ that satisfies $\Pr\bigl(y\in C_{\hat{q}}(x)\bigr) \;\ge\; 1-\alpha - \Delta$ with probability at least $1 - \delta$, where $1-\alpha$ is the desired coverage and $\tau$ is an a-priori bound on the length of an interval that can hold $\Delta$-probability mass.
\label{theorem:main_ldpcps}
\end{theorem}

\begin{proof}
    The privacy proof of a $\epsilon$-LDP quantile binary search algorithm can be found in \citet{gaboardi2019locallyprivatemeanestimation}. Users start by computing scores locally and then a local differentially private quantile binary search algorithm is taken place \citep{gaboardi2019locallyprivatemeanestimation}.
    $\epsilon$-LDP follows immediately from the fact that the only time we access the data is via randomized response. The output of the algorithm is the $(1-\alpha)$'th quantile that one would obtain by applying conformal prediction to the clean data, yet it is recovered solely from the privatized (noisy) scores.
\end{proof}

\section{Practical Considerations}

\subsection{LDP-CP-L vs. LDP-CP-S}
\label{sec:s_vs_l}

The proposed methods, LDP-CP-L and LDP-CP-S, offer distinct approaches for achieving local differential privacy in conformal prediction, each is tailored to specific privacy setups and has different coverage guarantees. An in-depth understanding of their trade-offs is essential to determine their suitability for various scenarios.

LDP-CP-L focuses on achieving local differential privacy by perturbing the labels ($y$) while exposing the features ($x$) to the untrusted aggregator. This approach ensures privacy for the labels, which are typically considered more sensitive in many applications. This mechanism is particularly advantageous in scenarios where users are resource-constrained, since they only need to perturb their labels locally before sending $(x, \tilde{y})$ to the aggregator. This design also keeps the model parameters and scoring functions secret from the users, because all computations related to the score are performed centrally. Consequently, LDP-CP-L is well-suited for real-world calibration datasets of moderate size, such as those containing several thousand records, where the additional noise introduced by the mechanism remains manageable. However, a notable drawback of LDP-CP-L is its sensitivity to the number of classes $k$ in the dataset. As $k$ increases, the calibration error ($\Delta$) grows, potentially compromising coverage guarantees for datasets with a large number of classes. Additionally, while the exposure of $x$ provides practical utility by allowing centralized score computation, it introduces privacy concerns. This drawback can be partially mitigated by employing the shuffle model of differential privacy, which adds an extra layer of anonymity to the users' data (See further discussion in Section~\ref{sec:on_exp_x_main}). Furthermore, privatizing $x$ would result in a substantial insertion of noise, thereby leading to a significant degradation in the accuracy of the algorithms.

By contrast, LDP-CP-S achieves privacy at the score level by having users compute scores locally and perturb their responses before submitting them to the aggregator. This design ensures that both $x$ and $y$ remain private and are never exposed to the aggregator, making LDP-CP-S particularly suitable for scenarios where feature privacy is paramount. Unlike LDP-CP-L, the performance of LDP-CP-S is independent of the number of classes $k$. However, this approach imposes additional computational requirements on the users, who must perform local computations involving the model and scoring function. This requirement necessitates sharing the model with users, which might raise concerns about intellectual property or model misuse. Furthermore, LDP-CP-S requires a larger calibration dataset ($n$) to achieve a sufficiently small calibration error, potentially limiting its applicability in scenarios with limited data availability.

In summary, the choice between LDP-CP-L and LDP-CP-S depends on the specific privacy and computational constraints of the application. LDP-CP-L is better suited for scenarios where label privacy is the primary concern, datasets are of moderate size (and larger). Its design minimizes user-side computations and protects the model from exposure. Conversely, LDP-CP-S is preferable when both feature and label privacy are critical, and when users have the computational resources to perform local scoring. Its robustness to the number of classes makes it a strong candidate for applications with a large class set, provided a sufficiently large calibration dataset is available. By carefully considering these trade-offs, practitioners can select the most appropriate method for their privacy-preserving conformal prediction tasks. In the experiment section, we report a numerical comparison of methods accuracy ($\Delta_L, \Delta_S$) as a function of the calibration set size $n$ and the number of classes $k$ (Figure \ref{fig:s_vs_l}).

\subsection{The Shuffle Model of Differential Privacy}
\label{sec:on_exp_x_main}
The shuffle model of differential privacy enhances privacy guarantees by introducing an additional layer of anonymization between users and the data aggregator. In this model, each user applies a local randomizer to their data and then sends the output to a secure shuffler, which permutes the messages uniformly at random before forwarding them to the aggregator. This anonymization mechanism breaks the association between individual users and their messages, thereby amplifying privacy guarantees beyond those attainable in the purely local model.

A notable benefit of the shuffle model is its capacity for \emph{privacy amplification}~\citep{cheu2022differentialprivacyshufflemodel}: if each user applies an $\varepsilon$-LDP mechanism before sending their message, the effective privacy loss can be reduced to approximately $\epsilon^{\text{eff}}=\nicefrac{\varepsilon}{\sqrt{n}}$ in the shuffled output. This amplification enables stronger privacy guarantees with the same local noise or, conversely, allows for reduced noise to achieve a given privacy target—thereby improving utility in downstream tasks. This is particularly beneficial in regimes where moderately high local privacy levels (e.g., $\varepsilon > 1$) would otherwise impose significant performance degradation.

Given these advantages, our approach incorporates the shuffle model to improve the privacy-utility trade-off of our mechanisms. While earlier versions of our framework considered exclusively the local model, we found that adopting the shuffle model allows us to retain decentralization and local control while achieving significantly improved accuracy through amplification. Importantly, this integration does not alter the algorithmic structure of our mechanisms but rather augments their privacy analysis and performance guarantees under realistic assumptions of an honest-but-curious shuffler. As such, the shuffle model serves not only as a technical enhancement but also as a practical enabler of more effective private learning in our setting.

\section{Experiments}
\label{sec:exps}

In this section, we evaluate the capabilities of our LDP-CP algorithms on various medical and scenery imaging datasets, and address the utility-coverage tradeoff.

 \textbf{Compared Methods.} Our method takes an existing conformity score $S$ and computes a threshold $q$ that considers the injected noise level.  We implemented two popular conformal prediction scores, namely APS
\citep{romano2020classification} and HPS \citep{vovk2005conformal}. For each score $S$, we compared the following CP  methods: (1) Not-Private-CP  with coverage guarantee $1-\alpha$ -  using a validation set with clean labels (2) LDP-CP-\{S,L\} and (3) LDP-CP-\{S,L\}* with coverage guarantee $1-\alpha + \Delta$. We share our code for  reproducibility\footnote{\url{https://anonymous.4open.science/r/LDP-CP}}.

\textbf{Datasets.}
We present results on several publicly available medical imaging classification datasets \citep{medmnistv2}. 
 \textbf{TissuMNIST} \citep{medmnistv1,medmnistv2}:
This dataset contains 236,386 human kidney cortex cells,  organized into 8 categories. Each gray-scale
image is $32 \times 32 \times 7$ pixels. The  2D projections were obtained by taking the maximum pixel value along the axial-axis of each pixel, and were resized into $28 \times 28$ gray-scale images \citep{woloshuk2021situ}.
 \textbf{OrganSMNIST} \citep{medmnistv2}: This dataset contains 25,221 images of abdominal CT in eleven classes. The images are  $28 \times 28$ in size. Here, we used a train/validation/test split of 13,940/2,452/8,829 images. \textbf{OrganAMNIST} and  \textbf{OrganCMNIST} are similar datasets, the differences of Organ\{A,C,S\}MNIST are the views and dataset size. Lastly, \textbf{OCTMNIST} Retina OCT images dataset.
 
\begin{figure*}[ht!]
    \centering
    \includegraphics[trim=0cm 0cm 0cm 0cm, width=1.02\linewidth]{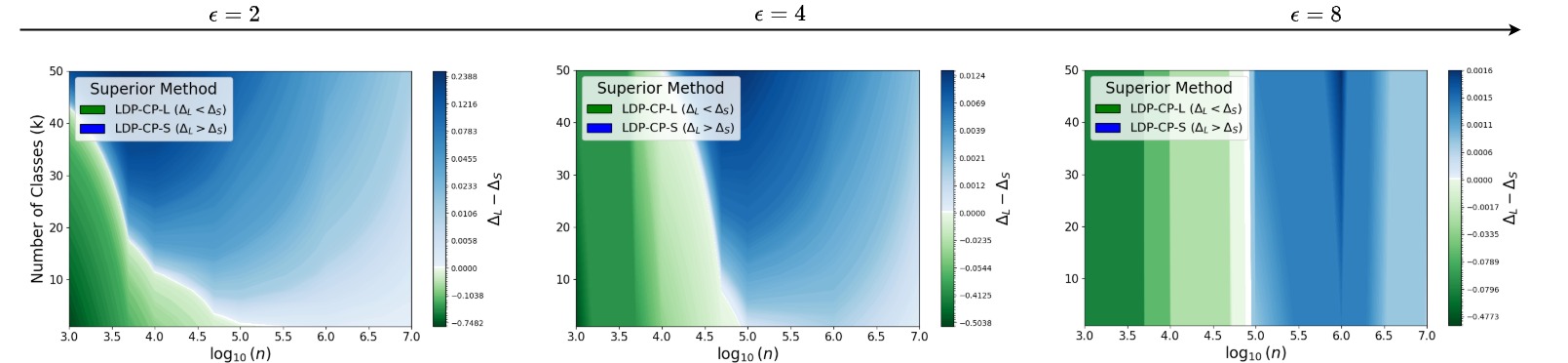}
    \caption{Comparison of $\Delta_l$ and $\Delta_s$ as a function of the number of classes $k$ and dataset size $n$, for $\epsilon=2,4,8$.}
    \label{fig:s_vs_l}
\end{figure*}

\begin{figure}[H]
    \centering
    \includegraphics[trim=0cm 0.5cm 0cm 0cm, width=0.95\linewidth]{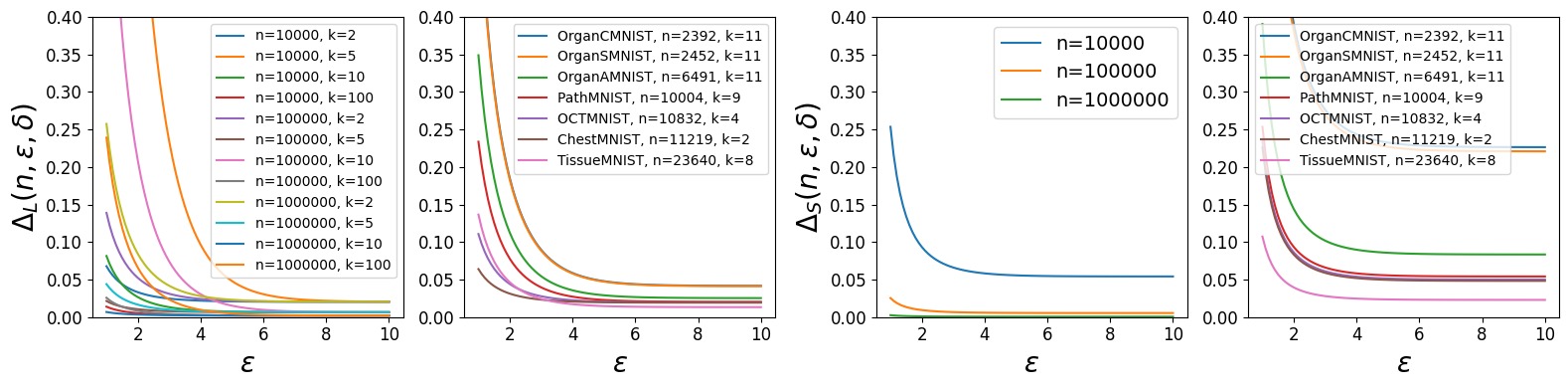}
    \quad \quad  \quad \quad \quad  \quad \quad  \quad \quad \quad \quad \quad \quad \quad \quad \quad \quad \quad \quad \quad  \quad \quad  \quad   \quad (a) LDP-CP-L - $\Delta_L$ \quad \quad \quad  \quad \quad \quad \quad \quad \quad \quad \quad \quad (b) LDP-CP-S - $\Delta_S$ \\
    \caption{CP correction terms $\Delta_L,\Delta_S$ as a function of $\epsilon$ privacy parameter across different dataset configurations of $n$ and $k$ \textbf{without the shuffle model}.}
    \label{fig:tradeoff-general-l-and-s}
    \includegraphics[trim=0cm 0.5cm 0cm 0cm, width=0.95\linewidth]{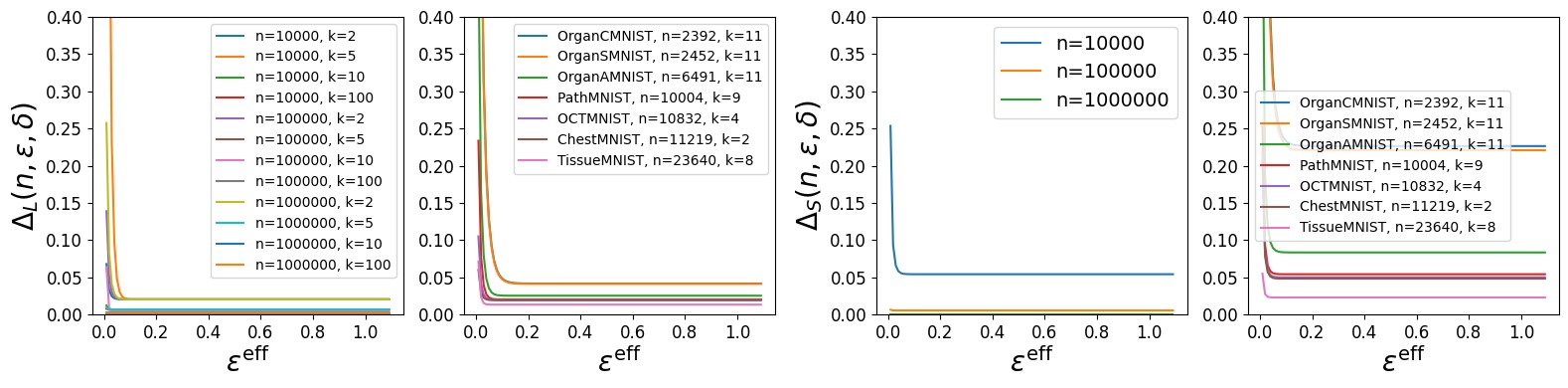}
    \quad \quad  \quad \quad \quad  \quad \quad  \quad \quad \quad \quad \quad \quad \quad \quad \quad \quad \quad \quad \quad  \quad \quad  \quad   \quad (a) LDP-CP-L - $\Delta_L$ \quad \quad \quad  \quad \quad \quad \quad \quad \quad \quad \quad \quad (b) LDP-CP-S - $\Delta_S$ \\
    \caption{CP correction terms $\Delta_L,\Delta_S$ as a function of $\epsilon^{\text{eff}}$ privacy parameter across different dataset configurations of $n$ and $k$ \textbf{with the shuffle model}.}
    \label{fig:tradeoff-general-l-and-s-shuffle}
\end{figure}

\textbf{Utility-Coverage Tradeoff}. 
\label{exp:utility_comp}
Theorems \ref{theorem:main_ldpcpl} and \ref{theorem:main_ldpcps} state that LDP-CP-L and LDP-CP-S are $\epsilon$-LDP with a conformal prediction coverage guarantee correction term of $\Delta_L(n, \beta(\epsilon,k), \delta)$ and $ \Delta_S(n,\epsilon,\delta)$, which depend on the number of samples n, the privacy $\epsilon$, $\delta$, and $\Delta_L$  also depends on the number of classes $k$. Figure \ref{fig:tradeoff-general-l-and-s} explores the trade-offs of $\epsilon,n,k$ on the conformal prediction correction terms $\Delta_L,\Delta_S$, on different configurations and various medical image datasets \citep{medmnistv2} respectively. The results show both the finite-sample practicality and theoretical asymptotic performance of LDP-CP on existing medical datasets that are medium in size and simulated scenarios with growing $n$. For the majority of the evaluated datasets with $\epsilon=3$, LDP-CP-L maintains correction terms that become negligible for $\alpha=0.1$ and higher ($1-\alpha$ is the coverage). $\Delta_L,\Delta_S \xrightarrow{} 0$ as $n \xrightarrow{} \infty$. While as the number of samples grows $\Delta_S$ goes to 0, for the evaluated medical datasets (with $n$ fixed) $\Delta_S>\Delta_L$, i.e. being inferior to $\Delta_L$, and is therefore more applicable to datasets with larger calibration sets.
Figure ~\ref{fig:tradeoff-general-l-and-s-shuffle} covers the same experiment setup, this time when the shuffle model is deployed and presents the correction terms as a function of the effective privacy. The effective privacy $\epsilon^{\text{eff}}$ ranges between 0 and 0.2, which is much more practical then the $\epsilon \geq 3$ acquired without the shuffle model.

Figure \ref{fig:s_vs_l} compares $\Delta_{S}$ and $\Delta_{L}$ as a function of n, k, and $\epsilon$. Results show that roughly speaking for $n \geq 10^5 \xrightarrow{} \Delta_S \leq \Delta_L$. In addition, as the number of classes $k$ increases (particularly for large values of $k$, e.g., 100 or 1000), LDP-CP-S dominates, whereas LDP-CP-L exhibits deteriorating performance.

\textbf{Conformal Prediction Results}. 
Table \ref{tab:ldpcp_full_results} reports the mean size and coverage when applying conformal prediction in a non $\epsilon$-LDP setting (Not-Private-CP), and when satisfying the local differentially private property using LDP-CP-\{S,L\}. The LDP-CP-\{S-L\} method has two variations where the first consists of calibration using $1-\alpha$ and getting $p(y\in C_{q}(x)) \;\ge\; 1-\alpha - \Delta$ (Theorems \ref{theorem:main_ldpcpl},\ref{theorem:main_ldpcps}).
 The second variation needs to satisfy $p(y\in C_{q}(x)) \;\ge\; 1-\alpha$ and therefore uses $1-\alpha+\Delta$ in the calibration phase. For the experiment in Table \ref{tab:ldpcp_full_results} we used $\epsilon=4$ and $\alpha=0.1$ over 100 different data splits (seeds). Table \ref{tab:ldpcp_full_results} also provides the effective privacy $\epsilon^{\text{eff}}$ per dataset to show increased practicality when the shuffle model is incorporated.
\begin{table}[h!]
\centering
\scalebox{.85}{
\begin{tabular}{l|l|cccc}
\toprule
\textbf{Dataset} & \textbf{Method} & \multicolumn{2}{c}{\textbf{HPS}} & \multicolumn{2}{c}{\textbf{APS}} \\
\cmidrule(lr){3-4} \cmidrule(lr){5-6}
 & & size $\downarrow$ & coverage (\%) & size $\downarrow$ & coverage (\%) \\
\midrule
           & Not-Private-CP & 2.57 $\pm$ 0.03 & 90.06 $\pm$ 0.99 & 2.61 $\pm$ 0.03 & 90.06 $\pm$ 0.97 \\  
           & LDP-CP-L & 2.56 $\pm$ 0.04 & 89.99 $\pm$ 1.01 & 2.61 $\pm$ 0.03 & 90.02 $\pm$ 1.01 \\  
OCTMNIST  & LDP-CP-S & 2.58 $\pm$ 0.04 & 90.21 $\pm$ 0.63 & 2.67 $\pm$ 0.08 & 90.84 $\pm$ 1.45 \\  
 ($\epsilon^{\text{eff}}=0.038$)          & LDP-CP-L*& 2.76 $\pm$ 0.04 & 92.22 $\pm$ 0.92 &  2.79 $\pm$ 0.03 & 92.28 $\pm$ 0.87 \\ 
           & LDP-CP-S*& 2.97 $\pm$ 0.07 & 94.38 $\pm$ 0.81 &  2.99 $\pm$ 0.06 & 94.35 $\pm$ 0.70 \\ \hline 
           & Not-Private-CP & 5.55 $\pm$ 0.02 & 90.00 $\pm$ 0.24 & 5.58 $\pm$ 0.02 & 89.96 $\pm$ 0.24 \\  
           & LDP-CP-L & 5.54 $\pm$ 0.02 & 89.97 $\pm$ 0.29 & 5.58 $\pm$ 0.02 & 89.97 $\pm$ 0.27 \\  
TissueMNIST  & LDP-CP-S & 6.12 $\pm$ 0.01 & 95.35 $\pm$ 0.07 & 5.61 $\pm$ 0.09 & 90.32 $\pm$ 0.91 \\  
($\epsilon^{\text{eff}}=0.026$)           & LDP-CP-L*& 5.71 $\pm$ 0.02 & 91.68 $\pm$ 0.27 &  5.76 $\pm$ 0.02 & 91.70 $\pm$ 0.25 \\ 
           & LDP-CP-S*& 6.12 $\pm$ 0.01 & 95.35 $\pm$ 0.07 &  5.83 $\pm$ 0.05 & 92.32 $\pm$ 0.45 \\ \hline 
           & Not-Private-CP & 1.93 $\pm$ 0.05 & 90.09 $\pm$ 0.66 & 2.35 $\pm$ 0.05 & 90.10 $\pm$ 0.55 \\  
           & LDP-CP-L & 1.88 $\pm$ 0.07 & 89.49 $\pm$ 0.93 & 2.30 $\pm$ 0.09 & 89.63 $\pm$ 0.91 \\  
OrganSMNIST  & LDP-CP-S & 1.61 $\pm$ 0.21 & 84.81 $\pm$ 4.48 & 2.09 $\pm$ 0.07 & 87.38 $\pm$ 1.37 \\  
($\epsilon^{\text{eff}}=0.080$)           & LDP-CP-L*& 2.77 $\pm$ 0.22 & 95.35 $\pm$ 0.74 &  3.35 $\pm$ 0.22 & 95.45 $\pm$ 0.74 \\ 
           & LDP-CP-S*& 3.90 $\pm$ 0.03 & 97.75 $\pm$ 0.06 &  4.75 $\pm$ 0.03 & 98.40 $\pm$ 0.07 \\ \hline 
           & Not-Private-CP & 1.19 $\pm$ 0.02 & 89.99 $\pm$ 0.46 & 1.61 $\pm$ 0.02 & 90.02 $\pm$ 0.38 \\  
           & LDP-CP-L & 1.31 $\pm$ 0.00 & 92.17 $\pm$ 0.09 & 1.60 $\pm$ 0.03 & 89.94 $\pm$ 0.54 \\  
OrganAMNIST  & LDP-CP-S & 1.15 $\pm$ 0.05 & 88.89 $\pm$ 0.96 & 1.67 $\pm$ 0.21 & 90.35 $\pm$ 3.10 \\  
($\epsilon^{\text{eff}}=0.049$)           & LDP-CP-L*& 1.43 $\pm$ 0.03 & 93.62 $\pm$ 0.34 &  1.89 $\pm$ 0.05 & 93.51 $\pm$ 0.51 \\ 
           & LDP-CP-S*& 1.88 $\pm$ 0.19 & 96.52 $\pm$ 0.82 &  2.44 $\pm$ 0.19 & 96.80 $\pm$ 0.60 \\ \hline 
           & Not-Private-CP & 1.18 $\pm$ 0.03 & 89.99 $\pm$ 0.71 & 1.56 $\pm$ 0.03 & 90.02 $\pm$ 0.65 \\  
           & LDP-CP-L & 1.30 $\pm$ 0.00 & 91.96 $\pm$ 0.13 & 1.52 $\pm$ 0.04 & 89.36 $\pm$ 0.91 \\  
OrganCMNIST  & LDP-CP-S & 0.95 $\pm$ 0.08 & 82.44 $\pm$ 2.74 & 1.39 $\pm$ 0.07 & 86.59 $\pm$ 2.44 \\  
($\epsilon^{\text{eff}}=0.081$)           & LDP-CP-L*& 1.63 $\pm$ 0.10 & 95.21 $\pm$ 0.74 &  2.05 $\pm$ 0.13 & 95.21 $\pm$ 0.80 \\ 
           & LDP-CP-S*& 2.47 $\pm$ 0.02 & 98.18 $\pm$ 0.06 &  2.90 $\pm$ 0.04 & 98.15 $\pm$ 0.10 \\ \hline 

\bottomrule
\end{tabular}
}
\caption{Calibration results for HPS and APS conformal scores across various datasets, using $\epsilon=4$, $\epsilon^{\text{eff}}=\frac{\epsilon}{\sqrt{n}}$, and $\alpha=0.1$ on 100 different seeds.}
\label{tab:ldpcp_full_results}
\end{table}

\begin{figure}[H]
    \centering
    \includegraphics[trim=0.5cm 1cm 0cm 0cm, width=0.3\textwidth]{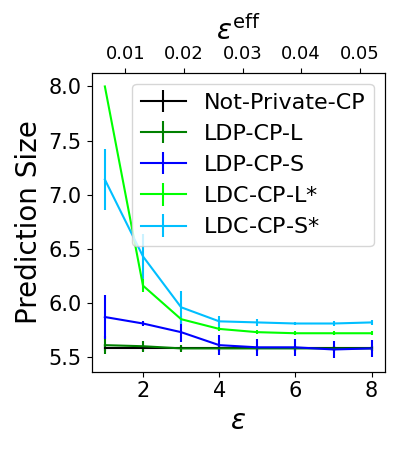} 
    \includegraphics[trim=0.5cm 1cm 0cm 0cm, width=0.3\textwidth]{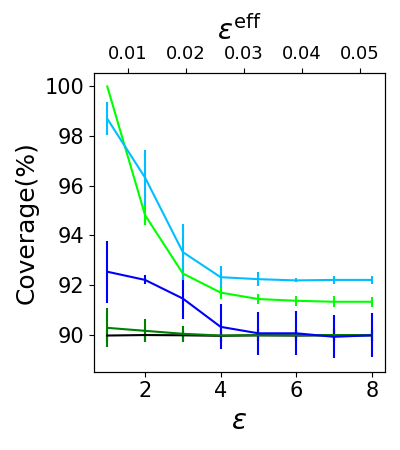}
    \caption{Size of prediction set (left) and coverage (right) as a function of the privacy $\epsilon$ (bottom x-axis) and effective privacy $\epsilon^{\text{eff}}$ (top x-axis). We show the (mean $\pm$ std) on TissueMNIST and APS score.}
    \label{fig:tissuemnist_calib_with_g}
\end{figure}

Next, we experimented with different values of $\epsilon$, namely $\epsilon \in [1,8]$. Recall that with the shuffle model the effective privacy is $\epsilon^{\text{eff}}=\frac{\epsilon}{\sqrt{n}}$. Figure \ref{fig:tissuemnist_calib_with_g} shows the coverage, size of the prediction set, and effective privacy on the TissueMNIST dataset. As expected as $\epsilon$ grows, $\Delta$ decreases and CP results get closer to the Non-Private-CP. In addition, LDP-CP-L performs on a par with the Non-Private-CP for all $\epsilon$'s, and LDP-CP-S for $\epsilon^{\text{eff}}\geq 0.03 $ ($\epsilon\geq4$), showcasing their true applicability.

\chapter{Unsupervised Target Domain Confidence Calibration}
\label{ch:utdc}

In this chapter, we explore the challenge of confidence calibration in Deep Neural Networks (DNNs) when transferring models from a source domain to an unlabeled target domain. While DNNs exhibit high accuracy in classification and detection tasks under well-supervised conditions, real-world applications often involve domain shifts where target domain labels are unavailable. Traditional calibration methods rely on labeled data to align model confidence with true probabilities, but these approaches fall short under unsupervised domain adaptation (UDA) scenarios. We introduce a novel method for directly calibrating model predictions on target domain data, demonstrating its effectiveness in addressing the pitfalls of existing approaches.

\section{Problem Statement}

Deep Neural Networks (DNN) have shown remarkable accuracy in tasks such as classification and detection when sufficient data and supervision are present. In practical applications, it is crucial for models not only to be accurate, but  also to indicate how much confidence users can have in their predictions. DNNs generate confidence scores that can serve as a rough estimate of the likelihood of correct classification, but these scores do not guarantee a match with the actual probabilities \cite{Guo2017}. Neural networks tend to be overconfident in their predictions, despite having higher generalization accuracy, due to the possibility of overfitting on negative log-likelihood loss without affecting classification error \cite{Guo2017,Balaji2017,Hein2019}. A classifier is said to be calibrated with respect to a dataset sampled from a given distribution if its predicted probability of being correct matches its true probability. 
Various methods have been introduced to address the issue of over-confidence. Network calibration can be performed in conjunction with training (see e.g. \cite{mukhoti2020calibrating, muller2019does,mixup}). Post-hoc scaling methods for calibration, such as Platt scaling \cite{Platt1999}, isotonic regression \cite{Zadrozny2002}, and temperature scaling \cite{Guo2017}, are commonly employed. These techniques apply calibration as post-processing, using a hold-out validation set to learn a calibration map that adjusts the model's confidence in its predictions to become better calibrated.

The implementation of deep learning systems on real-world problems  is hindered by the decrease in performance when a network trained on data from one domain is applied to data from a different domain where the distribution of features changes across domains (see e.g. \cite{pmlr-v139-miller21b}). This is known as the domain shift problem. In an Unsupervised Domain Adaptation (UDA) setup we assume the availability of data from the target domain but without annotation. There is a plethora of UDA methods  based on strategies
such as adversarial training methods that aim to align the distributions of the source and target domains \cite{ganin2016domain},  or self-training algorithms  based on computing pseudo labels for the target domain data \cite{zou2019confidence}.

In the following section we tackle the problem of calibrating predicted probabilities when transferring a trained model from a source domain to a target domain without any given labels.

Our major contributions include the following:
\begin{itemize}
\item We show that 
 current UDA calibration methods which are all based on the source domain data, rely on an overly  optimistic estimation of the target accuracy. Thus they can't well handle the domain shift problem.
\item We propose a calibration method that is directly applied to the target domain data, based on a realistic estimation of the accuracy of the adapted model on the target domain.
\item We show that previously proposed UDA calibration methods don't work at all and thus in this study we propose the first effective method for calibrating a network obtained by an unsupervised domain adaptation.  
\end{itemize}
The study described in this chapter was published in \cite{penso2024calibrationuda}.

\section{Calibration on the Target Domain in Unsupervised Domain Adaptation}

\begin{table*}[t]
	\caption{Comparison of calibration methods for unsupervised domain adaptation (UDA).}
	\label{table:utdc_ece_tabx}
     \resizebox{\textwidth}{!}{
	\centering
		\begin{tabular}{l|ccc|ll}
			\toprule 
                Calibration Method &
                Designed for &
                Works without &
                Works on &
                Approach &
                Granularity \\
                &
                domain shift &
                target label &
                target data \\
			\midrule
                Temp. Scaling \cite{Guo2017} & $\times$ & $\times$& $\times$ & -- & Instance level \\
                CPCS\cite{park2020}, TransCal\cite{wang2020} & \checkmark& \checkmark& $\times$ & Importance weight estimation & Instance level\\
                                UTDC (proposed) & \checkmark& \checkmark&\checkmark & Estimates target accuracy & Dataset level \\
			\bottomrule
		\end{tabular}%
      }
\end{table*}

Our method involves calibrating the adapted network directly on the target data.
While applying the network on the target domain data allows us to compute its confidence, we cannot determine its accuracy.
Thus, the challenge is to find a reliable estimate of the network accuracy on the target domain.
Our approach is based on the observation that when calibrating by minimizing the adaECE score, we do not need to know whether each individual prediction is correct. Instead, we only need to determine the mean accuracy for each bin.  Fortunately, there are techniques which
 given a trained network, can  estimate the network accuracy 
on  data samples from a new domain without access to their labels  \cite{deng2021labels, Guillory_2021_ICCV, garg2022leveraging, yu2022predicting}.

We next suggest a simple, intuitive, and very effective method that  calibrates the network directly on the target domain.  
We first compute the overall network accuracy  on the source  data $A_{\textrm{source}}$ and  estimate the network accuracy on the target domain (e.g, using \cite{deng2021labels}). Denote the estimated target accuracy by $\tilde{A}_{\textrm{target}}$. 
Next, we divide the source data into $M$ equal-size bins according to their confidence values and compute the corresponding network accuracy $A_{\textrm{source},m}$  at each bin $m$.
We also divide the target data into $M$ equal-size bins according to their confidence values and estimate the binwise accuracy of the target ${A}_{\textrm{target},m}$ by rescaling the binwise accuracy on the source domain  in the following way:
\begin{equation}
\tilde{A}_{\textrm{target},m} =  A_{\textrm{source},m} \cdot \frac{\tilde{A}_{\textrm{target}}}{A_{\textrm{source}}}, \hspace{1cm} m=1,...,M.
\label{targetaccuracy}
\end{equation}
In the next section, we empirically show  that the accuracy ratio between source and target is indeed
similar across the calibration bins.
The estimated network accuracy on the target data 
$\tilde{A}_{\textrm{target}}$ obtained by an unsupervised adaptation is usually lower than its 
accuracy on the source data $A_{\textrm{source}}$. Thus,  this accuracy rescaling provides a more realistic estimation of the bin-wise network average accuracy on the target data. The accuracy ratio 
$\tilde{A}_{\textrm{target}}/A_{\textrm{source}}$ indicates the size of the domain gap or the difficulty of the adaptation task \cite{Zou_2023_ICCV}.

Let $C_{target,m}$ be the bin-wise network average confidence values computed on the target data. Substituting the estimated accuracy term, based on the source labeled data (\ref{targetaccuracy}) into the adaECE definition,
yields the following adaECE measure for the target domain in a UDA setup:
\begin{equation}
    \textrm{UDA\mbox{-}adaECE} =  \frac{1}{M} \sum_{m=1}^{M} 
    \left| 
\tilde{A}_{\textrm{target},m}- C_{\textrm{target},m} \right|.
\label{ECEadapt}
\end{equation}

\begin{algorithm*}[t]
\caption { Unsupervised Target Domain Calibration (UTDC) }
\begin{algorithmic}[0]
\State   \textbf{input:} A labeled validation set from the source domain, an unlabeled dataset from the target 
domain, and a $k$-class classifier that  was adapted to the target domain.

\vspace{0.1cm}

 \State - Compute the source accuracy $A_{\textrm{source}}$ and 
estimate the target accuracy $\tilde{A}_{\textrm{target}}$ \\ 
\hspace{0.3cm} using a target accuracy estimation technique.

\State - Divide the source points into $M$ equal size sets  based on their  confidence and 
\\ \hspace{0.3cm} compute the binwise mean accuracy: 
$ A_{\textrm{source},m}$, \, $m=1,...,M$.
\State - Divide  the target  points into $M$ equal size sets $B_1,...,B_M$ based on their  confidence.
\For{each candidate value of $T$}
\State - Compute  the binwise  mean confidence on the target:
\[ C_{\textrm{target},m}(T)= 
  \frac{1}{|B_m|} \sum_{x \in B_m} \max_{i=1}^k 
\frac {\exp(z_{x,i}/T)} { \sum_{j=1}^k \exp(z_{x,j}/T)}  \hspace{1cm}  m=1,...,M.\]
\hspace{0.7cm} s.t. $z_{x,1}$,...,$z_{x,k}$ are the logits  computed by the network
that is fed by $x\in B_m$.
\State - Compute the adaECE score as a function of $T$: \hspace{0.3cm}
\[ \mathrm{UDA\mbox{-}adaECE}(T) = \frac{1}{M}  \sum_{m=1}^{M} \left| 
{A}_{\textrm{source},m} \times \frac{\tilde{A}_{\textrm{target}}}{A_{\textrm{source}}}
- C_{\textrm{target},m}(T) \right| \hspace{1.8cm} \]
\EndFor
\State   \textbf{output:}  The optimal temperature:  $\hat{T}= \arg \min_T \mathrm{UDA\mbox{-}adaECE}(T) $
  \end{algorithmic}
  \label{alg:http}
\end{algorithm*}

For each calibration method whose parameters can be found by minimizing the adaECE measure, we can form a UDA variant in which UDA-adaECE is minimized instead of adaECE. 
 Examples of these calibration
methods include Temperature Scaling (TS),
 Vector Scaling, Matrix Scaling \cite{Guo2017},  Mix-n-Match \cite{zhang2020mix}, Weight Scaling \cite{frenkel2022calibration}, 
and others.  

We next demonstrate the UDA calibration principle in the case of TS calibration.  
We can determine the temperature that minimizes the UDA-adaECE measure (\ref{ECEadapt}) by conducting a grid search on the possible values. 

 Given the division of the target data into bins, we can  compute the binwise average confidence after temperature calibration by $T$ on the target  ${C}_{\textrm{target},m}(T)$. We can then define the following temperature-dependent adaECE scores: 
\begin{equation}
\mathrm{UDA\mbox{-}adaECE}(T) =  \frac{1}{M}  \sum_{m=1}^{M} \left| 
\tilde{A}_{\textrm{target},m}
- C_{\textrm{target},m}(T) \right|.
\label{ECEt}
\end{equation}
The optimal temperature is thus obtained by applying a grid search to find $T$ that minimizes UDA-adaECE$(T)$.
The proposed  Unsupervised Target Domain Calibration  (UTDC)   algorithm is summarized in Algorithm Box \ref{alg:http}.

A major component of the UTDC is estimating the target domain accuracy. based on unlabeled target domain data.  We next describe several recently suggested estimation algorithms. 
  Deng et al. \cite{deng2021labels} suggested learning a dataset-level regression problem. 
 The first step is to augment the source domain validation set, denoted by $D_s$, using various visual transformations such as resizing, cropping, horizontal and vertical flipping, Gaussian blurring, and others. We then create  $n$ meta-datasets, denoted as $D_1,...,D_{n}$ (in our implementation we set $n=50$). This process preserves the labels so  we can  compute the model's accuracy on these datasets, denoted by $a_1,...,a_{n}$. 
Each dataset $D_i$ is represented as a Gaussian distribution using its mean vector $\mu_i$ and its diagonal covariance matrix $\Sigma_i$. Let $F_i$ be the Fr\'echet  distance \cite{dowson1982frechet}
between the Gaussian representations of $D_s$ and $D_i$. $F_i$ measures the domain gap between
the original dataset $D_s$ and $D_i$.
Next, a linear regression model is fitted to the dataset $(F_1,a_1),...,(F_n,a_n)$ in the form of  $\hat{a}=w\cdot F+b$.
Finally, the linear regression model is employed to predict the accuracy of the network on the unlabeled data from the target domain. 
Another method is Average Thresholded Confidence (ATC) \cite{garg2022leveraging} which first selects a threshold $t$ whose error in the source domain matches the expected number of points whose confidence is below $t$. Next, ATC predicts the error
on the target domain which is expressed as the fraction of unlabeled points that obtain a confidence value below that threshold $t$. Let $\hat{p}(x) = \max_i (p=i|x)$ be the network confidence and let $\hat{y}$ be the network prediction. A  threshold $t$ is calculated to satisfy the equality 
$E_{x \sim source}1_{\{\hat{p}(x) \leq t\}} = E_{(x,y) \sim source}1_{\{\hat{y} \neq y\}}$.
The estimated target accuracy is the expectation $E_{x \sim target}1_{\{\hat{p}(x) \leq t\}}$.
Finally, the Projection Norm (PN) method \cite{yu2022predicting} uses the model predictions to pseudo-label the test samples and then trains a new model on the pseudo-labels. The discrepancy between the parameters of the new and original models yields the predicted error of the target domain data. %
In Section \ref{s:utdc_expriments} we compare the UTDC's calibration performance when using each of the target accuracy prediction methods described above.

\section{Experiments}
\label{s:utdc_expriments}
In this section, we evaluate the capabilities of our UTDC
 technique to calibrate a network on a target domain 
 after applying a UDA procedure. 

\textbf{Compared methods.}  
 We compared our method  to six baselines:
(1) Uncalibrated - The adapted classifier as is, without
any post-hoc calibration;
(2-4) Source-TS, Source-VS and source-MS - The adapted network was calibrated by either Temperature Scaling (TS), Vector Scaling (VS) or Matrix scaling (MS) \cite{Guo2017} using the labeled validation set of the source domain; 
(5) CPCS \cite{park2020}, and (6) TransCal \cite{wang2020}, importance weighted UDA calibrators.
We also report Oracle results where TS calibration was applied to the labeled data from the target domain (denoted by Target-TS)  and an Oracle version of our approach (denoted  by UTDC*) where we used the exact accuracy of the adapted model on the target data instead of estimating it.

\begin{table*}[t]
        \caption{AdaECE results  on Office-home (with the lowest in bold) on various UDA classification tasks and models with different calibration methods.}
        \label{table:utdc_uda_adaece_tab1}
        \centering
        \resizebox{\textwidth}{!}{
                \begin{tabular}{ll|rrrrrrrrr|r}
                        \toprule 
                        {\small UDA}&
               {\small Method}& 
               {\scriptsize $A \rightarrow R$} & 
               {\scriptsize $A \rightarrow C$} &
               {\scriptsize $A \rightarrow  P$} &
               {\scriptsize $C \rightarrow  R$ }&
               {\scriptsize $C \rightarrow  P$} &
               {\scriptsize $C \rightarrow  A$} &
               {\scriptsize $P \rightarrow  R$ }&
               {\scriptsize $P \rightarrow  C$} &
               {\scriptsize $P \rightarrow  A$} &
               Avg \\
                        \midrule

& \small{Uncalibrated} &22.23 & 42.62 & 30.49 & 25.18 & 28.25 & 33.69 & 20.32 & 40.46 & 38.85 & 31.34\\
& \small{Source-TS} &8.09 & 24.43 & 14.89 & 10.00 & 14.17 & 13.85 & 11.14 & 27.42 & 26.60 & 16.73\\
& \small{Source-VS} &10.54 & 27.54 & 19.51 & 12.12 & 14.65 & 15.78 & 11.27 & 31.55 & 27.46 & 18.94\\
& \small{Source-MS} &28.62 & 47.87 & 35.74 & 31.62 & 31.54 & 40.43 & 23.59 & 43.90 & 40.56 & 35.99\\
& \small{CPCS} &15.84 & 49.78 & 23.42 & 14.02 & 16.60 & 18.45 & {6.31} & 49.21 & 25.62 & 24.36\\
\small{CDAN+E}& \small{TransCal} &6.01 & 27.30 & 9.46 & 16.67 & 16.81 & 21.69 & 19.90 & 41.23 & 39.71 & 22.09\\
 & \small{UTDC} &{4.46} & {9.74} & {7.53} & {8.36} & {5.91} & {8.08} & 10.45 & {7.46} & {9.37} & \textbf{ 7.93}\\
\cline{2-12} 
& \small{UTDC*} &4.30 & 5.93 & 7.41 & 7.85 & 4.62 & 10.16 & 10.76 & 4.55 & 9.54 & 7.24\\
& \small{Target-TS} &3.97 & 5.05 & 7.19 & 4.07 & 4.39 & 7.07 & 2.32 & 4.39 & 8.57 & 5.22\\
 \hline & \small{Uncalibrated} &19.90 & 39.19 & 26.75 & 24.47 & 26.33 & 33.53 & 20.25 & 40.06 & 39.25 & 29.97\\
& \small{Source-TS} &6.90 & 19.80 & 7.93 & 6.54 & 7.01 & 16.01 & 15.68 & 27.87 & 30.97 & 15.41\\
& \small{Source-VS} &10.15 & 25.83 & 15.31 & 12.13 & 10.70 & 17.90 & 14.69 & 32.40 & 31.64 & 18.97\\
& \small{Source-MS} &30.78 & 52.03 & 38.39 & 35.44 & 35.45 & 44.21 & 26.40 & 45.87 & 43.33 & 39.10\\
& \small{CPCS} &13.90 & 50.16 & 21.32 & {3.62} & 7.25 & 34.74 & 25.86 & 22.66 & 27.97 & 23.05\\
\small{DANN+E}& \small{TransCal} &7.21 & 27.42 & 12.36 & 17.81 & 15.43 & 29.93 & 24.64 & 46.61 & 45.83 & 25.25\\
  & \small{UTDC} &{4.14} & {5.86} & {5.47} & 10.28 & {3.89} & {6.67} & {15.33} & {5.70} & {12.65} & \textbf{ 7.78}\\
\cline{2-12}
& \small{UTDC*} &2.68 & 4.70 & 4.37 & 8.55 & 4.00 & 4.53 & 14.60 & 3.97 & 6.16 & 5.95\\
& \small{Target-TS} &2.68 & 2.76 & 3.67 & 2.24 & 3.16 & 2.99 & 1.15 & 1.62 & 4.55 & 2.76\\
 \hline & \small{Uncalibrated} &16.82 & 31.28 & 23.11 & 17.22 & 20.46 & 27.38 & 15.88 & 33.81 & 30.13 & 24.01\\
& \small{Source-TS} &6.33 & 16.41 & 13.22 & 2.83 & 5.00 & 15.82 & 10.91 & 29.09 & 23.61 & 13.69\\
& \small{Source-VS} &10.03 & 25.58 & 15.86 & 8.10 & 8.23 & 15.18 & 11.86 & 33.08 & 27.24 & 17.24\\
& \small{Source-MS} &31.61 & 50.68 & 41.31 & 34.23 & 36.48 & 44.23 & 25.49 & 44.75 & 40.17 & 38.77\\
& \small{CPCS} &8.89 & 33.56 & 19.99 & 25.29 & 9.62 & 12.82 & 16.87 & 27.49 & 45.93 & 22.27\\
\small{DANN}& \small{TransCal} &7.63 & 29.15 & 22.20 & 22.64 & 22.97 & 37.66 & 26.11 & 50.85 & 47.53 & 29.64\\
& \small{UTDC} & { 5.15} & { 4.87} & { 11.24} & { 8.63} & { 5.23} & { 15.08} & { 18.62} & { 12.62} & { 11.23} & \textbf{ 10.30}\\
\cline{2-12} 
& \small{UTDC*} &2.80 & 5.49 & 6.21 & 6.20 & 3.38 & 3.44 & 12.61 & 5.00 & 4.67 & 5.53\\
& \small{Target-TS} &2.45 & 2.38 & 4.65 & 2.08 & 1.73 & 2.16 & 1.22 & 2.35 & 2.92 & 2.44\\

            \bottomrule
        \end{tabular}%
            }
        \end{table*}

{\bf Datasets.} We report experiments on four standard real-world domain adaptation benchmarks, Office-home \cite{venkateswara2017deep}, Office-31 \cite{saenko2010adapting}, VisDa-2017 \cite{peng2017visda}, and DomainNet \cite{peng2019moment}. Office-home includes four domains - Art, Real-World, Clipart and Product, represented as A, R, C, and P in the experiments.  Office-31 contains three domains -  Amazon, Webcam and  DSLR, denoted  A, W, and D. VisDa-2017 is a simulation-to-real dataset for domain adaptation with over 280,000 images across 12 categories. DomainNet has six domains - Clipart, Infograph, Painting, Quickdraw, Real and Sketch, denoted  C, I, P, Q, R, and S.

\begin{table*}[h!]
        \caption{AdaECE results on Office-31  (with the lowest in bold) on various UDA  classification tasks and models with different calibration methods.}
        \label{table:utdc_ece_tab2}
        \centering
         \scalebox{0.8}{ 
               \begin{tabular}{ll|rrrrrr|r}
                        \toprule 
                        {\small UDA Method}&
               {\small Method}& 
               {\scriptsize $A \!\rightarrow\! W$} & 
               {\scriptsize $A \!\rightarrow\! D$} &
               {\scriptsize $W \!\rightarrow\!  A$} &
               {\scriptsize $W \!\rightarrow\!  D$ }&
               {\scriptsize $D \!\rightarrow\!  A$} &
               {\scriptsize $D \!\rightarrow\!  W$} &
               Avg \\
                        \midrule

& \small{Uncalibrated} &11.5 & 10.53 & 29.63 & 1.21 & 29.08 & {\bf 1.33} & 13.88\\
& \small{Source-TS} &6.03 & 7.43 & 33.21 & {\bf 0.86} & 27.25 & 2.12 & 12.82\\
& \small{Source-VS} &{\bf 3.74} & 7.10 & 33.75 & 1.52 & 32.98 & 1.42 & 13.42\\
& \small{Source-MS} &12.15 & 16.72 & 30.76 & { 1.02} & 29.99 & 1.38 & 15.34\\
& \small{CPCS} &9.67 & 12.66 & 33.47 & 1.11 & 28.16 & 2.18 & 14.54\\
\small{CDAN+E}& \small{TransCal} &3.78 & 9.45 & 34.43 & 1.27 & 33.68 & 1.56 & 14.03\\

& \small{UTDC} &4.19 & {\bf 5.18} & {\bf 5.15} & 1.20 & {\bf 5.14} & 2.18 & {\bf 3.84}\\
\cline{2-9} 
& \small{UTDC*} &3.82 & 5.18 & 5.09 & 1.13 & 5.36 & 2.18 & 7.13\\
& \small{Target-TS} &3.44 & 4.67 & 3.32 & 0.75 & 3.20 & 0.89 & 2.71\\
 \hline & \small{Uncalibrated} &13.05 & 13.55 & 28.29 & 0.87 & 27.15 & 1.68 & 14.10\\
& \small{Source-TS} &5.18 & 9.29 & 26.93 & 1.31 & 26.44 & 2.44 & 11.93\\
& \small{Source-VS} & {\bf 4.63} & 8.24 & 36.64 & {\bf 0.87} & 31.35 & 1.55 & 13.88\\
& \small{Source-MS} &18.01 & 14.02 & 31.10 & 1.09 & 28.51 & 1.51 & 15.71\\
& \small{CPCS} &15.58 & 6.81 & 33.97 & 1.99 & 32.69 & {\bf 1.14} & 15.36\\
\small{DANN+E}& \small{TransCal} &7.98 & 5.63 & 34.53 & 1.57 & 31.12 & 1.59 & 13.74\\
& \small{UTDC} &5.25 & {\bf 5.33} & {\bf 8.99} & 1.40 & {\bf 12.26} & 2.41 & {\bf 5.94}\\
\cline{2-9} 
& \small{UTDC*} &4.87 & 6.10 & 6.86 & 1.40 & 6.53 & 2.44 & 4.70\\
& \small{Target-TS} &3.98 & 4.77 & 2.87 & 0.85 & 2.80 & 0.82 & 2.68\\
 \hline & \small{Uncalibrated} &10.66 & 12.59 & 23.03 & 1.77 & 24.43 & 2.93 & 12.57\\
& \small{Source-TS} &3.89 & {\bf 7.17} & 29.58 & {\bf 0.98} & 30.71 & 4.43 & 12.79\\
& \small{Source-VS} &3.88 & 7.64 & 34.50 & 1.44 & 32.31 & 2.84 & 13.77\\
& \small{Source-MS} &21.06 & 24.70 & 28.81 & 1.35 & 28.45 & {\bf 1.30} & 17.61\\
& \small{CPCS} &16.96 & 10.10 & 33.69 & 2.61 & 35.39 & 4.80 & 17.26\\
\small{DANN}& \small{TransCal} &10.36 & 15.62 & 87.02 & 2.31 & 45.79 & 6.00 & 27.85\\
& \small{UTDC} &{\bf 3.71} & 8.70 & {\bf 5.14} & 2.61 & {\bf 9.26} & 5.23 & {\bf 5.78}\\
\cline{2-9} 
& \small{UTDC*} &5.04 & 7.52 & 5.54 & 2.61 & 12.25 & 6.54 & 6.58\\
& \small{Target-TS} &3.53 & 4.12 & 2.79 & 0.97 & 3.19 & 1.94 & 2.76\\

            \bottomrule
        \end{tabular}%
       }
\end{table*}

{\bf Implementation details.}
We followed the experiment setup described in  \cite{wang2020} and used their code to implement CPCS and TransCal baselines.
Following \cite{wang2020}, we implemented three  different UDA techniques; namely, DANN \cite{ganin2016domain}, DANN+E and CDAN+E \cite{long2018conditional}. 
The performance of more recent UDA models (e.g. \cite {liang2021domain, jin2020minimum, cui2020towards}) on the target domain of the evaluated datasets is slightly better  but  is still much worse than the performance on the source domain. 
In most experiments we used the Meta target domain accuracy estimation  \cite{deng2021labels} unless stated otherwise.
We provide a code implementation of our method for 
reproducibility\footnote{\href{https://github.com/cobypenso/unsupervised-target-domain-calibration}{https://github.com/cobypenso/unsupervised-target-domain-calibration}}.

\begin{table}[h!]
        \caption{adaECE results on VisDA Task $S \rightarrow R$, for various calibration methods.}
        \vspace{0.1cm}
	\label{table:utdc_adaece_visda}
   \centering
   \scalebox{0.83}{ 
               \setlength\tabcolsep{7pt}
                \begin{tabular}{l|rrr|r}
                        \toprule 
               {\small Method}& 
               {\footnotesize DANN} & 
               {\footnotesize DANN+E} &
               {\footnotesize CDAN+E} &
               {\footnotesize Avg} \\
                        \midrule
    \small{Uncalibrated} &33.23 & 31.79 & 29.88& 31.63\\		
    \small{Source-TS} &26.54 & 18.66  & 23.38& 34.29 \\		
    \small{Source-VS} &38.22 & 36.96 & 28.48& 34.55\\		
    \small{Source-MS} &41.19 & 38.17 & 30.87& 36.74\\		
    \small{CPCS} &31.86 & 11.08 & 26.88& 23.27\\		
    \small{TransCal} &43.52 & 35.93 & 36.71& 38.72\\		
    \small{UTDC} &\textbf{13.07} & \textbf{6.61}  & \textbf{3.85}& \textbf{7.84} \\	
    \cline{1-5}		
    \small{UTDC*} & 2.31 & 1.94 & 2.57& 2.27\\		
    \small{Target-TS} &2.02 & 1.84 & 2.21 & 2.02\\		
\bottomrule
    \end{tabular}%
 }
\end{table}

\begin{table}[h!]
        \caption{adaECE results on DomainNet for various UDA  classification tasks and models with different calibration methods.}
        \label{table:utdc_ece_tab4}
          \centering
         \scalebox{0.82}{ 
                      \begin{tabular}{ll|rrrrrr|r}
                        \toprule 
                        {\small UDA}&
               {\small Method}& 
               {\scriptsize $S \!\!\rightarrow\!\! R$} & 
               {\scriptsize $S \!\!\rightarrow\!\! P$} &
               {\scriptsize $P \!\!\rightarrow\!\!  R$} &
               {\scriptsize $P \!\!\rightarrow\!\!  S$ }&
               {\scriptsize $R \!\!\rightarrow\!\!  S$} &
               {\scriptsize $R \!\!\rightarrow\!\!  P$} &
               Avg \\
                        \midrule

& \small{Uncalibrated} &14.65 & 18.70 & 18.06 & 22.98 & 19.13 & 13.77 & 17.88\\
& \small{Source-TS} &12.68 & 14.48 & 11.51 & 12.76 & 13.56 & 9.60 & 12.39\\
& \small{Source-VS} &10.70 & 9.56 & 11.49 & 14.94 & 13.35 & 9.31 & 11.56\\
& \small{Source-MS} &22.24 & 25.28 & 23.43 & 30.93 & 22.55 & 18.07 & 23.75\\
\small{CDAN+E}& \small{CPCS} &9.41 & 11.20 & 13.26 & 17.06 & 17.16 & 11.86 & 13.32\\
& \small{TransCal} &12.50 & 20.82 & 16.41 & 28.85 & 36.70 & 28.23 & 23.92\\
& \small{UTDC} & {\bf 6.06} & {\bf 5.17} & {\bf 6.48} & {\bf 4.75} & {\bf 8.85} & {\bf 8.32} & \textbf{6.61}\\
\cline{2-9} 
& \small{UTDC*} &5.07 & 6.78 & 4.86 & 3.56 & 5.19 & 6.86 & 5.38\\
& \small{Target-TS} &1.31 & 1.35 & 2.18 & 1.39 & 1.25 & 1.07 & 1.42\\
 \hline & \small{Uncalibrated} &15.03 & 17.77 & 17.57 & 24.54 & 21.08 & 16.63 & 18.77\\
& \small{Source-TS} &10.12 & 12.20 & 10.31 & 11.75 & 11.76 & 10.69 & 11.14\\
& \small{Source-VS} &9.71 & 14.25 & 11.85 & 19.42 & 16.88 & 12.15 & 14.04\\
& \small{Source-MS} &23.68 & 28.77 & 24.18 & 35.03 & 24.94 & 20.91 & 26.25\\
\small{DANN+E}& \small{CPCS} &13.20 & 6.41 & 12.51 & 12.81 & {\bf  7.73} & {\bf 10.95} & 10.60\\
& \small{TransCal} &14.56 & 19.85 & 16.14 & 29.19 & 34.98 & 28.96 & 23.95\\
& \small{UTDC} & {\bf 6.39} & {\bf 6.07}  & {\bf 6.54} & {\bf 6.84} & { 11.24} & { 11.94} & \textbf{8.17}\\
\cline{2-9} 
& \small{UTDC*} &3.97 & 5.72 & 5.23 & 6.64 & 6.73 & 8.32 & 6.10\\
& \small{Target-TS} &1.24 & 1.19 & 1.60 & 1.03 & 1.10 & 0.84 & 1.17\\
 \hline & \small{Uncalibrated} &10.98 & 13.52 & 12.65 & 18.04 & 15.42 & 10.96 & 13.59\\
& \small{Source-TS} &7.33 & 8.63 & 9.50 & 10.11 & 10.99 & 9.15 & 9.29\\
& \small{Source-VS} &8.92 & 14.43 & 11.21 & 16.90 & 15.86 & 10.86 & 13.03\\
& \small{Source-MS} &22.51 & 27.48 & 21.97 & 31.46 & 24.53 & 19.72 & 24.61\\
\small{DANN}& \small{CPCS} &7.02 & 7.37 & 14.60 & 15.83 & 15.42 & 8.88 & 11.52\\
& \small{TransCal} &14.83 & 22.09 & 16.38 & 30.37 & 37.84 & 29.92 & 25.24\\
& \small{UTDC} & {\bf 5.82}  & {\bf 5.84} & {\bf 6.30} & {\bf 9.24} & {\bf 5.80} & {\bf 7.53} & \textbf{6.76}\\
\cline{2-9} 
& \small{UTDC*} &4.34 & 5.46 & 4.71 & 7.34 & 6.53 & 6.81 & 5.87\\
& \small{Target-TS} &1.07 & 1.25 & 1.06 & 0.90 & 1.53 & 1.61 & 1.24\\

            \bottomrule
        \end{tabular}%
            }
\end{table}
        
\textbf{Calibration results.} Tables \ref{table:utdc_uda_adaece_tab1}, \ref{table:utdc_ece_tab2}, \ref{table:utdc_adaece_visda} and \ref{table:utdc_ece_tab4} report the calibration results (computed by adaECE with 15 bins) on  Office-home, Office-31, VisDA, and DomainNet respectively.
 The results show that UTDC achieved significantly better results than the baseline methods on all tasks.  The calibration obtained by previous  IW-based methods was slightly better (but in some cases even worse) than a network with no calibration or a network that was calibrated on the source domain. 
 In contrast, the adaECE score obtained by UTDC was almost as good as the adaECE obtained by an oracle that had access to the labels of the domain samples. 
In addition to the adaECE evaluation measure, Table  \ref{table:bs_log_tab1} reports the average calibration results over all Office-home tasks, using three other calibration metrics:   ECE, Negative Log-Likelihood (NLL) and Brier Score (BS) \cite{brier1950verification}. The same trends as above were observed.

 \begin{table}[h!]
        \caption{Calibration metrics results of various UDA calibration methods on the Office-home tasks.}
        \vspace{0.1cm}
	\label{table:bs_log_tab1}
                     \centering
         \scalebox{0.82}{ 
\begin{tabular}{lrrrrrrrrr}
\toprule
   & \multicolumn{3}{c}{CDAN+E}  & \multicolumn{3}{c}{DANN+E}  & \multicolumn{3}{c}{DANN}\\ 
       method &  BS &   NLL &    ECE &  BS &   NLL &    ECE &  BS &   NLL &    ECE\\
\midrule
  Uncalibrated &   0.74 &  3.40 &  31.32  &   0.76 &  3.07 &  29.92 &   0.75 &  2.75 &  24.08\\
        Source-TS &   0.65 &  2.18 &  16.79 &   0.67 &  2.21 &  15.40 &   0.71 &  2.37 &  13.71 \\
                 CPCS &   0.71 &  3.48 &  24.46 &   0.72 &  3.08 &  23.12 &   0.76 &  2.87 &  22.37 \\
          TransCal &   0.69 &  2.70 &  22.12 &   0.73 &  3.08 &  25.22 &   0.81 &  3.72 &  29.71\\
              UTDC &  {\bf 0.62} & {\bf 1.95} & {\bf  8.01} & {\bf 0.64} &  {\bf2.01} &  {\bf 7.81}
           &   {\bf 0.69} & {\bf 2.26} & {\bf 10.35}            \\
      \midrule
                  UTDC* &   0.62 &  1.95 &   7.21 &   0.63 &  1.99 &   5.94 &   0.68 &  2.18 &   5.53\\
                   Target-TS &   0.61 &  1.92 &   5.41 &   0.63 &  1.96 &   2.72 &   0.68 &  2.14 &   2.78\\
\bottomrule
\end{tabular}
    }
\end{table}

\begin{table}[h]
        \caption{Computed temperature on various UDA Office-home tasks, and calibration methods using CDAN+E.}
	\label{table:utdc_t_tab1}
        \centering
                     \scalebox{0.82}{ 
                \begin{tabular}{ll|ccccccccc|r}
                        \toprule 
                        {\small UDA}&
               {\small Method}& 
               {\scriptsize $A \!\!\rightarrow\!\! R$} & 
               {\scriptsize $A \!\!\rightarrow\!\! C$} &
               {\scriptsize $A \!\!\rightarrow\!\!  P$} &
               {\scriptsize $C \!\!\rightarrow\!\!  R$ }&
               {\scriptsize $C \!\!\rightarrow\!\!  P$} &
               {\scriptsize $C \!\!\rightarrow\!\!  A$} &
               {\scriptsize $P \!\!\rightarrow\!\!  R$ }&
               {\scriptsize $P \!\!\rightarrow\!\!  C$} &
               {\scriptsize $P \!\!\rightarrow\!\!  A$} &
               Avg \\
                        \midrule

& \small{Source-TS} &1.96 & 2.02 & 2.02 & 1.87 & 1.90 & 2.06 & 1.63 & 1.72 & 1.68 & 1.87\\
& \small{CPCS} &1.46 & 0.57 & 1.49 & 1.68 & 1.75 & 2.05 & 1.93 & 0.50 & 1.73 & 1.46\\
& \small{TransCal} &2.12 & 1.86 & 2.39 & 1.50 & 1.74 & 1.62 & 1.03 & 0.96 & 0.95 & 1.57\\
   \small{CDAN+E}& \small{UTDC} &2.27 & 2.90 & 2.91 & 1.97 & 2.44 & 2.54 & 1.67 & 2.93 & 2.89 & 2.50\\
\cline{2-12}
& \small{UTDC*}  &2.29 & 3.21 & 2.68 & 2.00 & 2.62 & 2.30 & 1.65 & 3.41 & 2.90 & 2.56\\
& \small{Target-TS} &2.36 & 3.61 & 2.73 & 2.42 & 2.73 & 2.81 & 2.24 & 3.49 & 3.37 & 2.86\\

            \bottomrule
                \end{tabular}%
           }
        \end{table}

\begin{figure}[ht!]
         \resizebox{\textwidth}{!}{
    \includegraphics[scale=0.34]{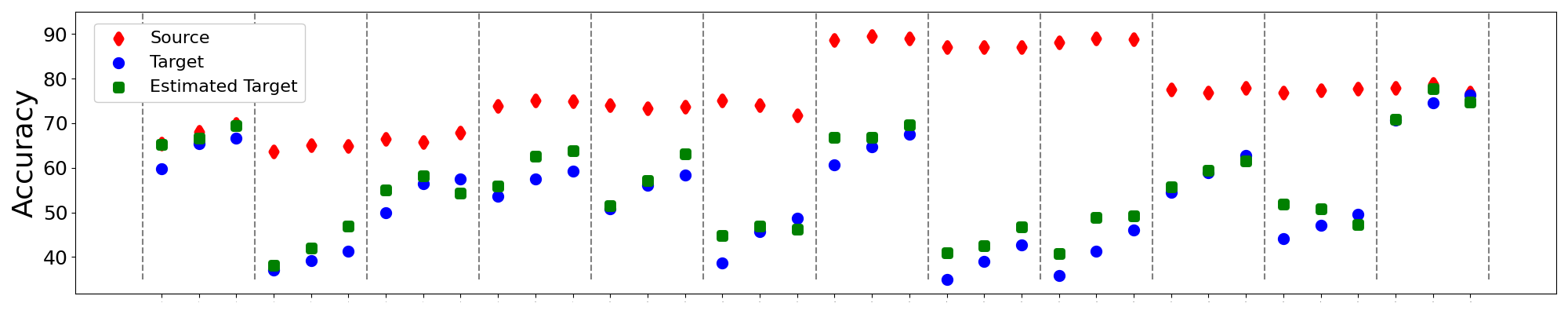}
     }
        { \noindent \hspace{1.6cm}  
         \resizebox{\textwidth}{!}
        { \hspace{2cm}
        $A \rightarrow R$ \hspace{0.29cm}
    { $A \rightarrow C$\hspace{0.29cm}
    $A \rightarrow P$\hspace{0.29cm}
    $C \rightarrow R$\hspace{0.29cm}
    $ C \rightarrow P$\hspace{0.29cm}
    $ C \rightarrow A$\hspace{0.29cm}
    $ P \rightarrow R$ \hspace{0.29cm}
    $ P \rightarrow C$\hspace{0.29cm}
    $ P \rightarrow A$\hspace{0.29cm}
    $ R \rightarrow A$\hspace{0.29cm}
    $ R \rightarrow C$ \hspace{0.29cm}
    $ R \rightarrow P$ \hspace{1.29cm} 
    }}}
    \caption{Average accuracy on Office-home tasks for 
    the three UDA techniques (DANN, DANN+E, CDAN+E).}
    \label{fig:utdc_avg_acc}
\end{figure}

 \section{Analysis}
  We next illustrate and analyze several key features of the proposed method. 
  
{\bf Accuracy gap between source and target.}
\label{subsec:understand_gap}
To gain a better understanding of the reasons 
why our method performs better than IW based methods, we first discuss the accuracy of the adapted models on the source and target domains. Figure \ref{fig:utdc_avg_acc} presents the accuracy on the source and target domains for three UDA techniques. It shows that even after adaptation to the target,  the model's performance on the source samples is consistently better than  its performance on the target samples, especially in cases of large  domain gaps.

Hence, using the network accuracy on the source to  estimate  the network's accuracy on the target while minimizing the ECE measure is misleading  
because the over-optimistic accuracy estimation leads to 
 a scaling  temperature that is too small.
Table \ref{table:utdc_t_tab1} compares the optimal temperatures computed by the  calibration methods. In all the baseline methods the computed calibration temperature was lower than the optimal value. This results in  poorer calibration performance, as seen in Tables \ref{table:utdc_uda_adaece_tab1}, \ref{table:utdc_ece_tab2}, \ref{table:utdc_adaece_visda}, and \ref{table:utdc_ece_tab4}. By contrast,  the temperature computed by all the  UTDC variants was much closer to the  optimal temperature computed by the Oracle method that had access to the target labels. 
 Figure \ref{fig:utdc_avg_acc} also presents the estimated accuracy of the adapted model on the target domain. This estimation is close to the true accuracy. Thus, when it is combined with the confidence computed on the target domain, we  obtain a calibrated mode.

\begin{figure*}[t!]
 \resizebox{\textwidth}{!}{
    \centering
    \includegraphics[trim= 180 00 180 0, clip, scale=0.30]{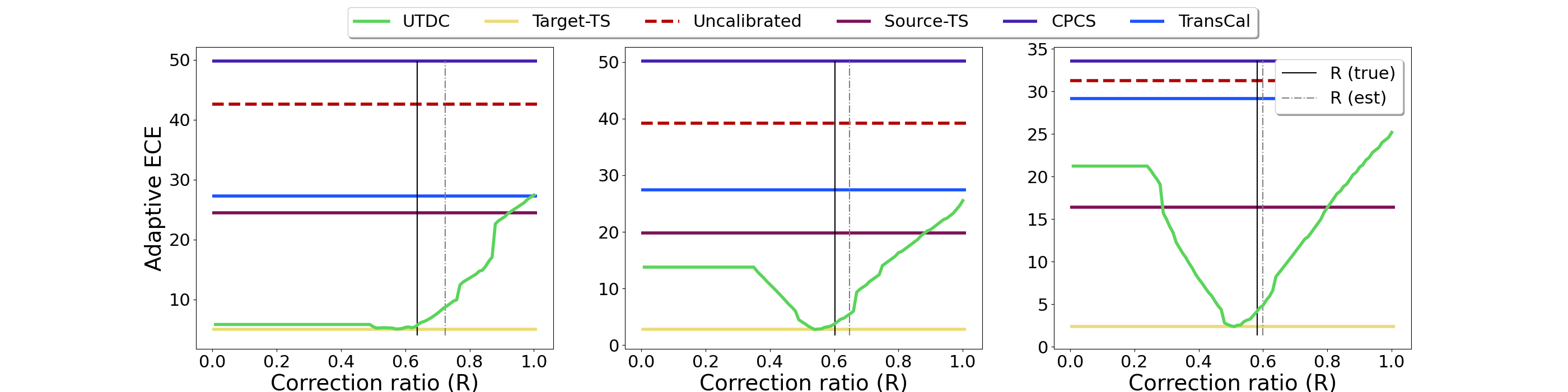} }
          \vspace{0.2cm}
    \hspace{1.2 cm} 
    (a) CDAN+E \hspace{2.0cm} 
    (b) DANN+E \hspace{2.2cm} 
    (c) DANN
      \caption{adaECE results as a function of the correction ratio $R$ on Office-Home, $A \rightarrow C$ task. }
    \label{fig:correction_ratio}
         \end{figure*}

     \begin{figure}[h!]
    \centering
    \includegraphics[scale=0.14]{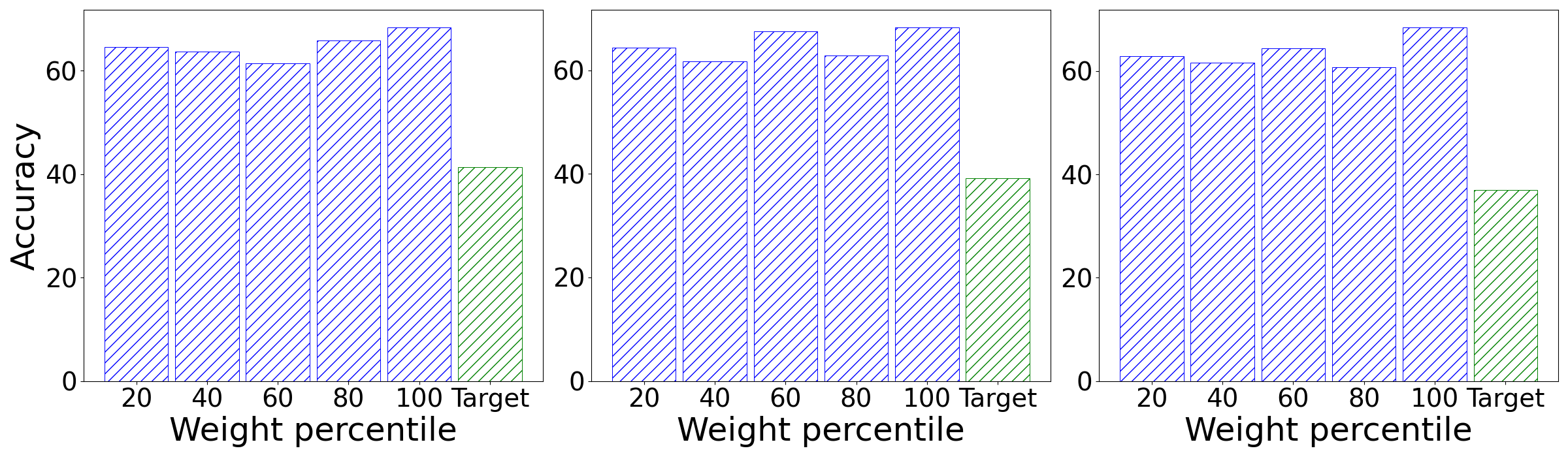} \\
       \hspace{0.65cm}
    (a) CDAN+E \hspace{.6cm}  (b) DANN+E \hspace{.6cm} (c) DANN \hspace{1.4cm} 
    \caption{Accuracy of $k$-th percentile source images based on their probability of being classified as target \cite{wang2020}, compared to target accuracy (Office-home, $A \rightarrow C$). }
    \label{fig:source_acc_gap}
\end{figure}
\begin{figure}[h!!]
    \centering
    \includegraphics[scale=0.15]{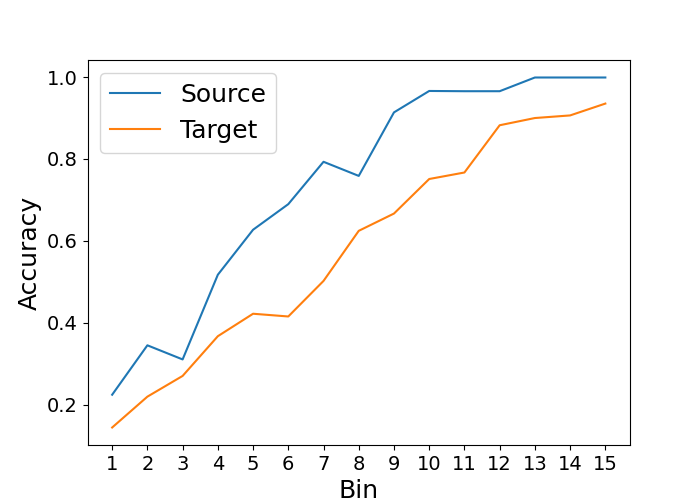}
    \includegraphics[scale=0.15]{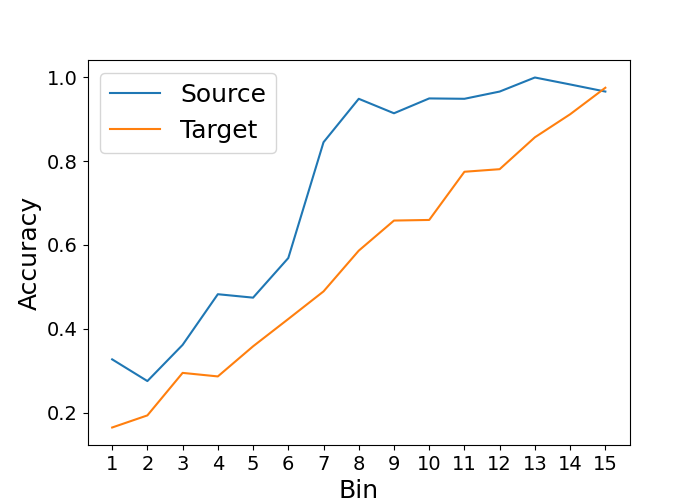}
    \includegraphics[scale=0.15]{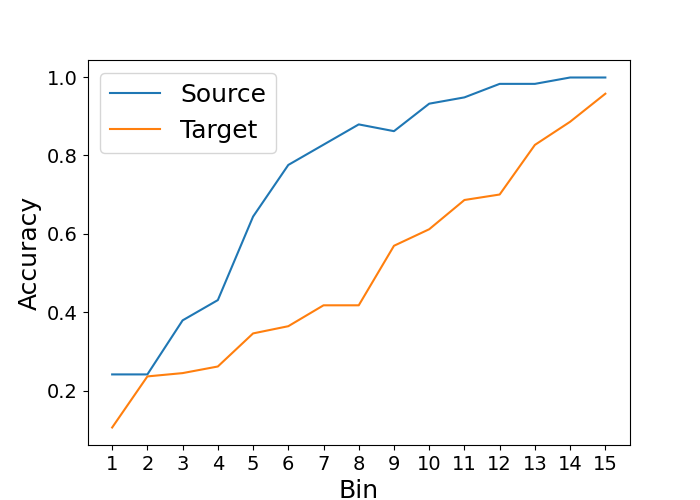}  \\
     (a) CDAN+E \hspace{0.6cm}  (b) DANN+E \hspace{0.6cm} (c) DANN 
    \caption{Accuracy per bin for source and target images. The results are shown on the Office-home $C \rightarrow P$ task. }
    \label{fig:acc_ratio_bins}
\end{figure}

{\bf Sensitivity of UTDC to the target accuracy prediction.} 
\label{sec:sweep}
UTDC is based on estimating the binwise average network accuracy on the target domain data from the labeled source domain data. This estimation is done by computing the ratio ${\tilde{A}_{\textrm{target}}}/{A_{\textrm{source}}}$ between the estimated target accuracy  and the source accuracy. We next analyze the sensitivity of our calibration method to errors in estimating $A_{\textrm{target}}$.
Let  $ R(\textrm{true}) = {A_{\textrm{target}}}/{A_{\textrm{source}}}$ and $ R(\textrm{estimated}) = {\tilde{A}_{\textrm{target}}}/{A_{\textrm{source}}}$
be the true and estimated ratio used by UTDC* and UTDC respectively.
In principle, any number  $0\!<\!R$ can be used to obtain an  estimation of the binwise target accuracy:  $\tilde{A}_{\textrm{target},m} = {A}_{\textrm{source},m} \cdot R$. We can thus find the temperature that minimizes the
adaECE function on the target data as a function of $R$: 
$\hat{T}(R) = \arg \min_T \mathrm{adaECE}_R(T)$
where  
$$ \mathrm{adaECE}_R(T) =
 \frac{1}{M} \sum_{m=1}^{M}\left| 
{A}_{\textrm{source},m} \cdot R - C_{\textrm{target},m}(T) \right|.
$$
Figure \ref{fig:correction_ratio} shows the adaECE measure on the target data after temperature scaling by $\hat{T}(R)$  as a function of the ratio $R$ for the task Office-home $A \rightarrow C$.  It shows  that with the  appropriate choice of $R$  we can achieve the calibration level of the Oracle TS-target algorithm (the case where target labels are known). This means that the difference in accuracy is indeed the main reason for the calibration degradation caused by methods that try to calibrate the target domain using the source data.    
Specifically, as the ratio $R$  drops towards $R$(true), the adaECE improves and approaches the  Oracle TS-target calibration. In addition, the adaECE reaches a minimum near $R$(true) and $R$(estimated).  Finally, there is a range of correction ratios where UTDC is better by a large margin than other baselines, thus providing a tolerance for error and resilience in estimating $\tilde{A}_{\textrm{target}}$.

{\bf The problem with the IW assumption.}
We showed that our method achieves better results by explicitly addressing the accuracy gap between the source and  target domains caused by the domain shift.
Previous methods based on importance weights 
 \cite{park2020,wang2020} rely on re-weighting the source data based on their proximity to the target data, i.e., concentrating on source samples that resemble the target and attributing less weight to others. We computed the target similarity weights associated with each sample in the source validation set and divided them into $20\%$ percentile subsets. Figure \ref{fig:source_acc_gap}  shows the average accuracy of each group and the average target accuracy.  
It shows that the source accuracy is similar in all bins regardless of the similarity to the target. Thus the IW assumption
that source samples that are classified as targets are more relevant for calibrating the target prediction is wrong.

\begin{table}[t]
\centering
\caption{AdaECE results for variations of UTDC based on different methods of domain accuracy estimation.}
\label{table:diff_est_methods__ece}
      \centering
         \scalebox{0.83}{ 
    \begin{tabular}{l|rrrr}
        \toprule 
        {\small Method}&
           {\small Office-home}& 
           {\small Office-31} & 
           {\small VisDA} &
           {\small DomainNet} \\
        \midrule
            \small{Uncalibrated} & 28.44 & 13.51 &  31.63 & 16.74\\
            \small{UTDC-Meta}\cite{deng2021labels} & 8.67 & 6.96 &  7.84 & 7.18\\
            \small{UTDC-ATC}\cite{garg2022leveraging} & 10.12  & 7.47 &  5.68 & 8.01\\
            \small{UTDC-PN}  \cite{yu2022predicting} &  11.55 &  7.83 &  10.20 & 8.63 \\
            \small{UTDC*} & 6.24  & 6.13  & 2.27 & 5.78\\
        \bottomrule
    \end{tabular}%
    }
\end{table}
\begin{table}[t]
\centering
\caption{Comparison of several target domain accuracy estimation methods measured by   $|ACC(True) - ACC(Est)|$.}
\label{table:diff_est_methods__acc}
      \centering
         \scalebox{0.83}{ 
    \begin{tabular}{l|cccc}
        \toprule 
        {\small Method}&
           {\small Office-home}& 
           {\small Office-31} & 
           {\small VisDA} &
           {\small DomainNet}\\
        \midrule 
            \small{Meta}\cite{garg2022leveraging} & 3.31 & 2.81 &  4.96 & 3.10\\
            \small{ATC}\cite{garg2022leveraging} & 5.05  & 3.37 &  3.48 & 4.25\\
            \small{PN} \cite{yu2022predicting} &  6.26 & 4.85 & 6.30 & 5.91 \\
        \bottomrule
    \end{tabular}%
    }
\end{table}

{\bf Accuracy ratio across bins.}
Our method computes $\tilde{A}_{\textrm{target},m}$ by re-scaling $A_{\textrm{source},m}$ with the same ratio for all bins, as defined in  \ref{targetaccuracy}. This estimation is based on the assumption that the accuracy ratio between the source and the target is similar across the bins. To illustrate the validity of this assumption, Figure \ref{fig:acc_ratio_bins} shows the accuracy of the adapted network  at  each bin, for the source and target data.

{\bf Different target accuracy estimation methods.} 
\label{sec:other_est}
Our UTDC method requires an estimation step of the target domain accuracy without labels. In all the experiments  reported above we used the Meta method \cite{deng2021labels}. 
We next examine combining UTDC with two other methods for target domain accuracy estimation:   ATC \cite{garg2022leveraging} and PN \cite{yu2022predicting}. 
  We implemented 3 variations of UTDC, dubbed  UTDC-Meta, UTDC-ATC, and UDTC-PN based on the estimated target accuracy that was used.  We also report results for UTDC*  based on the true target accuracy. 
  Tables \ref{table:diff_est_methods__ece} and \ref{table:diff_est_methods__acc} present the average calibration results and the discrepancy between the estimated and actual accuracy, respectively.
  The results indicate that UTDC achieved the best calibration performance out of all the three target accuracy estimation methods examined, thus reinforcing the observed low sensitivity of UTDC to the precision of target accuracy predictions. This underscores the compatibility of UTDC with existing methods for network calibration under unsupervised domain shift.
We also found that using UTDC-Meta yields better results, while UTDC-ATC exhibits improved performance and ease of implementation, since the ATC method is much simpler to implement and requires a small computational effort.

\newpage
\chapter{Discussion}
\label{ch:discussion}

This chapter brings together the key findings and contributions of this dissertation, highlighting their implications for the broader field of machine learning. Throughout our research, we addressed critical challenges related to confidence calibration and uncertainty quantification, focusing on complex scenarios involving label noise, privacy constraints, and domain shifts. By developing novel methodologies and demonstrating their effectiveness across diverse settings, our work advances the reliability and interpretability of machine learning models in real-world applications. Here, we summarize our contributions, reflect on key insights, and outline promising directions for future research.

\section{Summary of Contributions}

In this dissertation, we explored multiple facets of confidence calibration and uncertainty quantification in machine learning, particularly under challenging real-world conditions such as label noise, privacy constraints, and domain shift. Our work contributes novel methodologies that enhance the robustness of existing calibration and conformal prediction (CP) frameworks, thereby addressing key limitations in the current literature.

First, we investigated confidence calibration in classification models trained with noisy labels, a critical issue in domains like medical imaging where obtaining accurate labels is often impractical. We demonstrated that traditional calibration methods are highly sensitive to label noise, even when network training itself remains relatively robust. To address this, we proposed a noise-aware calibration framework that effectively models label corruption using a noise transition matrix. Our results show that this approach achieves calibration performance comparable to that of clean labels, provided the noise model is well estimated.

Second, we extended the conformal prediction framework to handle label noise. We introduced a procedure that adjusts the calibration threshold based on a given noise model, allowing CP to maintain valid coverage while minimizing the size of prediction sets. We derived finite-sample coverage guarantees for the uniform noise case and showed that our method significantly improves over existing noisy CP approaches in terms of efficiency and prediction set size.

Third, we addressed confidence calibration under privacy constraints by developing two complementary CP approaches for locally differentially private (LDP) settings. These methods, LDP-CP-L (label perturbation) and LDP-CP-S (score perturbation), provide valid uncertainty quantification while ensuring strong privacy guarantees for individual user data. Our results highlight the trade-offs between privacy, computational feasibility, and model performance, offering guidance for selecting the appropriate method in real-world applications.

Lastly, we studied network calibration in the context of unsupervised domain adaptation, where a model trained on a labeled source domain is deployed in a target domain with a different distribution. We showed that existing importance-weighting approaches fail to correct for domain shift in calibration. Instead, we proposed a method that directly calibrates using target-domain examples, leading to substantial improvements in calibration performance.

\section{Key Insights and Implications}

\textbf{Sensitivity of Calibration to Label Noise.} Our findings reveal that calibration procedures are far more sensitive to label noise than the training process itself. This observation underscores the need for noise-resilient calibration methods, as even a small fraction of incorrect labels in the validation set can severely degrade calibration quality. Our proposed approach, which explicitly accounts for label noise via a noise transition matrix, mitigates these effects and provides reliable confidence estimates despite label corruption.

\textbf{Conformal Prediction Under Noisy Labels.} We demonstrated that CP can be adapted to handle noisy labels by estimating a noise-free calibration threshold. Our method preserves valid coverage and improves efficiency compared to existing approaches. However, our analysis suggests that current coverage guarantees may be overly conservative, indicating room for further theoretical refinement. Additionally, we focused on noise models where label corruption is independent of input features; extending our method to more complex, feature-dependent noise processes remains an open challenge.

\textbf{Privacy-Preserving Uncertainty Quantification.} Our work on LDP-conformal prediction bridges the gap between uncertainty quantification and privacy protection. By leveraging randomized response mechanisms, we ensured valid CP coverage while preserving user privacy. A notable challenge for future work is integrating more sophisticated LDP techniques, such as RAPPOR, into our framework to enhance robustness and scalability. Furthermore, developing a method that preserves privacy without requiring local score computation remains an important open problem.

\textbf{Confidence Calibration in Domain Adaptation.} We showed that calibration methods relying on source-domain accuracy fail under domain shift due to differences in true label distributions. Our approach, which calibrates directly on target-domain examples, outperforms existing methods. Extending this idea to non-parametric calibration techniques such as CP, as well as to regression and segmentation tasks, is a promising avenue for future research. Additionally, addressing confidence calibration in source-free adaptation—where access to source data is not available—remains an open challenge.

\section{Future Research Directions}

While our research has addressed several key challenges, many promising directions remain for future exploration:

\textbf{Refining Theoretical Guarantees for Noisy Conformal Prediction.} Our results suggest that existing theoretical bounds for CP under noisy labels may be overly conservative. Further research into tighter finite-sample guarantees could improve the efficiency of CP methods in noisy settings.

\textbf{Advanced Noise Models in Calibration.} Our methods assume that the noise transition matrix can be estimated with reasonable accuracy. However, in cases of highly unbalanced class distributions or feature-dependent noise, estimating this matrix remains challenging. Developing robust estimation techniques for such cases is an important open problem.

\textbf{Extending Privacy-Preserving CP Methods.} Our current LDP-CP framework is limited by the assumption that users can either perturb labels or scores locally. A more stringent privacy setting, where both the input features and labels must remain private, presents a fundamental challenge. Designing methods that operate under such conditions while maintaining valid coverage is a crucial area for future work.

\textbf{Calibration in Source-Free Adaptation.} In domain adaptation, we demonstrated that calibration benefits from access to target-domain examples. However, a more challenging setting is source-free adaptation, where only unlabeled target data is available. Developing effective calibration methods under such constraints is a key research direction.

\textbf{Applying Calibration Strategies to Other Tasks.} Our work focused on classification tasks, but the principles developed here could extend to regression and segmentation problems. Exploring calibration techniques for these settings, particularly under domain shift and noisy labels, could have significant practical impact.

\section{Conclusion}

This dissertation advances the understanding and application of confidence calibration and uncertainty quantification in machine learning. By addressing challenges in label noise, privacy, and domain shift, we developed robust methodologies that improve calibration performance across diverse settings. Our work not only enhances the reliability of machine learning models but also opens new avenues for research in uncertainty estimation under real-world constraints. We hope that our contributions will inspire further studies in these areas, ultimately leading to more trustworthy and interpretable AI systems.

\addcontentsline{toc}{chapter}{Bibliography}
\bibliography{src/refs}
\bibliographystyle{plain}



\includepdf[pages=1]{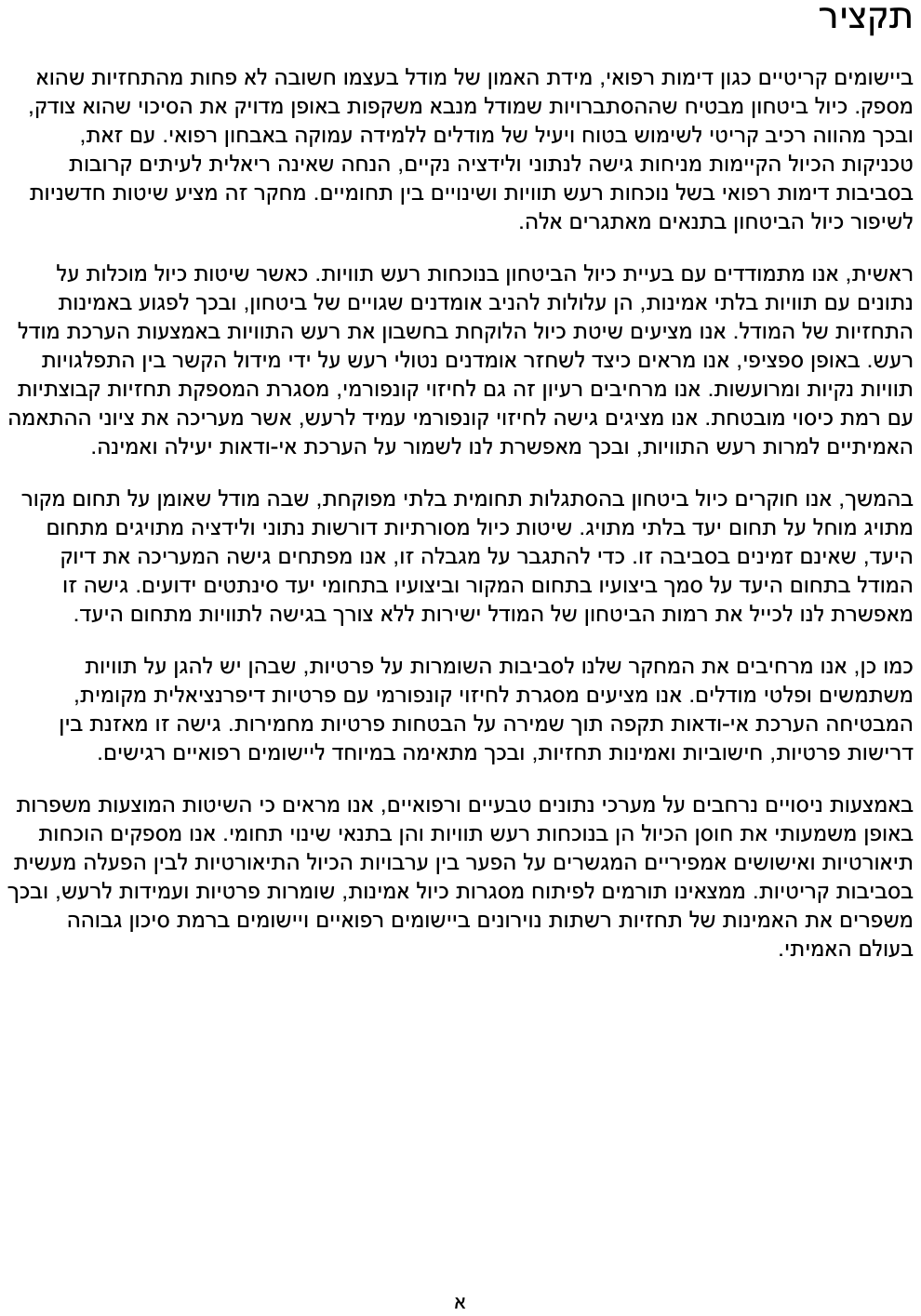}
\includepdf[pages=5]{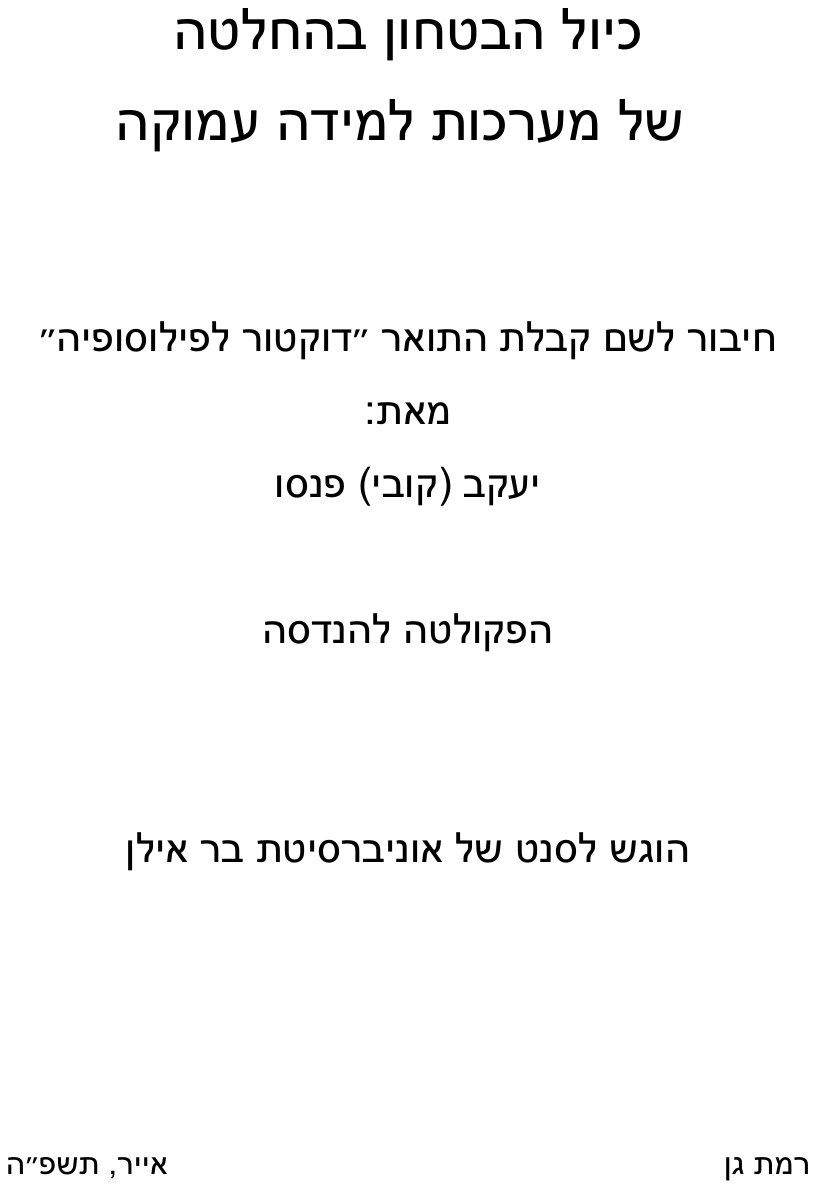}
\includepdf[pages=6]{src/pdf/thesis_hebrew_15_10_25}
\includepdf[pages=3]{src/pdf/thesis_hebrew_15_10_25}
\includepdf[pages=2]{src/pdf/thesis_hebrew_15_10_25}
\includepdf[pages=1]{src/pdf/thesis_hebrew_15_10_25}
	
\end{document}